\documentclass[twoside]{article}
\usepackage[round]{natbib}
\usepackage{amsthm,amsmath, amssymb, natbib, enumerate}
\usepackage {graphicx}
\usepackage[usenames,dvipsnames]{xcolor}
\usepackage[draft]{hyperref}
\usepackage[shortlabels]{enumitem}
\hypersetup{colorlinks, urlcolor=black, linkcolor=blue, citecolor=blue}
\usepackage[toc,page]{appendix}
\usepackage{algorithm}
\usepackage{algorithmic}

\newtheorem{example}{Example}
\newtheorem{theorem}{Theorem}
\newtheorem{remark}{Remark}
\newtheorem{lemma}{Lemma}

\newtheorem{prop}[theorem]{Proposition}
\def\argmax{\mathop{\rm argmax}}
\def\argmin{\mathop{\rm argmin}}
\def\X{\textbf{X}}
\def\y{\textbf{y}}
\def\Z{\textbf{Z}}
\def\z{\textbf{z}}
\DeclareMathOperator*{\sgn}{sgn}

\usepackage[accepted]{aistats2021}
\begin{document}

\twocolumn[

\aistatstitle{On the Generalization Properties of Adversarial Training}
\aistatsauthor{ Yue Xing \And Qifan Song \And  Guang Cheng }
\aistatsaddress{ Purdue University \And Purdue University \And Purdue University }]
\begin{abstract}
		Modern machine learning and deep learning models are shown to be vulnerable when testing data are slightly perturbed. Existing theoretical studies of adversarial training algorithms mostly focus on either adversarial training losses or local convergence properties. In contrast, this paper studies the generalization performance of a generic adversarial training algorithm. Specifically, we consider linear regression models and two-layer neural networks (with lazy training) using squared loss under low-dimensional and high-dimensional regimes. In the former regime, after overcoming the non-smoothness of adversarial training, the adversarial risk of the trained models can converge to the minimal adversarial risk. In the latter regime, we discover that data interpolation prevents the adversarially robust estimator from being consistent. Therefore, inspired by successes of the least absolute shrinkage and selection operator (LASSO), we incorporate the $\mathcal{L}_1$ penalty in the high dimensional adversarial learning and show that it leads to consistent adversarially robust estimation. A series of numerical studies are conducted to demonstrate how the smoothness and $\mathcal{L}_1$ penalization help improve the adversarial robustness of DNN models.
	\end{abstract}
	
	\section{INTRODUCTION}
	Recent advances in deep learning and machine learning have led to breakthrough performance and are widely applied in practice. However, empirical experiments show that deep learning models can be fragile and vulnerable against adversarial input which is intentionally perturbed (\citealp{biggio2013evasion,szegedy2013intriguing}). For instance, in image recognition, a deep neural network will predict a wrong label when the testing image is slightly altered, while the change is not recognizable by the human eye (\citealp{papernot2016limitations}). To ensure the reliability of machine learning and deep learning when facing real-world inputs, the demand for robustness is increasing. 
    The related research efforts in adversarial learning include designing adversarial attacks in various applications (\citealp{papernot2016limitations,papernot2016crafting,moosavi2016deepfool}), detecting attacked samples (\citealp{tao2018attacks,ma2019nic}), and modifications on the training process to obtain adversarially robust models, i.e., adversarial training \citep{shaham2015understanding,madry2017towards,jalal2017robust,Balunovic2020Adversarial}.
    
    To introduce adversarial training, let $l$ denote the loss function and $f_{\theta}(x): \mathbb{R}^d\rightarrow \mathbb{R}$ be the model with parameter $\theta$. The (population) adversarial loss is defined as
    $
    R_f(\theta,\epsilon):=\mathbb{E}\left[l\left( f_{\theta}[x+A_{\epsilon}(f_{\theta},x,y)],y \right)\right],
    $
    where $A_\epsilon$ is an attack of strength $\epsilon>0$ and intends to deteriorate the loss in the following way
    \begin{equation}\label{eqn:attack}
    A_{\epsilon}(f_{\theta},x,y):=\argmax\limits_{z\in \mathcal{R}(0,\epsilon)} \{l(f_{\theta}(x+z),y)\}.
    \end{equation}
    In the above, $z$ is subject to the constraint $\mathcal{R}(0,\epsilon)$, i.e. an $\mathcal{L}_2$  ball centered at $0$ with radius $\epsilon$. 
    
    Given i.i.d. training samples $\{(x_i,y_i)\}_{i=1}^{n}$, the adversarial training aims to minimize an empirical version of $R_f(\theta,\epsilon)$ w.r.t. $\theta$:
    \begin{equation}\label{eqn:emp_adv}
    \widehat{R}_f({\theta},\epsilon)=\frac{1}{n}\sum_{i=1}^nl\left( f_{\theta}[x_i+A_{\epsilon}(f_{\theta},x_i,y_i)],y_i \right),
    \end{equation}
    and
    $\widehat\theta= \argmin_{\theta}\widehat{R}_f({\theta},\epsilon).$
    The minimization in (\ref{eqn:emp_adv}) is often implemented through an iterative two-step (min-max) update. In the $t$-th iteration, we first calculate the adversarial sample $\widetilde{x}_i^{(t)}=x_i+A_{\epsilon}(f_{\theta^{(t)}},x_i,y_i)$ based on the current $\theta^{(t)}$, and then update $\theta^{(t+1)}$ based on the gradient of the adversarial training loss while fixing $\widetilde{x}_i^{(t)}$'s; see Algorithm \ref{alg}. This generic algorithm and its variants have been studied in \citealp{shaham2015understanding,madry2017towards,jalal2017robust,Balunovic2020Adversarial,sinha2018certifying,wang2019convergence} among others. Note that for complex loss function $l$ or model $f_{\theta}$, there may not be an analytic form for $A_{\epsilon}$ (e.g. deep neural networks). In this case, an additional iterative optimization is needed to approximate $A_{\epsilon}$ at each step; see \citealp{wang2019convergence}.
    
    \begin{algorithm}[ht!]
        \caption{A General Form of Adversarial Training}
        \label{alg}
        \begin{algorithmic}
            \STATE {\bfseries Input:} data $(x_1,y_1)$,..., $(x_n,y_n)$, attack strength $\epsilon$, number of steps $T$, initialization $\theta^{(0)}$, step size $\eta$.
            \FOR{$t=1$ {\bfseries to} $T$}
            \FOR{$i=1$ {\bfseries to} $n$}
            \STATE Calculate the attack for the $i$th sample and get $\widetilde{x}_i^{(t-1)}=x_i+A_{\epsilon}(f_{\theta^{(t-1)}},x_i,y_i).$
            \ENDFOR
            \STATE Fixing $\widetilde{x}_i^{(t-1)}$'s, update $\theta^{(t)}$ from $\theta^{(t-1)}$ through 
            \begin{eqnarray}
            &&\theta^{(t)}=\theta^{(t-1)}-\eta \nabla_\theta \widehat{R}_f(\theta^{(t-1)},\epsilon)\\&=&\theta^{(t-1)}-\eta\nabla_\theta\left[\frac{1}{n}\sum_{i=1}^nl\left( f_{\theta^{(t-1)}}(\widetilde{x}_i^{(t-1)}),y_i \right)\right].\nonumber
            \end{eqnarray}
            \ENDFOR
            \STATE{\bfseries Output:} $\theta^{(T)}$.
        \end{algorithmic}
    \end{algorithm}

In the literature, there are three major strands of theoretical studies related to this work. The first strand focuses on the statistical properties or generalization performance of adversarially robust estimators without taking account of the role of optimization algorithms \citep{javanmard2020precise,yin2018rademacher,raghunathan2019adversarial,schmidt2018adversarially,najafi2019robustness,min2020curious,zhai2019adversarially,hendrycks2019using,chen2020more}. For instance, \cite{javanmard2020precise} studied the statistical properties of $R_f({\widehat\theta},\epsilon)$, without specifying how to obtain the exact/approximate global minimizer $f_{\widehat\theta}$. The second strand studies the adversarial training loss, i.e., the limiting behaviors of $\widehat{R}_f(\theta^{(t)},\epsilon)$ as a training algorithm iteration $t$ grows. For instance, \citet{gao2019convergence,zhang2020over} showed that for over-parameterized neural networks, the empirical adversarial loss could be arbitrarily close to the minimum value in a local region near initialization. The third strand studies the (local) convergence of adversarial training under certain convexity assumptions \citep[e.g.,][]{sinha2018certifying,wang2019convergence}. Besides these theoretical results, there are a few empirical works as well (e.g., \citealp{wong2020fast,wang2019improving,rice2020overfitting,lee2020rethinking,wu2020revisiting,xie2020smooth}).

    In this paper, we investigate the global convergence and the generalization ability (i.e., $R_f(\theta^{(T)},\epsilon)$) of the adversarial training, for two models $f_\theta$, linear regression and two-layer neural networks, with the squared loss $l(f_{\theta}(x),y)=(f_{\theta}(x)-y)^2$ under $\mathcal{L}_2$ and $\mathcal{L}_{\infty}$ attacks. 
Our theoretical contributions are summarized as follows.

First, the adversarial loss $R_f({\theta},\epsilon)$ suffers from non-differentiability even if both $l$ and $f_\theta$ are smooth, thus the optimization, i.e., Algorithm~\ref{alg}, works poorly. This motivates us to introduce a surrogate attack to overcome the non-smoothness problem.  Under low dimensional setup and $\mathcal{L}_2$ attack, we show that under proper conditions, the iterative estimate $\theta_{\xi}^{(T)}$ trained from the surrogate adversarial loss asymptotically achieves the minimum adversarial risk (Section \ref{sec:low}).
 
Secondly, we observe the ``data interpolation'' behavior of $\mathcal{L}_2$ adversarial training under high dimensional setup ($d/n\rightarrow\infty$). More specifically, the training loss converges to zero, but the population loss $R_f(\theta_{\xi}^{(T)},\epsilon)$ converges to a large constant which is the adversarial loss of null model. To remedy the poor generalization, we penalize the adversarial training loss using LASSO. The resulting adversarially robust estimator and adversarial risk are both consistent for under some sparsity assumption (Section \ref{sec:high}).

 Thirdly, we examine the differences between the $\mathcal{L}_{\infty}$ and $\mathcal{L}_2$ adversarial training. One similarity with $\mathcal{L}_2$ case is that, when attack strength is small, data interpolation prevents $\mathcal{L}_{\infty}$ adversarial training from achieving a consistent estimator as well, and the issue can be improved by the use of LASSO penalty. In terms of the differences, in general, it is harder to conduct adversarial training under $\mathcal{L}_{\infty}$ attack, in the sense that a lower learning rate and more iterations are required to ensure the convergence of $\theta_{\xi}^{(T)}$  (Section \ref{sec:linf}).

    It is worth mentioning that a recent work by \cite{allen2020feature} also conducted a similar analysis for the generalization ability of adversarial training for a two-layer ReLU network model. However, the focus of their work is to explain the feature purification effect of adversarial training in neural network and overlook the difficulties of adversarial training when $\epsilon$ does not asymptotic goes to 0.
  
    For technical simplicity, throughout the paper, we assume the data are generated from linear regression: 
    \begin{eqnarray}\label{eqn:model}
    y=\theta_0^{\top}x+\varepsilon,
    \end{eqnarray}
    where $x\in\mathbb{R}^d$ is Gaussian vector with mean 0 and variance $\Sigma$ ($\theta_0$ is not perpendicular to $\Sigma$), and $\varepsilon$ is a Gaussian noise (independent of $x$) with variance $\sigma^2<\infty$. As $d$ diverges, we assume both maximum and minimum eigenvalues of $\Sigma$ are finite and bounded away from 0. In addition, $\|\theta_0\|$ and $\sigma^2$ are allowed to increase in $d$, but the signal-to-noise ratio $\|\theta_0\|_{\Sigma}/\sigma$ is large, say bounded away from zero. 
	
	\section{LOW DIMENSIONAL ASYMPTOTICS}\label{sec:low}
	
This section considers linear regression models and two-layer neural network models (training the first layer weights) under the low dimensional scenario. In particular, we examine the landscape of the adversarial loss and further investigate the testing performance of the estimator $\theta_{\xi}^{(T)}$. In what follows, we rewrite $R_{f}$ as $R_L$ for linear models and as $R_N$ for two-layer networks. 
	
\subsection{Linear regression model}
	
Consider the linear regression model: $f_{\theta}(x)=\theta^{\top}x$. By the definition of $A_{\epsilon}(f_{\theta},x,y)$ under the $\mathcal{L}_2$ ball constraint and the fact that $x$ and $\varepsilon$ are both Gaussian, $R_L$ has an analytical form as
	\begin{equation}\label{eqn:population}
	\begin{split}
	R_L(\theta,\epsilon)=&\|\theta-\theta_0\|_{\Sigma}^2+\sigma^2+\epsilon^2\|\theta\|^2\\&+2\epsilon c_0\|\theta\|\sqrt{ \|\theta-\theta_0\|_{\Sigma}^2+\sigma^2},
	\end{split}
	\end{equation}
	where $\|a\|_{\Sigma}^2:=a^{\top}\Sigma a$, $\|a\|:=\|a\|_2$ for any vector $a$, and $c_0:=\sqrt{2/\pi}$. Given any $\epsilon\ge 0$, define 
	\begin{equation*}
	\theta^*(\epsilon):=\argmin_{\theta }R_L(\theta,\epsilon)\;\text{and }\;R^*(\epsilon):=\min\limits_{\theta } R_L(\theta,\epsilon).
	\end{equation*}
	When no confusion arises, we will rewrite $\theta^*(\epsilon)$ and $R^*(\epsilon)$ as $\theta^*$ and $R^*$ for simplicity. 
	
We first analyze $R_L({\theta},\epsilon)$ based on the form (\ref{eqn:population}).
	\begin{prop}\label{prop:geo}
		Assume the attack strength $\epsilon>0$ and the data generation follows (\ref{eqn:model}). The adversarial loss $R_L$ is differentiable w.r.t. $\theta$ if and only if (1) $\sigma^2>0$ and $\theta\neq 0$, or (2) $\sigma=0$, $\theta\neq 0$, and $\theta\neq\theta_0$. The function $R_L$ is convex w.r.t. $\theta$, and there exists some constant $c$ such that when $\epsilon\geq c$, $\theta^*=0$.
	\end{prop}
	
From Proposition \ref{prop:geo}, there always exists some $\theta$ where $R_L(\theta,\epsilon)$ is not differentiable (e.g., $\theta=0$), and it is not avoidable when $\epsilon$ is large. Even if we smooth the standard loss as in \cite{xie2020smooth} to improve the quality of the gradients (or use smooth classifiers as in \citealp{salman2019provably}), the non-differentiable issues remains for the adversarial loss. This makes it difficult to track the trajectory of gradient descent algorithms in the sense that the adversarial training loss will fluctuate during training rather than strictly decreases over iterations.

To solve this problem and improve the gradient quality, for both models under consideration, we introduce $A_{\epsilon,\xi}$ that is a surrogate for $\mathcal{L}_2$ attack $A_{\epsilon}$:
    $$A_{\epsilon,\xi}(f_{\theta},x,y)=\frac{\|\partial l(f_{\theta}(x),y)/\partial x\|}{\sqrt{\|\partial l(f_{\theta}(x),y)/\partial x\|^2+\xi^2}}A_{\epsilon}(f_{\theta},x,y).
    $$
     \cite{lee2020rethinking} showed that the instability of adversarial training when $\partial l(f_{\theta}(x),y)/\partial x$ is closed to zero. Therefore, the design of $A_{\epsilon,\xi}$ aims to impose shrinkage effect on $A_{\epsilon}$ when $\|\partial l(f_{\theta}(x),y)/\partial x\|$ is small, thus the training process will be more stable.

    When $\xi=0$, $A_{\epsilon,\xi}$ is reduced to $A_{\epsilon}$. 
    With the surrogate $A_{\epsilon,\xi}$, we define the empirical and population surrogate loss as:
    \[
    \begin{split}
        &\widehat{R}_{L,\xi}({\theta},\epsilon):=\frac{1}{n}\sum_{i=1}^n l(f_{\theta}(x_i+A_{\epsilon,\xi}(f_{\theta},x_i,y_i)),y_i),\\
        &R_{L,\xi}({\theta},\epsilon):=\mathbb{E}l( f_{\theta}(x_i+A_{\epsilon,\xi}(\theta,x_i,y_i)),y_i ).
    \end{split}
    \]
    One can show, $R_{L,\xi}({\theta},\epsilon)$ is smooth and convex everywhere.
    Accordingly, Algorithm \ref{alg} is modified by replacing $A_{\epsilon}$ with $A_{\epsilon,\xi}$. Note that a smaller value of $\xi$, which leads to a closer approximation to $R_{L,0}$, makes $R_{L,\xi}$ less smooth at the origin and thus requires a lower learning rate $\eta$ and more iteration steps $T$. 
    Therefore, we require a fine-tuning of $\xi$ in the sense that $\xi/\|\theta_0\|$ slowly decreases to 0. The surrogate loss for two-layer neural networks, i.e., $R_{N,\xi}$ and $\widehat{R}_{N,\xi}$, is defined in the same way.

	The use of surrogate attack improve the gradient quality of adversarial training and enforces the smoothness and convexity of surrogate loss $R_{L,\xi}(\theta,\epsilon)$. It facilitates the convergence and generalization analysis of adversarial training. We study the consistency of $\theta_{\xi}^{(T)}$ toward $\theta^*$ and $R_{L,\xi}(\theta_{\xi}^{(T)},\epsilon)$ toward $R^*$,  under suitable choices of $(\xi,\eta,t)$ after accounting for the dimensional effect $v^2:={ \|\theta_0\|_{\Sigma}^2+\sigma^2 }$ (for simplicity assume $v^2$ bounded away from zero); see Theorem~\ref{thm:opt}.
	\begin{theorem}\label{thm:opt}
		Assume the data generation follows (\ref{eqn:model}) and let $\epsilon$ be a fixed constant. If there exists some constant $B_0$ such that $\|\theta_{\xi}^{(0)}\|\leq B_0v$, and the dimension growth rate satisfies $\log n\sqrt{d^2/n}\rightarrow 0$, then with probability tending to 1, the surrogate loss $R_{L,\xi}(\theta_{\xi}^{(t)},\epsilon)$ decreases in each iteration, and 
		$$
		\frac{R_{L,\xi}(\theta_{\xi}^{(T)},\epsilon)-R^*}{  v^2 }\rightarrow 0,\mbox{ and } \frac{\|\theta_{\xi}^{(T)}-\theta^*\|}{  v }\rightarrow 0,
		$$
		given $\eta=\xi/(v^2L)$ for some large constant $L$, $T=(v^2\log\log n)/\xi$ and $\xi=v^2d/\sqrt{n}\log n$.
	\end{theorem}
	The proof of Theorem \ref{thm:opt} is postponed to Appendix C. The results of Theorem \ref{thm:opt} actually hold for a wide range of $(\xi,\eta,T)$ (as elaborated in the Appendix C). In general, the choice of $(\eta,T)$ can be invariant to $v$, while a larger $v$ implies a wider range of possible $\xi$. In terms of tuning $\theta_{\xi}^{(0)}$ and $\xi$, one can estimate $v^2$ via the sample variance of $y_1,...,y_n$. Note that the above discussions apply to Theorems \ref{thm:low_nonlinear},  \ref{thm:low_relu}, \ref{thm:high_linear}, \ref{thm:high_network}, and \ref{thm:high_relu} as well.
	
	Theorem~\ref{thm:opt} also confirms the phenomenon that ``adversarial training hurts standard estimation'' \citep{raghunathan2019adversarial}. Denote $\widehat{\theta}_{\rm OLS}$ as the common least square estimator for un-corrupted data (i.e. without attack) and trivially $\widehat{\theta}_{\rm OLS}\rightarrow\theta_0$. By Theorem \ref{thm:opt},  $\theta_{\xi}^{(T)}\rightarrow\theta^*\neq \theta_0$, thus

	\[
	\begin{split}
	    \frac{R_{L}(\theta_{\xi}^{(T)},0)-R_{L}(\widehat{\theta}_{\rm OLS},0)}{v^2}&\rightarrow \frac{R_{L}(\theta^*,0)-R_{L}({\theta}_0,0)}{v^2}\\&= c(\epsilon)>0,
	\end{split}\] where the function $c(\epsilon)$ increases in $\epsilon$, and converges to $\|\theta_0\|_{\Sigma}^2/v^2$ as $\epsilon$ diverges. 
	
	\begin{remark}
	The Gaussian assumption of $x$ is used in Theorem \ref{thm:opt} for the convenience of showing the analytical form of $R_L$. All the results in this paper, except for Section \ref{sec:sparse}, are still valid as long as the density of $(x,y)$ is finite and $\exp(-t\|x\|^2)<\infty$ for some $t>0$.
	\end{remark}
	
	\textbf{Numerical experiments}\\
	To demonstrate the necessity of smoothing attacks under large $\epsilon$, a simple simulation is conducted. Let $d=10$, $n=1000$, $x\in\mathcal{R}^d$ and $\Sigma=I_d$. The response $y=\beta_0^{\top}x+\varepsilon$ with $\beta_0=(1,...,1)^{\top}$ and $\sigma^2=1$. The level of attack is taken as $\epsilon=3$. We use the same training data set in the adversarial training for $\xi=0$ and $\xi=0.05$. Figure \ref{fig:low_xi} displays the effectiveness of $\xi$: the loss with $\xi=0$ always fluctuates, and after we smooth the training process through taking $\xi=0.05$, (surrogate) training and (non-surrogate) testing loss smoothly decrease in each iteration. 
	\begin{figure}
	    \centering
	    \vspace{-0.1in}
	    \includegraphics[scale=0.5]{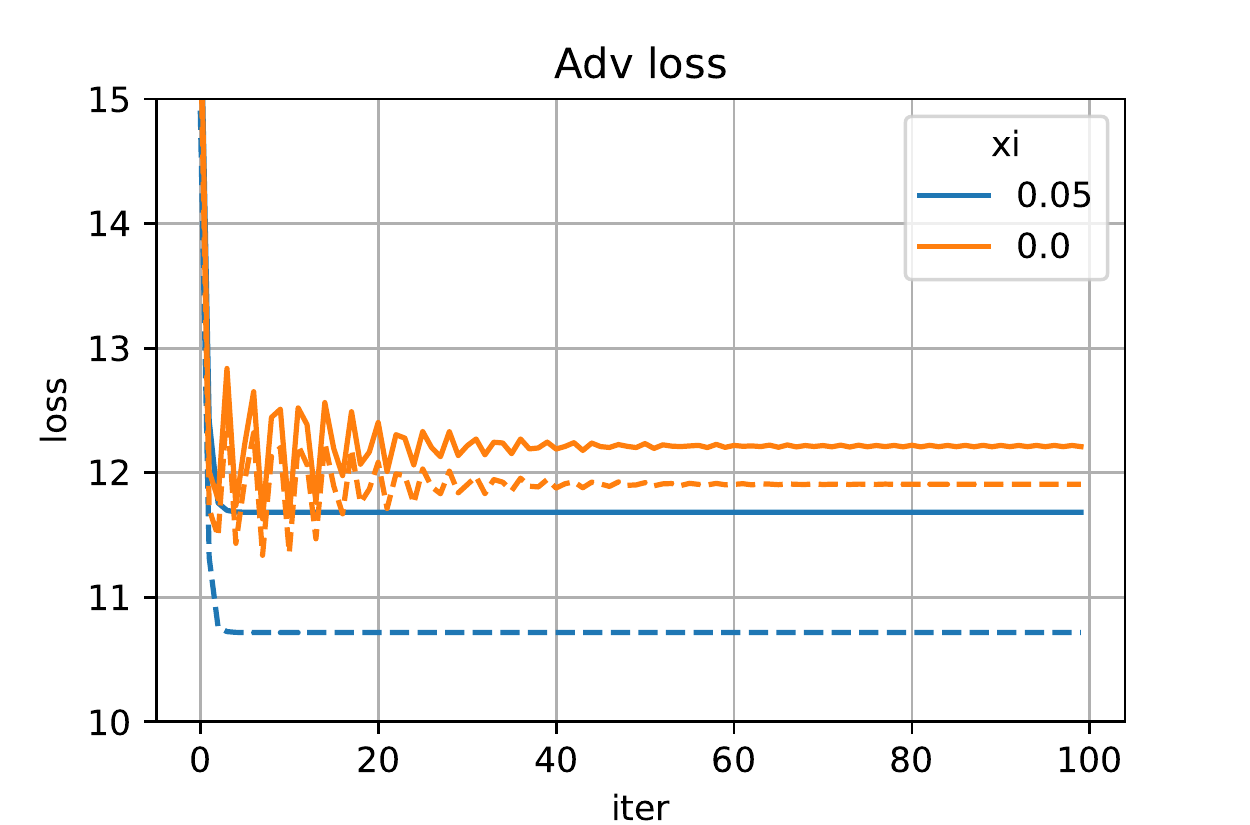}
	    \vspace{-0.12in}
	    \caption{ Adversarial training and testing loss in linear regression. Solid line: adversarial testing loss. Dashed line: adversarial training loss. The training process with $\xi=0$ is more fluctuate. }
	    \label{fig:low_xi}
	\end{figure}
	
	\subsection{Two-layer neural networks}
	We consider the two-layer neural network with an activation function $\phi$, say
	\begin{equation}\label{eqn:nonlinear}
	f_{\theta}(x)=\frac{1}{\sqrt{h}}\sum_{j=1}^h \phi(x^{\top}\theta_j)a_j,
	\end{equation}
	where $a_j$'s are known values and $\theta=(\theta_1,...,\theta_h)\in\mathbb{R}^{d\times h}$ is the parameter to be trained. This ``lazy training'' setup has been widely used in the literature (e.g. \citealp{du2018gradient2,du2018gradient1,arora2019fine,ba2020generalization,allen2020feature}), it eases the theoretical analysis. Additional, we adopt a vanishing initialization scheme \citep{ba2020generalization}: $\theta_{\xi,j}^{(0)}\sim N(0,I_d/dh^{1+\delta})\;\;\mbox{for some $\delta>0$}.$
	
	Theorem \ref{thm:low_nonlinear} shows that $R_{N,\xi}(\theta_{\xi}^{(T)},\epsilon)$ converges to the same minimal loss $R^*$ as in linear regression under proper choice of $T$ and $\xi$, as $n$ and $h$ diverge. 
	
	\begin{theorem}\label{thm:low_nonlinear}
		Under the generative model (\ref{eqn:model}), assume the activation function $\phi$ in model (\ref{eqn:nonlinear}) is twice continuously differentiable, $\phi'(0)\neq 0$, and $\phi(0)=0$. If $a$ and $h$ satisfy {$\|a\|_{\infty}=O(1)$, $\max |a_j|/(\min |a_j|)=\Theta(1)$, $(d\log n)\|a\|_{\infty}v/\sqrt{h}\rightarrow 0$, and $ (d\log n)\|a\|_{\infty}v\sqrt{h}/\|a\|^2\rightarrow 0 $}. When $\xi/v^2\rightarrow 0$ where $\xi/v^2=-\log\log n/\log(\sqrt{d^2\log n/n}\vee (d\log n)\|a\|_{\infty}/\sqrt{h})$, and $\eta=\xi h/(v^2L\|a\|^2)$ for some large constant $L$, with probability tending to 1, for $T=(v^2\log\log n)/\xi$, if $\sqrt{d\log n}(1+v^2\eta\|a\|^3/(h^{3/2}\xi)+L\eta v)^T/h^{\delta/2}\rightarrow 0$, then
		\begin{equation}\label{eqn:low_nonlinear}
		 	\frac{R_{N,\xi}(\theta_{\xi}^{(T)},\epsilon)-R^*}{v^2}\rightarrow 0,
		\end{equation}
		where $R^*$ is the exactly the same as Theorem \ref{thm:opt}.
	\end{theorem}

	The detailed proof is postponed to Appendix D. The choice of $\xi$ here depends on the weights $a$ together with the number of hidden nodes $h$.
	Note that although both Theorems \ref{thm:opt} and \ref{thm:low_nonlinear} establish the convergence of $R_{N,\xi}(\theta,\epsilon)$, Theorems \ref{thm:low_nonlinear} requires that $\xi/v$ converges to zero in a slower speed, leading to a slower convergence rate for $R_{N,\xi}(\theta,\epsilon)$. 
	
	\begin{remark}
The proof of Theorem~\ref{thm:low_nonlinear} is similar to \cite{ba2020generalization}: as the number of hidden nodes $h$ grows, the trajectories of optimization using linear network (with zero initialization) and nonlinear network (with vanishing initialization) are slightly different, while the convergence result of the former one can be simply extended from linear models. Different from \cite{ba2020generalization}, we specify the learning rate as well as the number of iterations as functions of $(d,a,n,h)$, while \cite{ba2020generalization} utilized gradient flow, which is not applicable in our setup. In addition, compared with \cite{ba2020generalization}, the relationship of $(\xi,a,\eta)$ is revealed in our result when $\|\theta_0\|\rightarrow\infty$.    
\end{remark}

        Theorem \ref{thm:low_nonlinear} requires a continuous differentiable $\phi$, and similar results can be established forReLU activation function as well:
		\begin{theorem}\label{thm:low_relu}
		Under the generative model (\ref{eqn:model}), assume the activation function $\phi$ is ReLU function with zero initialization and no bias. Take $\xi/v^2=-\log\log n/\log(\sqrt{(d^2\log n)/n })$. Set $\eta=\xi h/(v^2L\|a\|^2)$ for some constant $L$. Denote $a^+$ as a vector such that $a^+_j=a_j1\{a_j>0 \}$, and similarly define $a^-$. If $a$ satisfy $ \|a^+\|/\|a^-\|=1$ and $\|a\|_{\infty}=O(1)$, with probability tending to 1, for $T=(v^2\log\log n)/\xi$, (\ref{eqn:low_nonlinear}) holds.
	\end{theorem}

	\textbf{Numerical experiments}\\
	A series of simulation studies of low dimensional linear regression and two-layer neural network model with lazy training are conducted. Due to page limit, these results are detailed in Appendix A and successfully validate our Theorems \ref{thm:opt}-\ref{thm:low_relu} on the convergence of adversarial training with surrogate loss.
	
	Here we present another experiment that shows the improvement of predicting the performance of adversarial trained estimator via surrogate loss for complicated models beyond our theorems.
	We fit a ResNet-34 (WideResNet34-1) model for  CIFAR-10 dataset. Since the ReLU function is not smooth, we utilize the training technique introduced by \cite{xie2020smooth}: we use ReLU in the forward path and use Softplus($\beta=10$) in the backward path to improve the gradient quality. The number of epochs is taken as 100. The initial learning rate is 0.1, and at the 75th and 90th epoch, it is multiplied by 0.1. The value of $\xi$ is taken as 0.001 at the initial stage and multiplies 0.1 whenever the learning rate is changed. We repeat this experiment 10 times to obtain the mean and variance and conduct this experiment under various levels of attacks. We use  $\mathcal{L}_2$ attack with strength $\epsilon=3.0$ in this experiment. The results are summarized in Figure \ref{fig:cifar10}. It shows that using surrogate loss leads to slightly higher adversarial testing accuracy and much higher standard testing accuracy than the one with $\xi=0$. 
	
	In addition, we conduct two experiments (in case that the adversarial training with $\xi=0$ does not converges algorithmically in the above experiment): (1) with 200 epochs and (2) with initialization that is obtained by standard training as in \cite{allen2020feature}. The results are postponed to the appendix. In short, a better performance is obtained under surrogate loss, and the initialization from standard training does not improve the performance.

Although our theory only reveals a single non-differentiable point, it is still important to handle this carefully in neural networks.  The non-differentiable problem is partially due to that the adversary has no preference in the direction of attack.
If we estimate the attack twice (from different initializations) and the difference between the two estimates will be large when the non-differentiable problem is severe. We conduct a small experiment to investigate the attack difference. The details are postponed to the appendix. The results justify the importance of accommodating this non-differentiability issue, especially for large $\epsilon$.

	 		\begin{figure}
	    \centering
	    \vspace{-0.1in}
	    \includegraphics[scale=0.5]{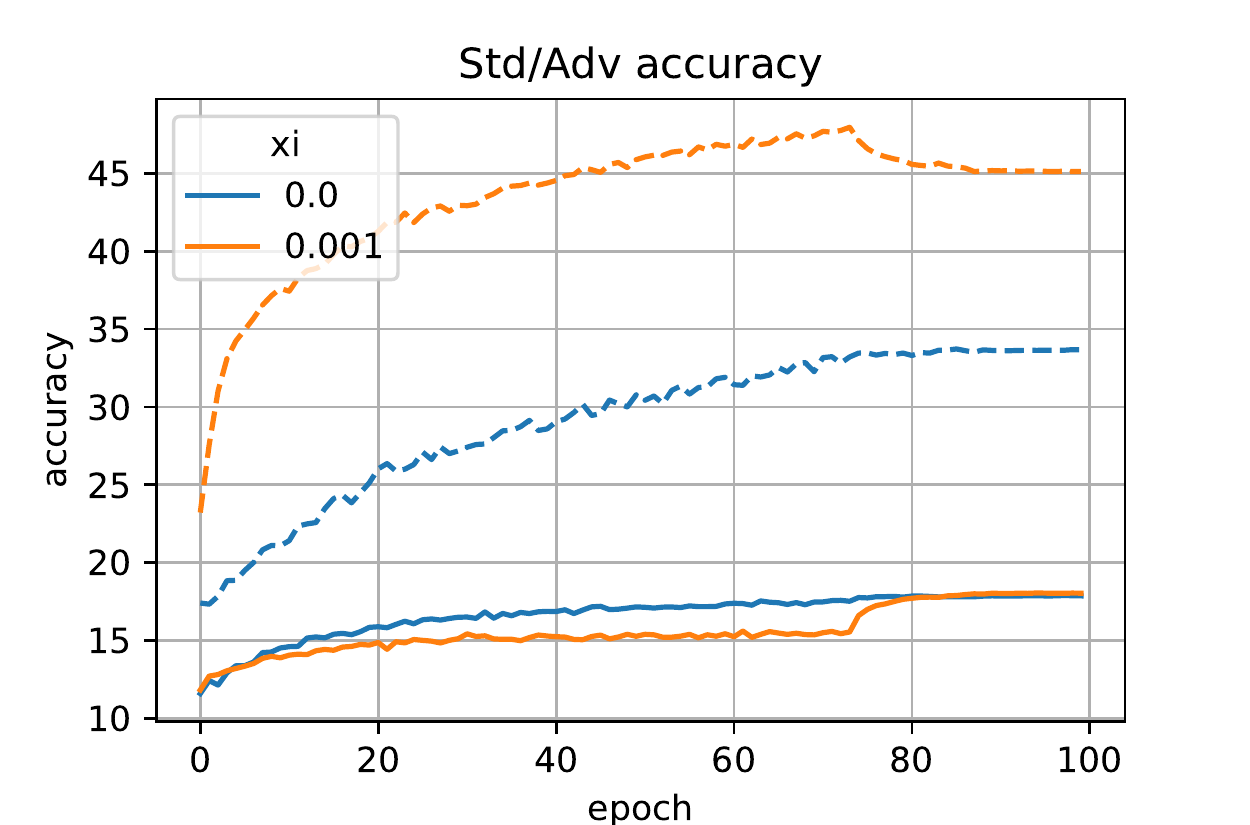}
	    \vspace{-0.15in}
	    \caption{ Standard and adversarial testing accuracies for CIFAR-10 with ResNet34. Solid line: adversarial testing accuracy. Dashed line: standard testing accuracy. With $\xi=0.001$, both adversarial testing accuracy and standard testing accuracy (with standard deviation as 0.458, 0.321) are higher than that for $\xi=0$ (0.743, 0.219).}
	    \label{fig:cifar10}
	\end{figure}

	\section{HIGH DIMENSIONAL ASYMPTOTICS}\label{sec:high}
	In this section, we focus on the high dimensional regime where $d/n\to\infty$. It is first revealed that the adversarial training also suffers from the classical interpolation effect, i.e., near-zero (surrogate) adversarial training loss but high generalization error. As a potential remedy, we penalize the adversarial training loss using LASSO and show that the estimate is consistent when $\theta^*$ is sparse.

\subsection{Effect of interpolation}
It is well known that interpolation may occur under high dimensionality. For instance of linear regression, if a gradient descent with zero initialization is applied to minimize the squared loss when $d/n\rightarrow \infty$, then the solution converges to 
	$$
	\theta(\y):=\X^{\top}(\X\X^{\top})^{-1}\y,
	$$
	where $\X=(x_1,x_2,...,x_n)^{\top}$ and $\y=(y_1,y_2,...,y_n)^{\top}$, given a sufficiently small learning rate. This perfectly interpolated estimator $\theta(\y)$ is proven to be inconsistent to $\theta_0$ and lead to a large generalization error \citep[e.g.,][]{hastie2019surprises,belkin2019two}. Note that this over-fitting scenario is different from the one in \cite{rice2020overfitting}, which is caused by over-parameterization in deep neural networks rather than high dimensionality of the input.
	
	Our first result shows that $\theta(\y)$ also induces the same effect of interpolation in adversarial learning for linear models. 
	\begin{lemma}\label{lem:ridge}
		Assume data generation follows (\ref{eqn:model}). When $\theta(\y)\neq0$ and $d/n\rightarrow\infty$, we have $\|\theta(\y)\|^2/v^2=O_p(n/d),$ and with probability tending to 1, for any $\xi\geq0$, it holds that
		$$\frac{\widehat{R}_{L,\xi}(\theta(\y),\epsilon)}{v^2}\rightarrow 0\;\text{and }\;\frac{{R}_{L,\xi}(\theta(\y),\epsilon)}{v^2}\rightarrow 1.
		$$
	\end{lemma}
We next show that $\theta_{\xi}^{(T)}$ shares the same properties as $\theta(\y)$. The core idea is that the training trajectory $\{\theta_{\xi}^{(t)}\}_{t=1}^T$ can be sufficiently close to that in the standard training, when both are initialized from zero.
Since the latter converges to $\theta(\y)$, the surrogate adversarial training loss and testing loss of $\theta_{\xi}^{(T)}$ act in a similar way as those for $\theta(\y)$ respectively.
	
	\begin{theorem}\label{thm:high_linear}
		Under the same assumptions as in Lemma \ref{lem:ridge}, when $(\log n)\sqrt{n/d}\rightarrow 0$, take $\eta$ small enough such that the largest eigenvalue of $\eta \X^{\top}\X$ is smaller than 1.  Use zero initialization, and denote $T:=\min\{t\in\mathbb{Z}^+: |\|\X\theta_{\xi}^{(t)}-\y\|_2/(v\sqrt{n}) <1/\sqrt{\log n}\} $, then with probability tending to 1, for any $\xi>0$, we have $T<\infty$ and   $$
		\frac{\widehat{R}_{L,\xi}(\theta_{\xi}^{(T)},\epsilon )}{v^2}\rightarrow 0,\;\text{and }\frac{{R}_{L,\xi}(\theta_{\xi}^{(T)},\epsilon )}{v^2}\rightarrow  1.
		$$
	\end{theorem}
	
The proof of Theorem \ref{thm:high_linear} is postponed to Appendix E. Compared with Theorem \ref{thm:opt}, Theorem \ref{thm:high_linear} no longer requires $\xi$ to be associated with $(d,n)$. A crucial reason for this difference is that under high dimensionality, when $t\leq T$, the smoothness of $\widehat R_{L,\xi}$ along the training trajectory (i.e., the gradient of $\widehat R_{L,\xi}(\theta_{\xi}^{(t)},\epsilon)$) is always dominated by a term that is only determined by the eigenvalues of high dimensional design matrix $\X$, regardless of how small $\xi$ is (refer to equation (\ref{eqn:gradient}) in Appendix E for details). This is contrast to the low dimensional case.

Theorem \ref{thm:high_linear} shows that $R_{L,\xi}(\theta_{\xi}^{(T)},\epsilon)/v^2$ does not converge to $R^*/v^2$. Similar results can be established for two-layer neural networks (with lazy training):
	\begin{theorem}\label{thm:high_network}
		For the two-layer neural network (\ref{eqn:nonlinear}), under the same conditions on $\phi$ as in Theorem \ref{thm:low_nonlinear}, $(\log n)\sqrt{d/n}\rightarrow\infty$, take zero/vanishing initialization and $\eta=\eta_{linear}h/\|a\|^2$. Assume { $\|a\|_{\infty}=O(1)$, $(d\log n)\|a\|_{\infty}v/\sqrt{h}\rightarrow 0$, $\max |a_j|/(\min |a_j|)=\Theta(1)$, $\sqrt{dn}(\log n)^2 \|a\|_{\infty}v\sqrt{h}/\|a\|^2 \rightarrow 0 $ }, $\sqrt{d\log n}(1+v^2\eta\|a\|^3/(h^{3/2}\xi)+L\eta v)^T/h^{\delta/2}\rightarrow 0$. Denote $T:=\min\{t\in\mathbb{Z}^+:|\|f_{\theta_{\xi}^{(t)}}-\y\|_2/(v\sqrt{n}) <1/\sqrt{\log n}\} $, then with probability tending to 1, for any $\xi >0$, we have $T<\infty$ and 
		\begin{equation}\label{eqn:high_linear}
		\frac{\widehat{R}_{N,\xi}(\theta_{\xi}^{(T)},\epsilon )}{v^2}\rightarrow 0,\;\text{and }\frac{{R}_{N,\xi}(\theta_{\xi}^{(T)},\epsilon )}{v^2}\rightarrow  1.
		\end{equation}
	\end{theorem}
	
	For ReLU network, we have the following result:
	\begin{theorem}\label{thm:high_relu}
		For the two-layer neural network (\ref{eqn:nonlinear}), $(\log n)\sqrt{d/n}\rightarrow\infty$, assume the activation function $\phi$ is ReLU function with zero initialization and no bias.  If $ \|a^+\|/\|a^-\|=1$,  $\|a\|_{\infty}=O(1)$.  Denote $T:=\min\{t\in\mathbb{Z}^+:|\|f_{\theta_{\xi}^{(t)}}-\y\|_2/(v\sqrt{n}) <1/\sqrt{\log n}\} $, then for any $\xi >0$, with probability tending to 1, we have $T<\infty$, and (\ref{eqn:high_linear}) holds.
	\end{theorem}
	
	\textbf{Numerical experiment}\\
	A simulation is conducted to verify Theorem \ref{thm:high_linear}. We choose $n=20$, $d= 1000$, and $\sigma^2=1$. The true underlying model $\theta_0$ is all-zero except for its first 10 elements being $1$. The attack intensity is $\epsilon=0, 0.01, 0.1$. Learning rate is taken as 0.001 with zero initialization and $\xi=0.5$. The curves in Figure \ref{fig:1000} represent means of respective statistics, and the shaded areas represent mean $\pm$ one standard deviation, based on 100 replications. Figure \ref{fig:1000} shows that the surrogate adversarial training loss keeps decreasing to around zero for all the choices of $\epsilon$, while the adversarial testing loss converges to some nonzero constant. Note that the three adversarial training loss curves in the left plot of Figure \ref{fig:1000} overlap. 

	More experiments for larger $d$ and $\xi=0$ (with a change when $\theta=0$) are postponed to Appendix A. Besides, we also postpone experiments for neural networks to Appendix A.

	\begin{figure}[!ht]
		\centering
		\includegraphics[trim=100 0 400 0,clip,scale=0.4]{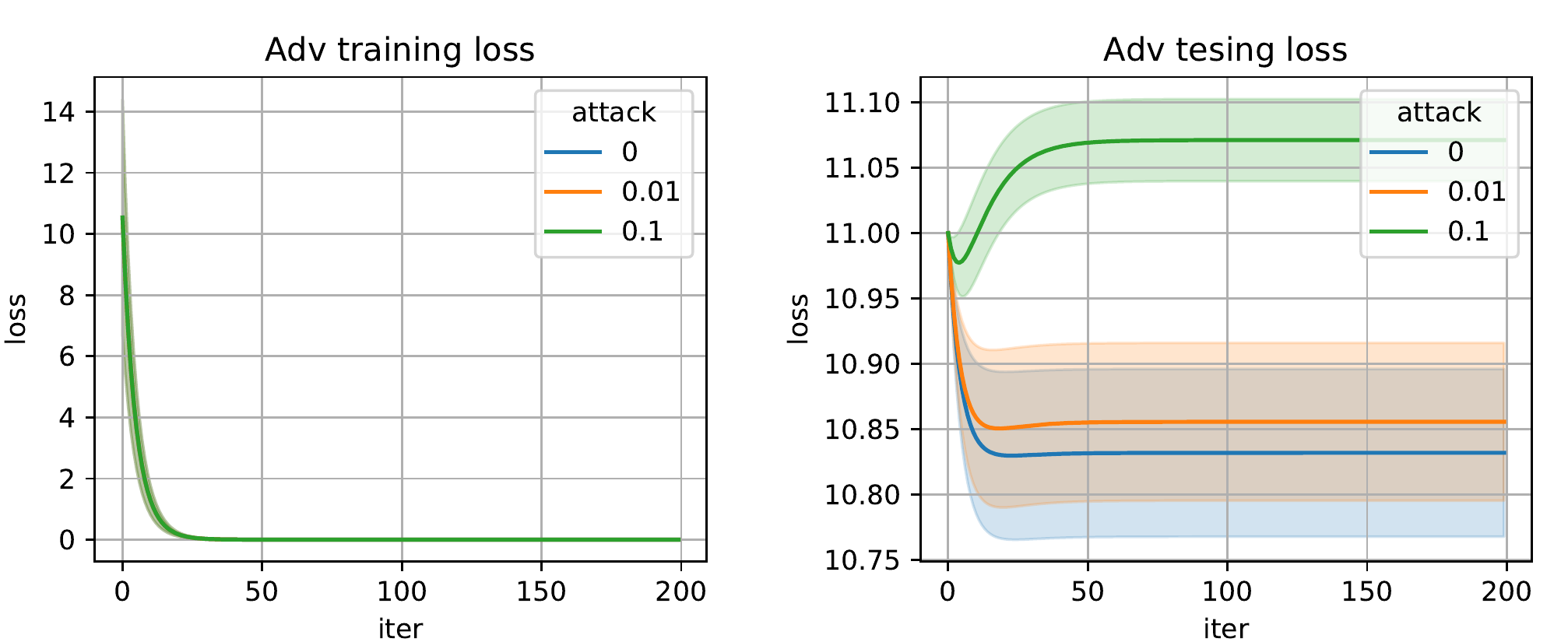}\vspace{-0.1in}
		\caption{Adversarial training, high-dimension.}
		\label{fig:1000}
	\end{figure}

	\subsection{Improving adversarial robustness using LASSO under high dimensionality}\label{sec:sparse}
	
In this section, we explore how incorporating LASSO improves adversarial learning under high dimensionality. In particular, we will present some theoretical justifications for linear models and conduct numerical exploration to evaluate the potential of LASSO in neural networks.
	
Intuitively, a sparse adversarial learning via LASSO makes sense only when the adversarial loss has a sparse global optimization, i.e., $\theta^*$ is a sparse vector. Therefore, certain investigation is necessary to understand the sparsity relation between $\theta_0$ and $\theta^*$.
	
	\begin{prop}\label{prop:sigma}
		Under model (\ref{eqn:model}), the optimal solution of ${R}_{L,\xi}(\theta,\epsilon)$ is of the form
		$
		\theta^*=(\Sigma+\kappa I)^{-1}\Sigma \theta_0
		$
		for some $\kappa$ as a function of $(\theta_0,\Sigma,\sigma^2,\xi)$.  Assuming $\theta_0$ is sparse, then whether the robust coefficient ($(\Sigma+\kappa I)^{-1}\Sigma\theta_0$) is sparse or not depends on $\Sigma$.
	\end{prop}
	
	The following example illustrates that, based on Proposition \ref{prop:sigma}, when the correlation between active set and inactive set is zero, the adversarially robust model $\theta^*$ is sparse as well. 
	\begin{example}\label{example:sigma}
		When $\theta_0$ is sparse, and $\Sigma$ can be represented as $\begin{Bmatrix}
		\Sigma_1 & 0\\
		0 & \Sigma_2
		\end{Bmatrix}$, where $\Sigma_1$ is the covaraince of the active attributes, and $\Sigma_2$ is for the other attributes of $x$, the model $\theta^*$ will be sparse as well.
	\end{example}
	To simplify the derivation, we assume $\Sigma=I$ in the following result. Denote $S$ as the active set of $\theta^*$, and $s=|S|$ as the size of $S$. We consider applying LASSO in the adversarial training loss:
	$$\frac{1}{n}\sum_{i=1}^n l(f_{\theta}(x_i+A_{\epsilon,\xi}(f_{\theta},x_i,y_i)),y_i)+\lambda\|\theta\|_1.$$
	The statistical property of $\widehat{\theta}_{\xi}$, the minimizer of the above objective function, is as follows:
	
	\begin{theorem}\label{thm:lasso}
		Assume data generation follows (\ref{eqn:model}), $\theta_0$ is sparse and $\Sigma=I$. Take $\xi/v^2\rightarrow 0$, $\lambda/v=o(1)$ and $\lambda/v\geq (c\sqrt{{(s\log d)/}{n}})\vee (\xi a_n / v^2)$ for some large constant $c$ and $a_n\rightarrow\infty$. If $\epsilon<\sqrt{\pi}\|\theta_0\|_2/(\sqrt{2}v)$, then $\theta^*\neq 0$, and  with probability tending to 1, we have
		$$
		{\frac{R_{L,\xi}(\widehat\theta_{\xi},\epsilon)-R^*}{v^2}\rightarrow 0,}\;\text{and } \frac{\|\widehat{\theta}_{\xi}-\theta^*\|_1}{\lambda}<\infty.
		$$
	\end{theorem}
The proof of Theorem \ref{thm:lasso} is similar to the traditional LASSO analysis as in \cite{bickel2009simultaneous,belloni2009least} but with an important modification. In the literature, the Hessian of the standard training loss, i.e., $\X^{\top}\X/n$, is usually required to satisfy the so-called restricted eigenvalue condition. However, in adversarial setting, the Hessian changes as $\theta$, so it takes more steps to verify the above condition.

	 \begin{remark}\label{remark}
	 Theorem \ref{thm:lasso} shows the effectiveness of LASSO in sparse linear model and it performs better then the case without LASSO. Note that the estimation consistency still holds for low-dimensional dense model, if $d$ and $n$ satisfies $(d\log d)/n\rightarrow 0$. But, to ensure that LASSO improves the performance in this case, $\lambda$ should be carefully tuned.
	 \end{remark}

We conduct some empirical study to explore the potential applications of LASSO in the adversarial training of neural networks. Similar experiments under large-sample regime can be found for adversarial training \cite{sinha2018certifying,wang2019convergence,raghunathan2019adversarial}, and pruning in adversarial training \cite{ye2019adversarial,li2020towards}.
		
	\begin{figure}[!ht]
		\centering\vspace{-0.05in}
		\includegraphics[trim=30 20 0 30,clip,scale=0.38]{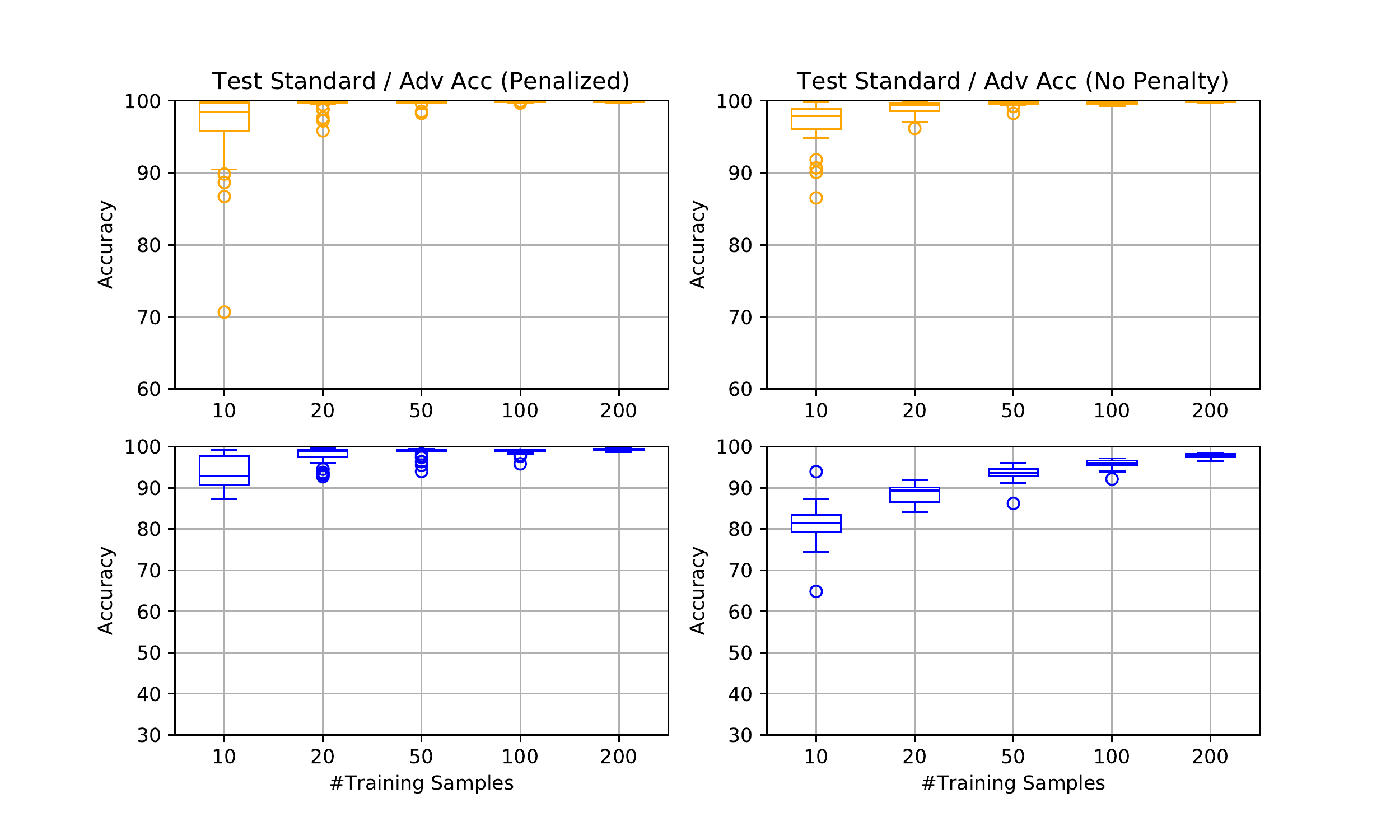}
		\vspace{-0.15in}
		\caption{Comparison on standard (upper) / adversarial (lower) test accuracies between training with/without $\mathcal{L}_1$ penalty under $\mathcal{L}_2$ attack with $\xi=10^{-4}$. Attack strength $\epsilon=3$.}
		\label{fig:l2_xi}
	\end{figure}
		
	\textbf{Numerical experiments} \\
	The program was modified from a repository in Github\footnote{\url{https://github.com/louis2889184/pytorch-adversarial-training}} and library \texttt{Advertorch}. A simple two-layer neural network is constructed with 1024 hidden nodes and ReLU as activation. We use MNIST dataset to distinguish between digits 0 and 1, and randomly select a small number of samples of  0 and 1 from the training dataset to create a high-dimension scenario. The $\mathcal{L}_{2}$ attack level is set to be 3. We trained 2000 epochs to ensure the convergence of the algorithms and repeat the experiment for 30 times to draw a boxplot. After training 2000 epochs, for both $\lambda=0$ (No penalty) and $\lambda=0.001$ (LASSO), the training accuracies for clean data and adversarial data both reach 100\%. The penalty $\lambda=0.001$ was chosen such that the magnitude of penalty is comparable with loss. The results are summarized in Figure \ref{fig:l2_xi}.

	For adversarial accuracy, as shown in Figure \ref{fig:l2_xi}, the results with two different $\lambda$'s are significantly different, where the choice of $\lambda=0.001$ improves the adversarial accuracy compared with $\lambda=0$. As a reference, we also plot the standard accuracy (i.e. prediction accuracy for un-corrupted data), even though the objective function minimized is the (penalized) adversarial training loss. Figure \ref{fig:l2_xi} shows, it approaches 99\% quickly in $n$ for both adversarial training with and without LASSO. 
	
	 We also observe similar results when $\xi=0$, and the details are postponed to Appendix A.

	  We also tried on CIFAR-10 with WideResNet34-10. The $\lambda$ is chosen to ensure that the magnitude of cross entropy loss and the LASSO penalty are comparable. We use all data in the training dataset to conduct this experiment. The results are summarized in Table \ref{tab:cifar10_wide}. Our theorem only concerns the high dimensional case, i.e., small-$n$-large-$d$, however, as showed in Table \ref{tab:cifar10_wide}, both standard and adversarial testing accuracies are still enhanced when using LASSO, in this large-$n$ application (Refer to Remark \ref{remark}).

	\begin{table}[!ht]
	\centering
\begin{tabular}{lll}
\hline Method                & std acc (\%)       & adv acc (\%)        \\\hline
Benchmark             &  84.346(0.355) &	61.760(0.204)  \\
LASSO $10^{-5}$          & 86.568(0.214) & 63.072(0.292) \\\hline
\end{tabular}
\vspace{-0.1in}
  \caption{ Adversarial training in CIFAR-10 using WideResNet34-10 with $\mathcal{L}_2$ attack, $\epsilon=0.5$. }\label{tab:cifar10_wide}
\end{table}

	\section{$\mathcal{L}_{\infty}$ ADVERSARIAL TRAINING}\label{sec:linf}
	In this section, we discuss the adversarial loss and adversarial training under $\mathcal{L}_{\infty}$ attack. Similar as stated in \cite{chen2020more}, the adversarial risk of the linear model becomes
	\begin{equation}
	\begin{split}
	    R_L^{\infty}(\theta,\epsilon)=&\|\theta-\theta_0\|_{\Sigma}^2+\sigma^2+\epsilon^2\|\theta\|_1^2\\
	    &+2\epsilon c_0 \|\theta\|_1\sqrt{\|\theta-\theta_0\|_{\Sigma}^2+\sigma^2 }.
	\end{split}\label{eqn:linf}
	\end{equation}

Below are discussions w.r.t $\mathcal{L}_{\infty}$ adversarial training:

\textbf{Harder to train}\\
From (\ref{eqn:linf}), $R_{L}^{\infty}$ is not differentiable when some element in $\theta$ is zero. Similar as for $\mathcal{L}_2$ attack, we propose to shrink the size of adversarial attack when $R_{L}$ is not differentiable, while a difference is that the shrinkage is applied on each dimension of $x$: for $i=1,...,d$,
	\begin{equation*}
	    [A_{\epsilon,\xi}^{\infty}(f_{\theta},x,y)]_i = \frac{ |[\partial l/\partial x]_i| }{{|[\partial l/\partial x]_i|+\xi }}[A^{\infty}_{\epsilon}(f_{\theta},x,y)]_i.
	\end{equation*}

    A major difference between  $\mathcal{L}_{\infty}$ and $\mathcal{L}_2$ attacks is that, $\mathcal{L}_{\infty}$ attack is more sensitive to $\xi$. For example, if $\theta=(1/d,...,1/d)^{\top}$ and $\epsilon=1/\sqrt{d}$, then $A_{\epsilon}^{\infty}$ becomes $(\epsilon,...,\epsilon)^{\top}$ whose $\mathcal{L}_2$ norm is 1, while $A_{\epsilon,\xi}^{\infty}$ is $(\epsilon/(1+d\xi),...,\epsilon/(1+d\xi))^{\top}$, whose $\mathcal{L}_2$ norm quickly shrinks to zero if $\xi d\rightarrow\infty$. As a result, it is necessary to require that $\xi=o(1/d)$ to avoid overshrinkage of the $\mathcal{L}_{\infty}$ attack. However, as discussed in previous sections, a smaller $\xi$ requires smaller learning rate and more training iterations, thus training under $\mathcal{L}_{\infty}$ attack is more difficult.
    
   \textbf{Effect of interpolation}\\
   In high-dimensional case, adversarial training still suffers from data interpolation: when $\epsilon=O(1/\sqrt{d})$ and $d/n\rightarrow \infty$, the minimal adversarial training loss converges to zero, while the population adversarial loss converges to $v^2$ (recall that $v^2=\|\theta_0\|^2+\sigma^2$). Similar as for $\mathcal{L}_2$ attack, we add LASSO in the $\mathcal{L}_{\infty}$ adversarial training in MNIST and CIFAR-10. The results are summarized in Figure \ref{fig:linf_xi} and \ref{fig:linf} in appendix, as well as Table \ref{tab:cifar10_linf}. For both datasets, LASSO improves both standard and adversarial testing accuracies.
		\begin{table}[!ht]
	\centering
\begin{tabular}{lll}
\hline Method                & std acc (\%)       & adv acc (\%)        \\\hline
Benchmark             & 82.870(0.131) & 50.338(0.315) \\
LASSO $10^{-5}$          & 84.800(0.282) & 54.260(0.376)  \\\hline
\end{tabular}
\vspace{-0.1in}
  \caption{ Adversarial training in CIFAR-10 using WideResNet34-10 with $\mathcal{L}_{\infty}$ attack, $\epsilon=8/255$. }\label{tab:cifar10_linf}
\end{table}

	\begin{remark}
	From the aspect of formulation, $\mathcal{L}_{\infty}$ adversarial loss and LASSO has overlapped effect as both introduce $\mathcal{L}_1$ penalty effect into the loss function, see (\ref{eqn:linf}). However, LASSO and $\mathcal{L}_{\infty}$ are designed for different purposes. From the aspect of loss landscape in deep learning, LASSO does not intend to change the loss landscape near the global minima as the penalty term goes to zero asymptotically, i.e., any global optimum for standard loss are optimum for LASSO problem given infinite training data. On the other hand, for adversarial robustness under $\mathcal{L}_{\infty}$ attack, it aims to select the certain global minima such that the prediction is robust in the nearby region of training samples, and not all minimizers of standard loss are robust to adversarial attack. From this aspect, they can be applied simultaneously. We refer readers to \cite{guo2020connections} for more discussion. 
\end{remark}	
	\section{CONCLUSION AND FUTURE WORKS}
This paper studies the convergence properties of adversarial training in linear models and two-layer neural networks (with lazy training). In the low-dimensional regime, using adversarial training with surrogate attack, the adversarial risk of the trained model converges to the minimal value. In a high-dimensional regime, data interpolation causes the adversarial training loss close enough to zero, while the generalization is poor. One potential solution is to add $\mathcal{L}_1$ penalty in the adversarial training, which results in both consistent adversarial estimate and risk in high dimensional sparse models.
    
    There are several future directions. First, we may focus on classification tasks as a future work. In regression, the adversarially robust model generally outputs smaller-in-magnitude predictions, which is not practical in classification. One may be interested in how adversarial training works in classification. Second, the scenarios we consider are $d/n\rightarrow 0$ and $\infty$, and one can consider the linear dimensionality case, i.e. $d/n\rightarrow c$, as a future direction. Finally, the non-smoothness issue happens to the adversarial loss and the penalty term (e.g., LASSO, \citealp{wang2019improving,wu2020revisiting}), so there is potential to improve further the gradient quality of penalized adversarial training via smoothing the penalty function.

\section*{Acknowledgements}
Dr. Song’s research activities are partially supported by National Science Foundation DMS-1811812.

	\bibliographystyle{asa}
	\bibliography{VaRHDIS}

\newpage
\onecolumn
	\appendix
		The structure of appendix is as follows.  In Section A, we provide more numerical experiments. Section B presents the proof of Theorem \ref{thm:opt}. Section C presents the proof of Theorem \ref{thm:low_nonlinear}, \ref{thm:low_relu},  \ref{thm:high_network} and \ref{thm:high_relu}. Section D shows the proof for Lemma \ref{lem:ridge} and Theorem \ref{thm:high_linear}. And finally Section E is for high-dimensional sparse model (Theorem \ref{thm:lasso}).

	\section{More numerical results}
	\subsection{Low-dimensional linear models}\label{sec:low_linear}
	To verify Theorem \ref{thm:opt} and the statement that ``adversarial training hurts standard testing performance", we run a linear model this experiment. The model is set to be $d=10$ with $\theta_{0,i}=1$ for $i=1,...,10$. The covariance $\Sigma$ is $I$, and for noise, $\sigma^2=1$.
	
	For adversarial training, we use zero initialization, $\xi=0.1$, and $\eta=0.01$. We repeat 100 times to get mean and standard deviation. The results are summarized in Figure \ref{fig:low_linear}. From Figure \ref{fig:low_linear}, one can find that the adversarial testing loss is closed to $R^*(\epsilon)$, while the standard testing loss is away from 1 when $\epsilon=0.5$.
	\begin{figure}[!ht]
		\centering
		\includegraphics[trim=100 0 0 0,clip,scale=0.43]{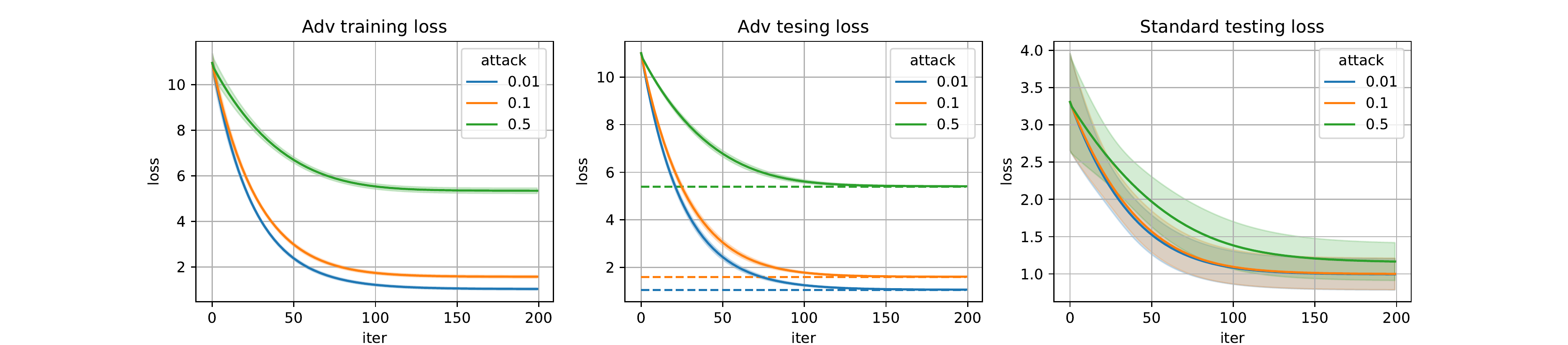}
		\caption{Adversarial training in linear regression under low data dimension. Dashed line in the middle panel: $R^*(\epsilon)$.}
		\label{fig:low_linear}
	\end{figure}
	
	\subsection{Low-dimensional two-Layer networks}\label{sec:low_nn}
	Here we present a numerical experiment to verify our results in Theorems \ref{thm:low_nonlinear}, and  \ref{thm:low_relu} on lazy training. 
	
	We take $d=3$,  $\theta_0=(1/\sqrt{3},1/\sqrt{3},1/\sqrt{3})$, $\Sigma= I$, $\sigma^2=1$ for the data generation model, and $n=100$. For the two-layer neural network, we take $h=50$, and $a_j\sim \text{Unif}[-1,1]$. For a network in Theorem \ref{thm:low_nonlinear}, we take $\phi(x)=1/(1+e^{-x})-1/2$ and $\eta=0.2$. To match the same $\phi'(0)$ and learning rate for all three models, we take $\phi(x)=x/4$ and $\eta=0.2$ for linear network, and $\phi(x)=x1\{x>0\}/4$ and $\eta=0.8$ for ReLU network. For ReLU, we adjust negative $a_j$'s so that $\|a^+\|=\|a^-\|$.
	
	For initialization, we take $\delta=0.5$. For nonlinear networks, we use fast gradient method to approximate $A_{\epsilon,\xi}$ for both training and testing, and for surrogate loss, we take $\xi=0.01$. We run the optimization for 4000 iterations, and repeat 50 times to get mean and standard deviation for the (population) adversarial risk. To estimate the adversarial risk, we randomly simulate 10000 samples and calculate the sample adversarial loss.
	
	The results are shown in Figure \ref{fig:low}. Since we match $\phi'(0)$ and learning rate for all the three networks, the adversarial risks decrease in the same speed and all converges to $R^*(\epsilon)$. However, due to the existence of $\xi$, linear network cannot reach an adversarial testing loss as $R^*(\epsilon)$. For finite $h$, sigmoid networks and ReLU networks have higher adversarial testing loss than linear networks. 
	\begin{figure}[!ht]
		\centering
		\includegraphics[trim=100 0 0 0,clip,scale=0.43]{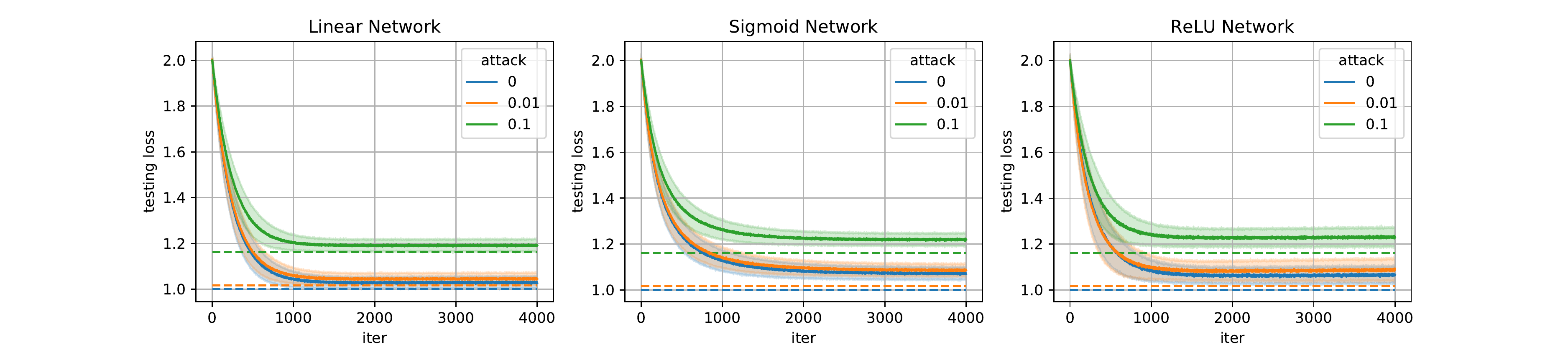}
		\caption{Adversarial training in three two-Layer neural networks (with lazy training) under low data dimension. Dashed line : $R^*(\epsilon)$.}
		\label{fig:low}
	\end{figure}
	
	\subsection{CIFAR-10}
	
	As mentioned in the main text, we conduct some additional experiments. In the first experiment, we run 200 epochs of adversarial training with $\mathcal{L}_2$ attack $\epsilon=3.0$. The initial learning rate is 0.1, and multiplies 0.1 at the 100th and 150th epoch. The value of $\xi$ is initialized as 0.001 and changes according to the learning rate. The first 74 epochs are the same as those in Figure \ref{fig:cifar10}, and we display the remaining epochs in Figure \ref{fig:long}. From Figure \ref{fig:long}, the performance of $\xi>0$ is still better than the case of $\xi=0$ for both adversarial and standard testing accuracy.
	\begin{figure}
	    \centering
	    \includegraphics[scale=0.6]{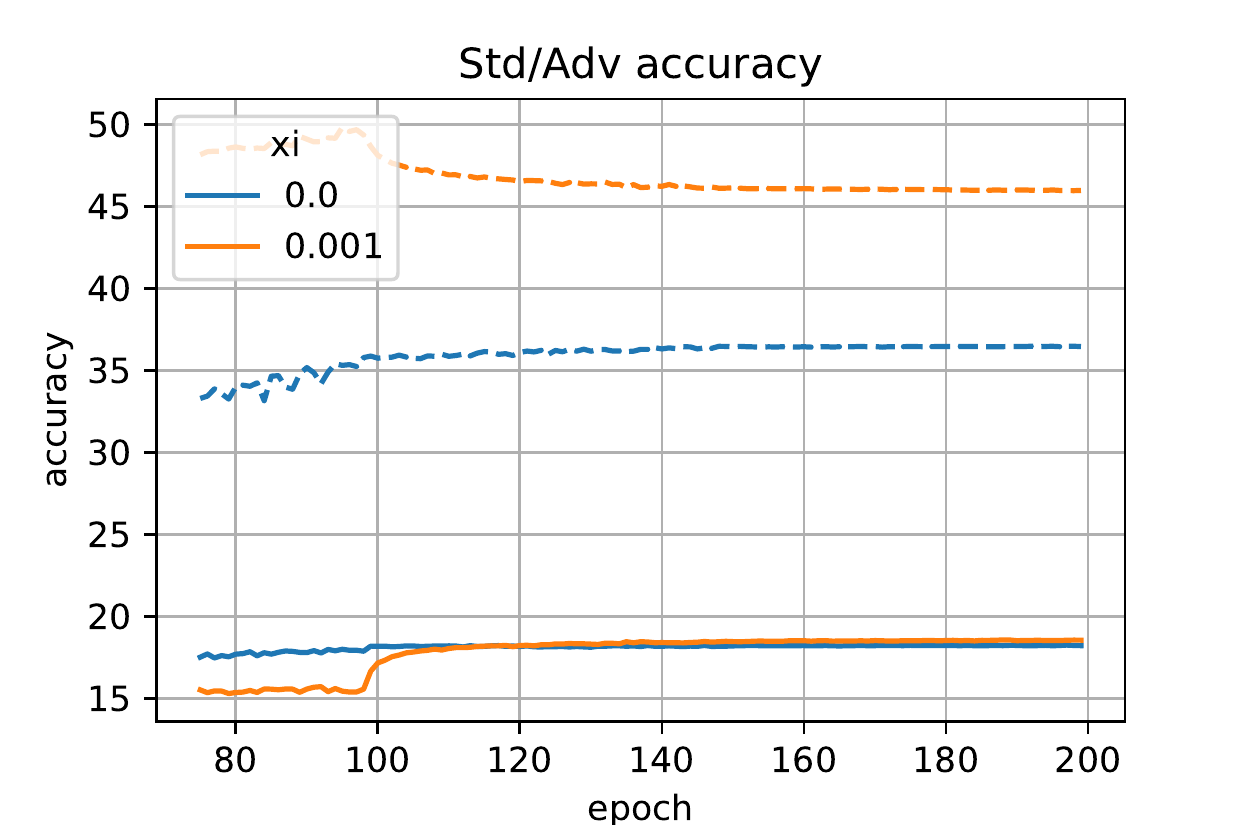}
	    \caption{Adversarial training in CIFAR-10.}
	    \label{fig:long}
	\end{figure}
	
	In the second experiment, besides the 100 epochs of adversarial training, we use standard training to first train 50 epochs. The final standard/adversarial testing accuracy for $\xi=0$ is 33.347(1.257)/17.285(0.241), and the one w.r.t. $\xi=0.001$ becomes 10.313(0.532)/29.762(2.767), while training only in 100 epochs of adversarial training with $\xi=0$ is 33.637(0.724)/17.806(0.241). To conclude, using additional standard training at the beginning may lead to a volatile training process when $\epsilon$ is large.

In the last experiment, we evaluate the attack difference. We use WideResNet34-1 in CIFAR-10 for $\mathcal{L}_2$ PGD-5 attack of strength 0.5 and 3.0 and calculate the attack difference. Attack step size is taken as $2\epsilon/k$ for PGD-$k$. The results are shown in Figure \ref{fig:my_label}. When $\epsilon=0.5$, the attack difference is around 0.16 to 0.17. When $\epsilon=3.0$, the attack difference slowly increases to 1.5 in the end. As a result, a larger attack leads to a larger attack difference, which indicates that potential improvements should be considered to stabilize the training process when $\epsilon$ is large. A similar observation can be found when using PGD-20. 
Despite of that it is difficult to mathematically characterize all non-differentiable points for DNN adversarial loss, the above simulations justify the importance of accommodating this non-differentiability issue, especially for large $\epsilon$.
\begin{figure}
  \begin{center}
    \includegraphics[scale=0.6]{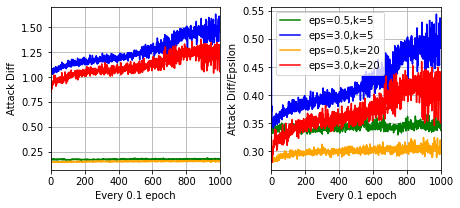}
  \end{center}
\vspace{-0.1in}
  
  \caption{Attack difference in $\mathcal{L}_2$ adversarial training. Both the attack difference and ``relative" attack difference (dividing $\epsilon$) for $\epsilon=3.0$ are larger than the ones in $\epsilon=0.5$. }\label{fig:my_label}
\end{figure}

	\subsection{High-dimensional dense models}
	\subsubsection{Linear model}
	Besides the experiment in Figure \ref{fig:1000}, we further run some experiments with larger $d$  and $\xi=0$.
	\begin{figure}[!ht]
		\centering
		\includegraphics[trim=100 0 400 0,clip,scale=0.43]{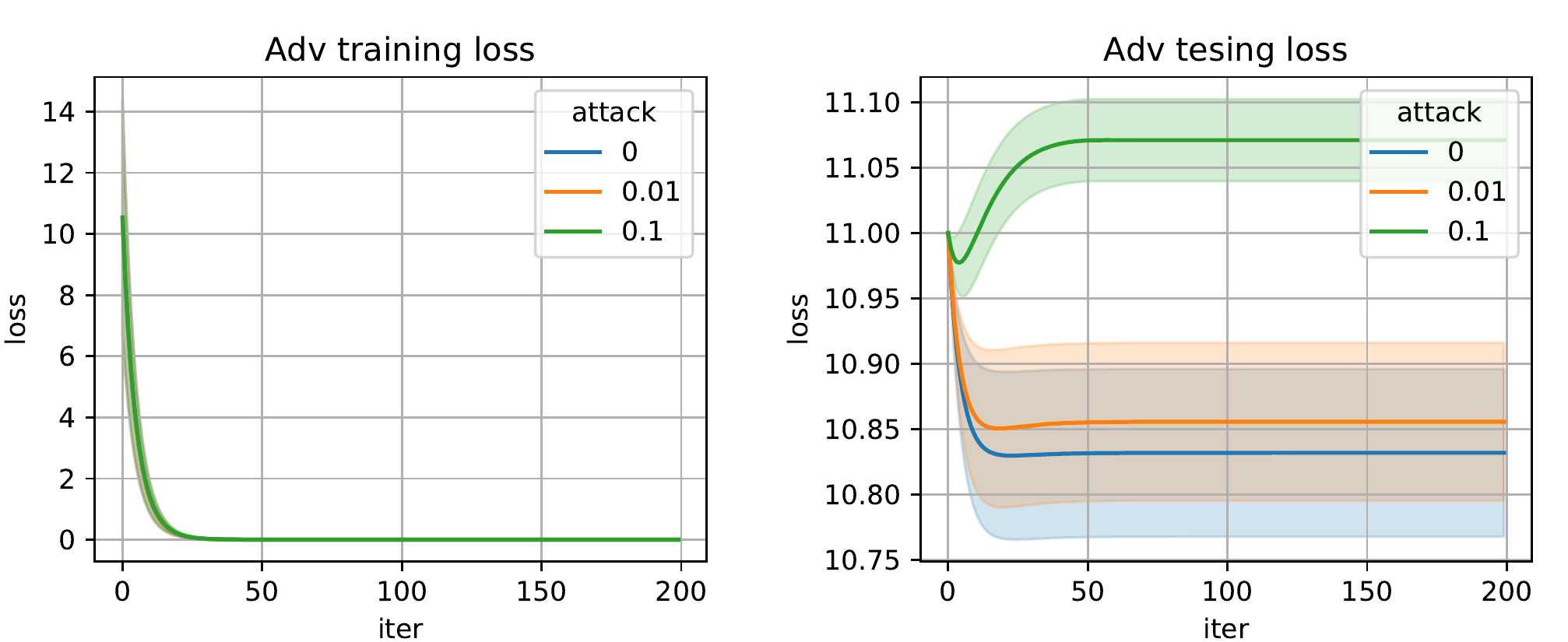}
		\caption{Adversarial training in high-dimensional setup. $\xi=0$}
		\label{fig:xi}
	\end{figure}
	Figure \ref{fig:xi} shows the experiment with the same setting as in Figure \ref{fig:1000} but with $\xi=0$. When $\theta=0$, since the adversarial training loss is not differentiable, we does not impose attack. From Figure \ref{fig:xi}, the adversarial training / testing loss have similar performance as when $\xi=0.1$, while the difference between the gradients of adversarial training and standard training becomes larger than the case when $\xi>0$. To explain this, since $\|\theta_{\xi}^{(t)}\|\rightarrow 0$ when $d/n\rightarrow\infty$, the introduction of positive $\xi$ will leads to a surrogate attack with strength almost zero.
	\begin{figure}[!ht]
		\centering
		\includegraphics[trim=100 0 400 0,clip,scale=0.43]{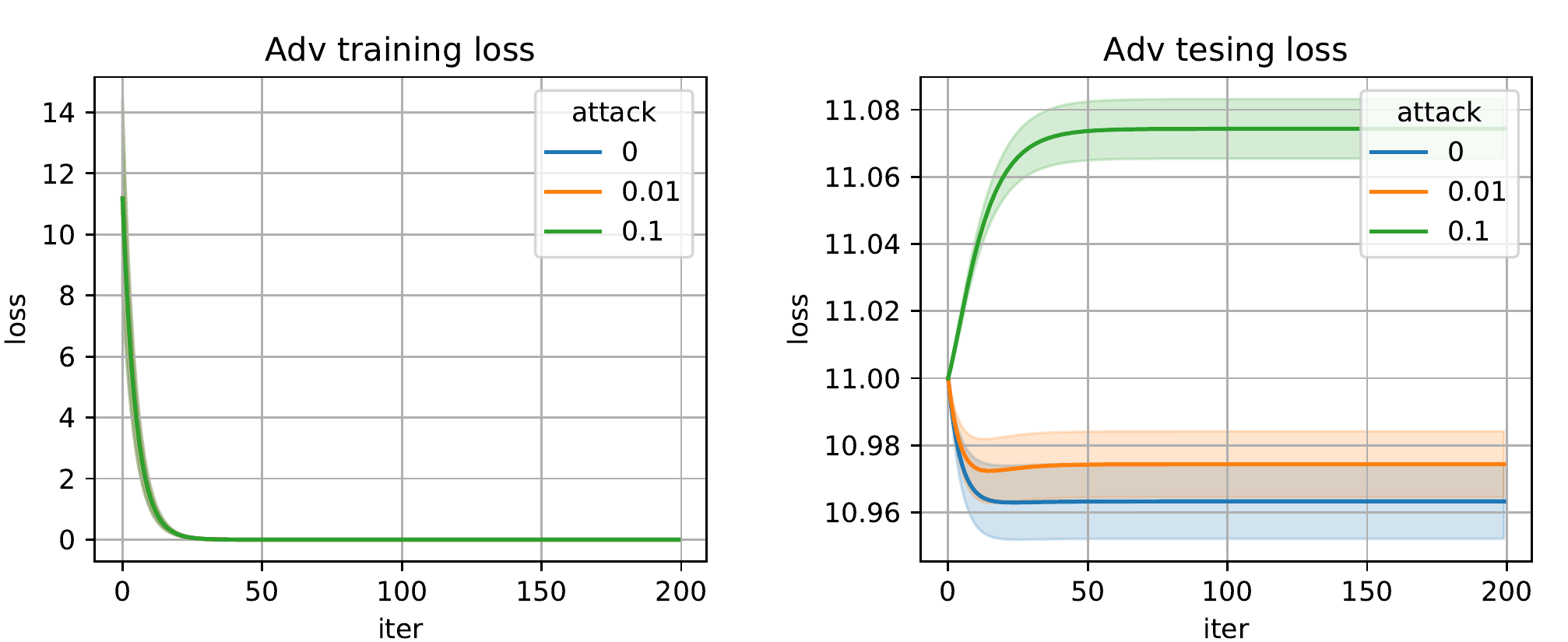}
		\caption{Adversarial training in high-dimensional setup. $\xi=0.5$, $d=5000$}
		\label{fig:5000_001}
	\end{figure}
	\begin{figure}[!ht]
		\centering
		\includegraphics[trim=100 0 400 0,clip,scale=0.43]{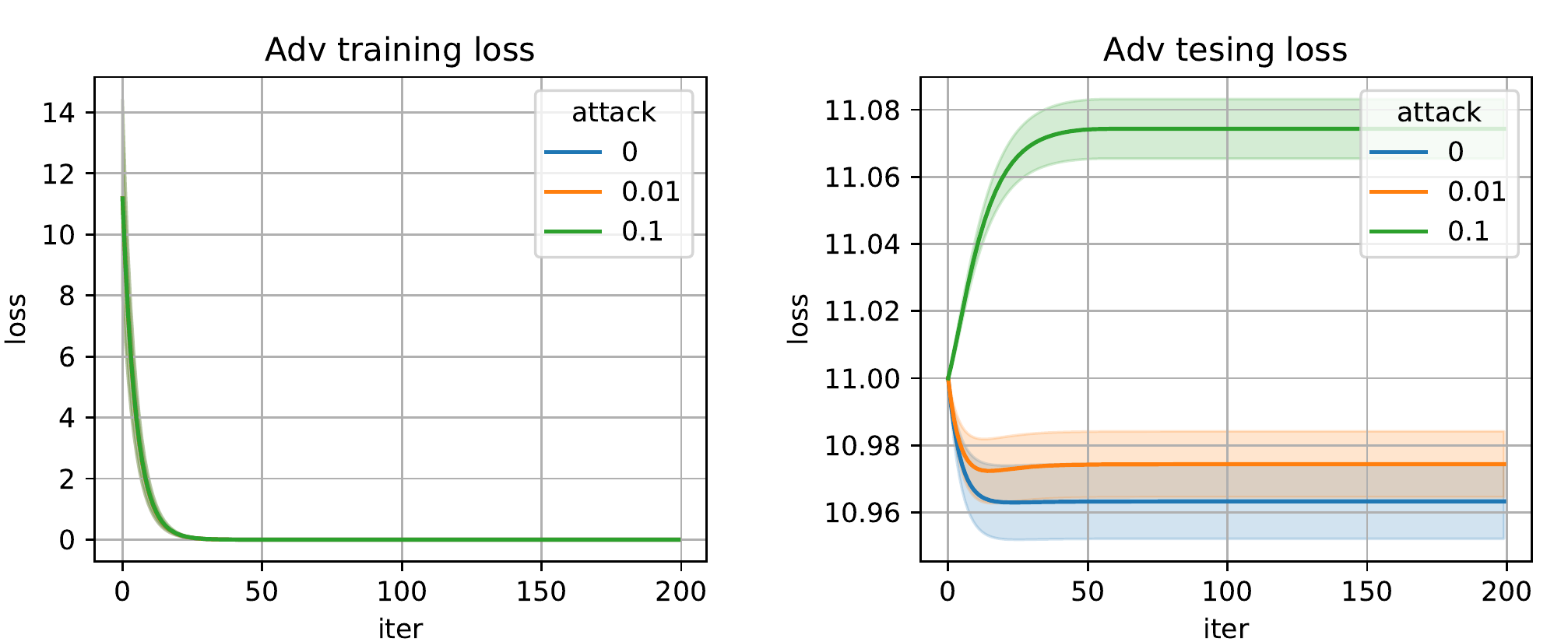}
		\caption{Adversarial training in high-dimensional setup. $\xi=0$, $d=5000$}
		\label{fig:5000}
	\end{figure}
	
	Figures \ref{fig:5000_001} and  \ref{fig:5000} show the experiment with the same setting as in Figure \ref{fig:1000} but with $\eta=0.0002$ and $d=5000$. In addition to the adversarial training / testing loss, from Figures \ref{fig:5000_001} and  \ref{fig:5000}, one can observe that the difference between the gradients of adversarial training and standard training gets smaller. 
	
	\subsubsection{Two-layer neural networks}
	Similar as in Section \ref{sec:low_nn}, we conduct experiment on neural network with three different activation functions. The data generation follows those for Figure \ref{fig:1000} with $\sigma^2=0.1$, $d=1000$, $n=20$. For neural networks, we take $h=50$, and $a_j\sim \text{Unif[-1,1]}$. For ReLU, we adjust negative $a_j$'s so that $\|a^+\|=\|a^-\|$. Initialization takes $\delta$ such that $h^{\delta}=d^{0.6}$. The learning rate is taken as 0.16 for linear and sigmoid networks, and 0.64 for ReLU network so that the convergence pattern is clear in the first 100 iterations.  For adversarial surrogate loss, we take $\xi=0.1$. 
	
	The results are summarized in Figure \ref{fig:high_train} for adversarial training loss and \ref{fig:high_test} for adversarial testing loss. For all three neural networks, the adversarial training loss decreases as fast as standard loss, while the adversarial testing loss are as higher than 1.

	\begin{figure}[!ht]
		\centering
		\includegraphics[trim=100 0 0 0,clip,scale=0.43]{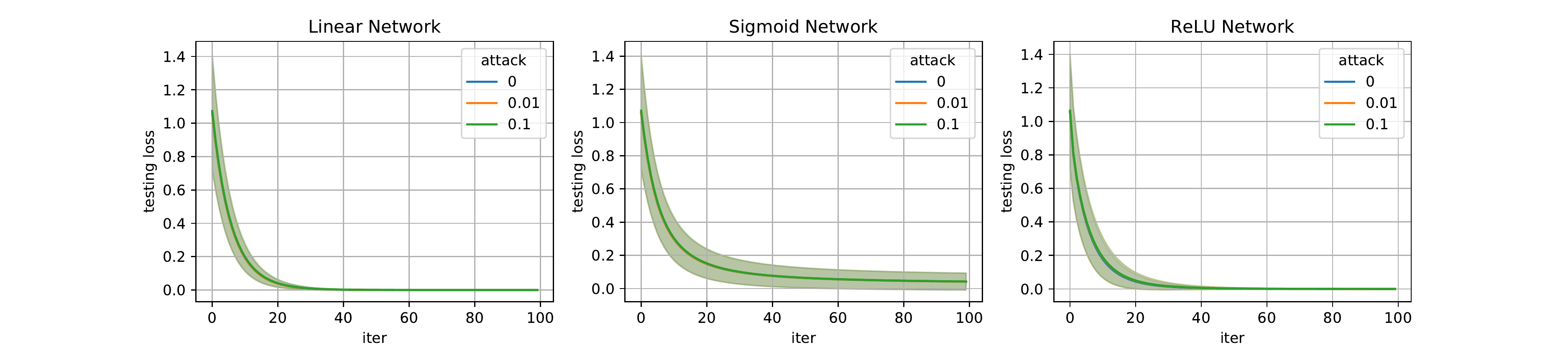}
		\caption{Adversarial training loss of adversarial training in three two-Layer neural networks (with lazy training) under high data dimension.}
		\label{fig:high_train}
	\end{figure}
	\begin{figure}[!ht]
		\centering
		\includegraphics[trim=100 0 0 0,clip,scale=0.43]{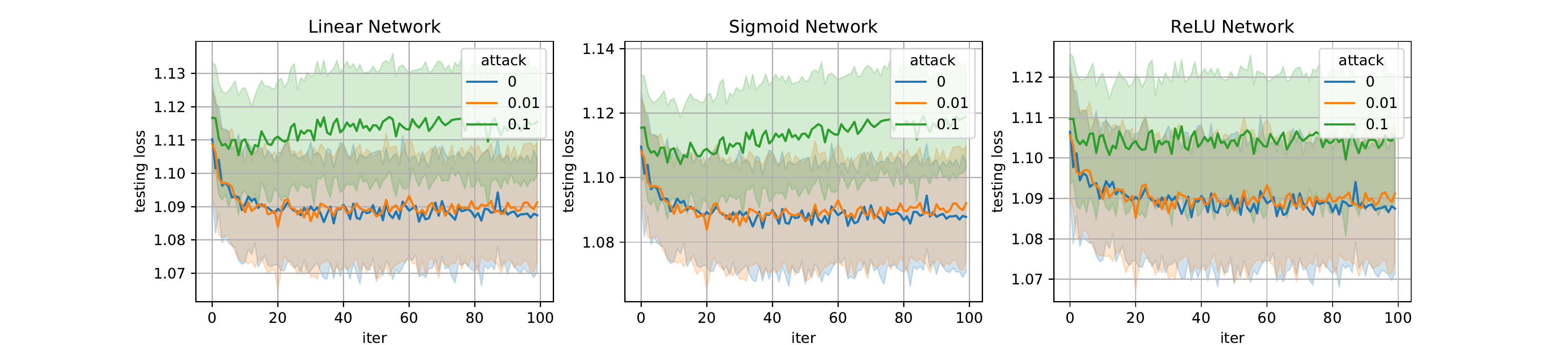}
		\caption{Adversarial testing loss of adversarial training in three two-Layer neural networks (with lazy training) under high data dimension.}
		\label{fig:high_test}
	\end{figure}
	
	\subsection{High-dimensional sparse models}
	In addition to the $\mathcal{L}_{2}$ attack with $\xi=0$ as in Figure \ref{fig:l2_xi}, we also conduct experiments of $\mathcal{L}_{\infty}$ attack with $\xi=0$ and the two attacks with $\xi=1e-04$ below. When $\xi>0$, the adversarial testing accuracy is based on the original attack, i.e. not the surrogate one. As mentioned in Figure \ref{fig:xi}, under high-dimensional setup, a positive $\xi$ leads to the adversarial training getting closer to standard training when $n/d\rightarrow 0$. As a result, we choose a small enough $\xi$ such that the adversarial testing performance is closed to the case when $\xi=0$. Similar as Figure \ref{fig:l2_xi}, all the experiments in Figures \ref{fig:l2}, \ref{fig:linf_xi}, and \ref{fig:linf} shows a better performance using LASSO. 
	
	\begin{figure}[!ht]
		\centering
		\includegraphics[scale=0.5]{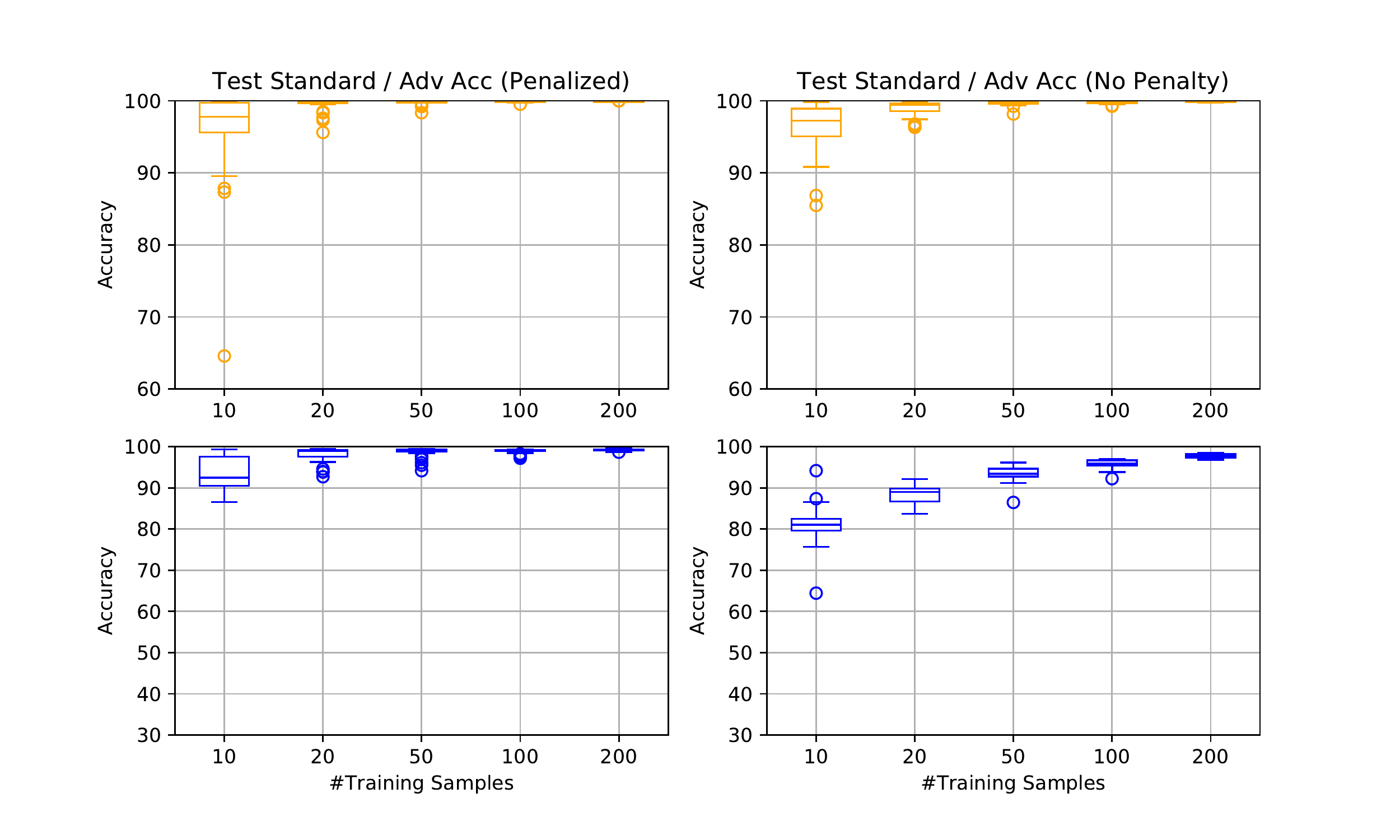}
		\caption{Comparison on standard (upper)/adversarial (lower) test accuracies between training with/without $\mathcal{L}_1$ penalty under $\mathcal{L}_2$ attack. Attack strength $\epsilon=3$.}
		\label{fig:l2}
	\end{figure}
			\begin{figure*}[!ht]
		\centering
		\includegraphics[scale=0.5]{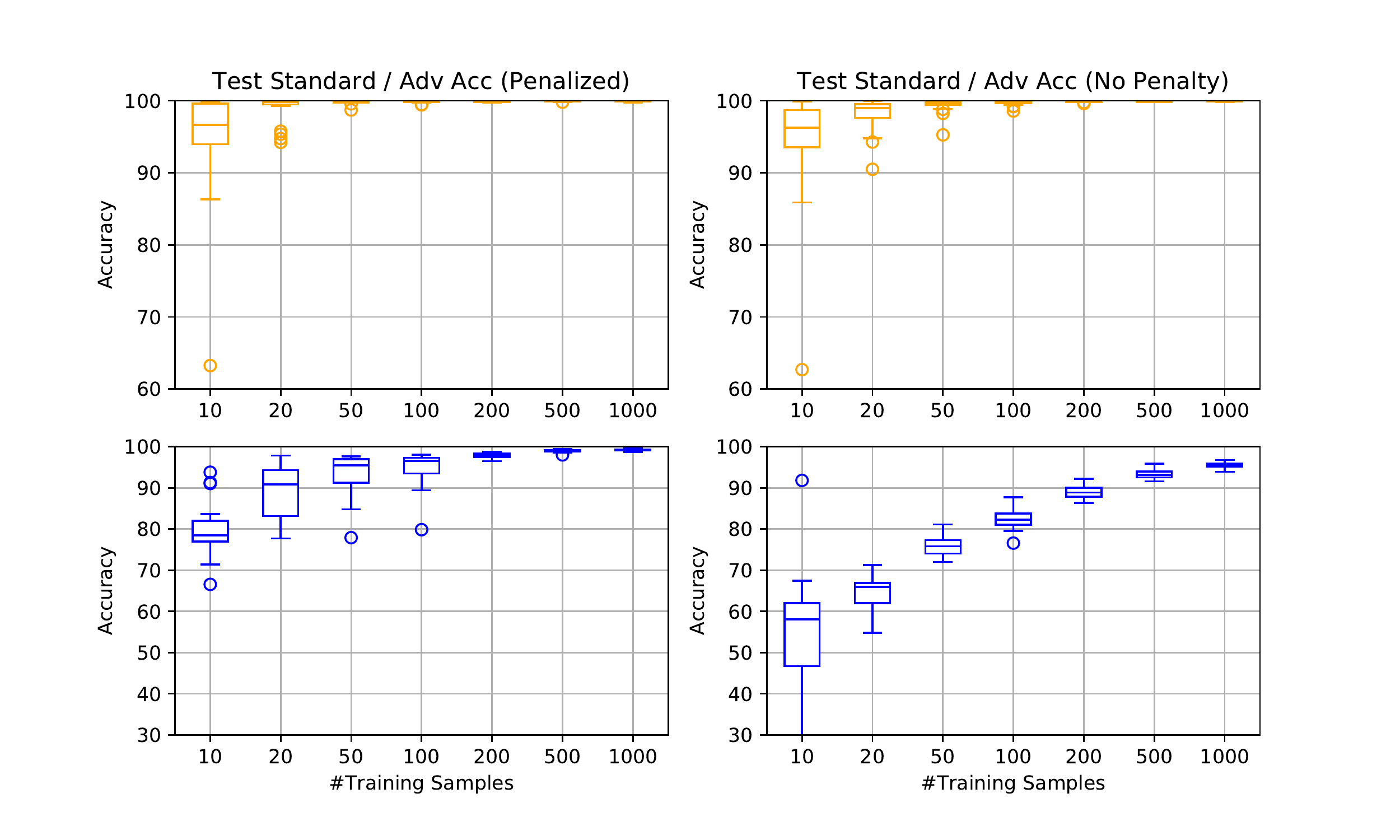}
		\caption{Comparison on standard (upper)/adversarial (lower) test accuracies between training with/without $\mathcal{L}_1$ penalty under $\mathcal{L}_\infty$ attack with $\xi=1e-04$. Attack strength $\epsilon=0.3$.}
		\label{fig:linf_xi}
	\end{figure*}

	\begin{figure}[!ht]
		\centering
		\vspace{-0.15in}
		\includegraphics[trim=0 0 0 20,clip,scale=0.5]{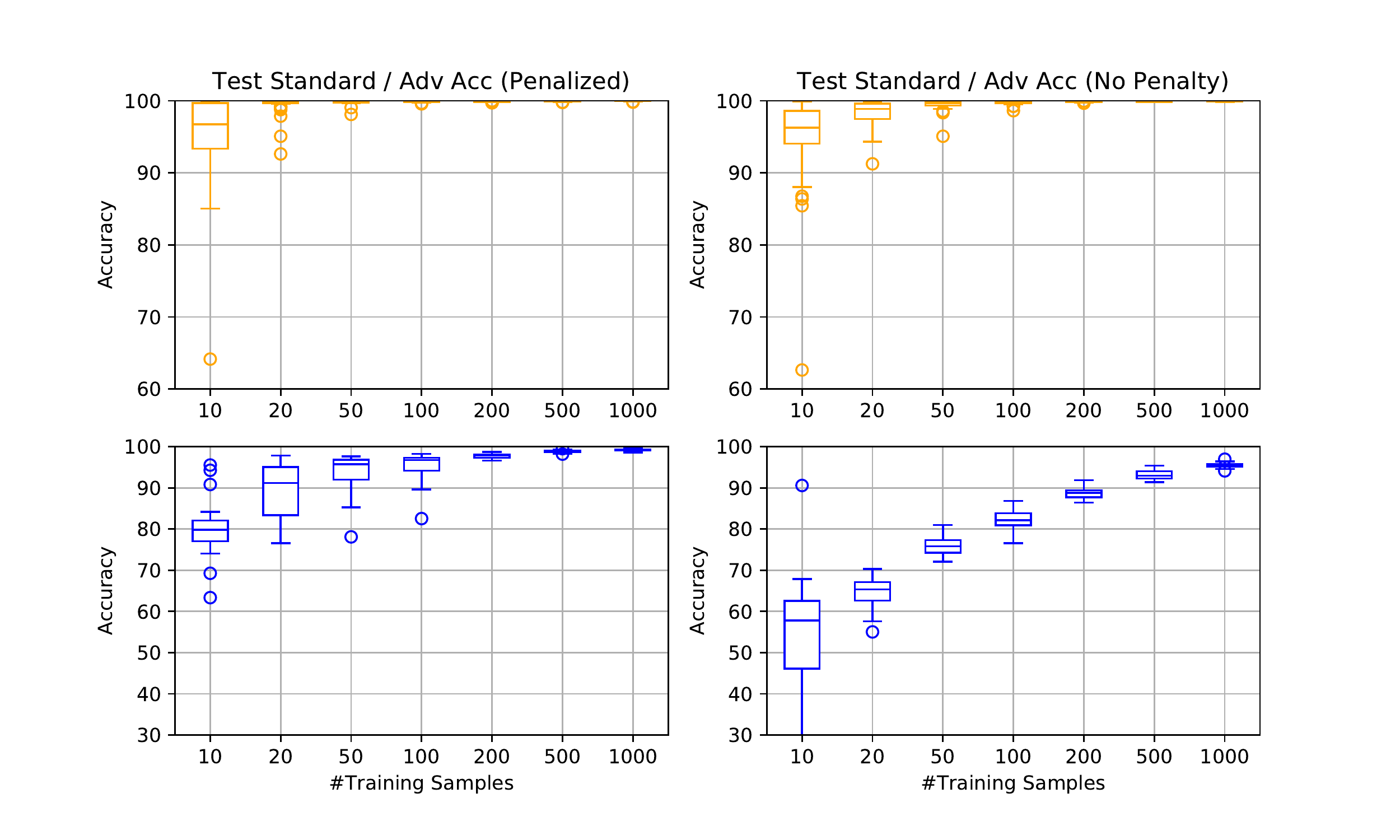}
		\vspace{-0.2in}\caption{Comparison on standard (upper)/adversarial (lower) test accuracy between training with/without $\mathcal{L}_1$ penalty under $\mathcal{L}_{\infty}$ attack.} 
		\label{fig:linf}
	\end{figure}

\section{Proofs for low dimension linear model}
	In the proof of Theorem \ref{thm:opt}, we assume $v^2$ is a constant number, which implies that $\|\theta^*\|\leq B_0$ for some $B_0$. After investigating the results for bounded $v$, we then use a trick to extend to the case when $v$ is changing.

	\begin{lemma}\label{lem:gradient}
		Under the model in (\ref{eqn:model}), when $\|\theta_0\|\leq B_0$  for some constant $B_0>0$, there exists some  function $g_1(\delta,\xi,d,n)$ such that
		\begin{equation*}
		P\left(\sup_{\|\theta\|\leq B_0}\left\|\triangledown \widehat R_{L,\xi}(\theta,\epsilon)-\triangledown R_{L,\xi}(\theta,\epsilon)\right\|>\delta\right)\leq e^{-g_1(\delta,\xi,d,n)},
		\end{equation*}
		where $\triangledown \widehat{R}_{{L,\xi}}(\theta,\epsilon)$ and $\triangledown R_{{L,\xi}}(\theta,\epsilon)$ denote the gradient of $\theta$ \textbf{after fixing} $A_{\epsilon,\xi}$ (so it is not $\partial \widehat{R}_{{L,\xi}}/\partial \theta$ or $\partial R_{{L,\xi}}/\partial \theta$ if $\xi>0$). 
	\end{lemma}

	\begin{proof}[Proof of Lemma \ref{lem:gradient}]
		Assume $\theta\in B(0,B_0)$. For any sample $i$, the gradient of $\theta$ on surrogate adversarial loss is (we ignore the constant multiplier of $\xi^2/|x_i^{\top}\theta-y_i|^2$)
		\begin{eqnarray*}
			&&2x_i(x_i^{\top}\theta-y_i)+2\epsilon \frac{\theta}{\sqrt{\|\theta\|^2+\xi^2/|x_i^{\top}\theta-y_i|^2}}|x_i^{\top}\theta-y_i|\\
			&&+2\epsilon x_i\frac{\|\theta\|^2}{\sqrt{ \|\theta\|^2+\xi^2/|x_i^{\top}\theta-y_i|^2 }}\sgn(x_i^{\top}\theta-y_i)+2 \epsilon^2\theta\frac{\|\theta\|^2}{\|\theta\|^2+\xi^2/|x_i^{\top}\theta-y_i|^2}.
		\end{eqnarray*}
		Therefore, by Bernstein inequality, for any fixed $\theta$,
		\begin{equation*}
		P\left( \left\|\triangledown \widehat R_{L,\xi}(\theta,\epsilon)-\triangledown R_{L,\xi}(\theta,\epsilon) \right\|>\delta\right)\leq e^{-c_1 n \delta^2/d}.
		\end{equation*}
		Our aim is to figure out a bound for $\sup_{\|\theta\|\leq B_0}\left\|\triangledown \widehat R_{L,\xi}(\theta,\epsilon)-\triangledown R_{L,\xi}(\theta,\epsilon)\right\|$, thus we consider the following decomposition:
		\begin{eqnarray*}
		\left\|\triangledown \widehat R_{L,\xi}(\theta,\epsilon)-\triangledown R_{L,\xi}(\theta,\epsilon) \right\|&\leq& \left\|\triangledown \widehat R_{L,\xi}(\theta,\epsilon)-\widehat\triangledown R_{L,\xi}(\theta_k,\epsilon) \right\|\\
		&&+\left\|\triangledown R_{L,\xi}(\theta,\epsilon)-\triangledown R_{L,\xi}(\theta_k,\epsilon) \right\|\\
		&&+\left\|\triangledown \widehat R_{L,\xi}(\theta_k,\epsilon)-\triangledown R_{L,\xi}(\theta_k,\epsilon) \right\|,
		\end{eqnarray*}
		where $\theta_k$ is an element in the fixed sequence $\{\theta_j\}_{j=1,...}$. Now we first introduce how to design $\{\theta_j\}_{j=1,...}$. For the ball $B(0,B_0)$, we use balls with radius $1/M$ to cover it. Then there are total $c_dB_0^dM^{d}$ balls for some constant $c_d$ which only depends on $d$. Denote $\theta_k$ as the center of the $k$th ball and thus we obtain $\{\theta_j\}_{j=1,...}$. The worst case among the $c_dM^d$ centers of balls satisfies
		\begin{eqnarray*}
			P\left(\sup_{k}\left\|\triangledown \widehat R_{L,\xi}(\theta_k,\epsilon)-\triangledown R_{L,\xi}(\theta_k,\epsilon) \right\|>\delta\right)\leq c_dB_0^dM^de^{-c_1n\delta^2/d}.
		\end{eqnarray*}
		For any $\theta\in B(0,B_0)$, the distance from $\theta$ to its nearest $\theta_k$ is at most $1/M$, thus there exists some constant $c_2$ such that
		\begin{eqnarray*}
			\inf_k\left\|\triangledown  R_{L,\xi}(\theta_k,\epsilon)-\triangledown R_{L,\xi}(\theta,\epsilon) \right\|\leq  c_2/ M\xi.
		\end{eqnarray*}
		In terms of $\triangledown  \widehat R_{L,\xi}(\theta_k,\epsilon)-\triangledown \widehat R_{L,\xi}(\theta,\epsilon)$, we have
		\begin{eqnarray*}
		&&\triangledown  \widehat R_{L,\xi}(\theta_k,\epsilon)-\triangledown \widehat R_{L,\xi}(\theta,\epsilon)\\
		&=&\frac{2}{n}\sum_{i=1}^n x_ix_i^{\top}( \theta_k-\theta)+\frac{2\epsilon}{n}\sum_{i=1}^n\left( \frac{\theta_k}{\sqrt{\|\theta_k\|^2+\xi^2/|x_i^{\top}\theta_k-y_i|^2}}|x_i^{\top}\theta_k-y_i|-\frac{\theta}{\sqrt{\|\theta\|^2+\xi^2/|x_i^{\top}\theta-y_i|^2}}|x_i^{\top}\theta-y_i|\right)\\
		&&+\frac{2\epsilon}{n}\sum_{i=1}^n \left(\frac{\|\theta_k\|^2}{\sqrt{\|\theta_k\|^2+\xi^2/|x_i^{\top}\theta_k-y_i|^2}}x_i\sgn(x_i^{\top}\theta_k-y_i)-\frac{\|\theta\|^2}{\sqrt{\|\theta\|^2+\xi^2/|x_i^{\top}\theta-y_i|^2}}x_i\sgn(x_i^{\top}\theta-y_i)\right)\\
		&&+2\epsilon^2\frac{1}{n}\sum_{i=1}^n\theta_k\frac{\|\theta_k\|^2}{\|\theta_k\|^2+\xi^2/|x_i^{\top}\theta_k-y_i|^2}-\theta\frac{\|\theta\|^2}{\|\theta\|^2+\xi^2/|x_i^{\top}\theta-y_i|^2},
		\end{eqnarray*}
		thus taking $\widehat{\Sigma}=\frac{1}{n}\sum_{i=1}^n x_ix_i^{\top}$, and denoting $\|A\|$ as the operator norm of matrix $A$,
		\begin{eqnarray*}
		&&\left\|\triangledown  \widehat R_{L,\xi}(\theta_k,\epsilon)-\triangledown \widehat R_{L,\xi}(\theta,\epsilon)\right\|\\
		&\leq& 2\|\widehat{\Sigma}-\Sigma\|\|\theta_k-\theta\|+\left\|\frac{2\epsilon}{n}\sum_{i=1}^n\left( \frac{\theta_k}{\sqrt{\|\theta_k\|^2+\xi^2/|x_i^{\top}\theta_k-y_i|^2}}|x_i^{\top}\theta_k-y_i|-\frac{\theta}{\sqrt{\|\theta\|^2+\xi^2/|x_i^{\top}\theta-y_i|^2}}|x_i^{\top}\theta_k-y_i|\right)\right\|\\
		&&+\left\|\frac{2\epsilon}{n}\sum_{i=1}^n\left( \frac{\theta}{\sqrt{\|\theta\|^2+\xi^2/|x_i^{\top}\theta-y_i|^2}}|x_i^{\top}\theta_k-y_i|-\frac{\theta}{\sqrt{\|\theta\|^2+\xi^2/|x_i^{\top}\theta-y_i|^2}}|x_i^{\top}\theta-y_i|\right)\right\|\\
		&&+\left\|\frac{2\epsilon}{n}\sum_{i=1}^n \left(\frac{\|\theta_k\|^2}{\sqrt{\|\theta_k\|^2+\xi^2/|x_i^{\top}\theta_k-y_i|^2}}x_i\sgn(x_i^{\top}\theta_k-y_i)-\frac{\|\theta\|^2}{\sqrt{\|\theta\|^2+\xi^2/|x_i^{\top}\theta-y_i|^2}}x_i\sgn(x_i^{\top}\theta_k-y_i)\right)\right\|\\
		&&+\left\|\frac{2\epsilon}{n}\sum_{i=1}^n \left(\frac{\|\theta\|^2}{\sqrt{\|\theta\|^2+\xi^2/|x_i^{\top}\theta-y_i|^2}}x_i\sgn(x_i^{\top}\theta_k-y_i)-\frac{\|\theta\|^2}{\sqrt{\|\theta\|^2+\xi^2/|x_i^{\top}\theta-y_i|^2}}x_i\sgn(x_i^{\top}\theta-y_i)\right)\right\|\\
		&&+\left\|2\epsilon^2\frac{1}{n}\sum_{i=1}^n\theta_k\frac{\|\theta_k\|^2}{\|\theta_k\|^2+\xi^2/|x_i^{\top}\theta_k-y_i|^2}-\theta\frac{\|\theta\|^2}{\|\theta\|^2+\xi^2/|x_i^{\top}\theta-y_i|^2}\right\|\\
		&:=&A_1+A_2+A_3+A_4+A_5+A_6.
		\end{eqnarray*}
		Following Lemma A.2 of \cite{ing2011stepwise} to quantify $\|\widehat\Sigma-\Sigma\|$, there is some large enough $c_2'$ such that with probability tending to 1,
		\begin{eqnarray*}
		A_1&\leq& 2 c_2' \sqrt{\frac{d\log n}{n}} \|\theta_k-\theta\|.
		\end{eqnarray*}

		For $A_2$,
		\begin{eqnarray*}
		A_2&=&\left\|\frac{2\epsilon}{n}\sum_{i=1}^n \left( \frac{\theta_k}{\sqrt{\|\theta_k\|^2+\xi^2/|x_i^{\top}\theta_k-y_i|^2}}-\frac{\theta}{\sqrt{\|\theta\|^2+\xi^2/|x_i^{\top}\theta-y_i|^2}} \right) |x_i^{\top}\theta_k-y_i| \right\|\\
		&\leq& \left\|\frac{2\epsilon}{n}\sum_{i=1}^n \left( \frac{\theta_k-\theta}{\sqrt{\|\theta_k\|^2+\xi^2/|x_i^{\top}\theta_k-y_i|^2}} \right) |x_i^{\top}\theta_k-y_i| \right\|\\
	    &&+\left\|\frac{2\epsilon}{n}\sum_{i=1}^n \left( \frac{\theta}{\sqrt{\|\theta_k\|^2+\xi^2/|x_i^{\top}\theta_k-y_i|^2}}-\frac{\theta}{\sqrt{\|\theta\|^2+\xi^2/|x_i^{\top}\theta-y_i|^2}} \right) |x_i^{\top}\theta_k-y_i| \right\|\\
	    &\leq& \frac{1}{M}\frac{2\epsilon}{n}\sum_{i=1}^n \frac{|x_i^{\top}\theta_k-y_i|}{\sqrt{\|\theta_k\|^2+\xi^2/|x_i^{\top}\theta_k-y_i|^2}}\\
	    &&+\left\|\frac{2\epsilon}{n}\sum_{i=1}^n \left( \frac{\theta}{\sqrt{\|\theta_k\|^2+\xi^2/|x_i^{\top}\theta_k-y_i|^2}}-\frac{\theta}{\sqrt{\|\theta\|^2+\xi^2/|x_i^{\top}\theta-y_i|^2}} \right) |x_i^{\top}\theta_k-y_i| \right\|.
		\end{eqnarray*}
		Using Bernstein inequality, for some constant $c_1'$,
		\begin{eqnarray*}
		P\left(\sup_{k}\left|\frac{2\epsilon}{n}\sum_{i=1}^n \frac{|x_i^{\top}\theta_k-y_i|}{\sqrt{\|\theta_k\|^2+\xi^2/|x_i^{\top}\theta_k-y_i|^2}}-\mathbb{E} \frac{2\epsilon|x^{\top}\theta_k-y|}{\sqrt{\|\theta_k\|^2+\xi^2/|x^{\top}\theta_k-y|^2}}\right|\geq \delta\right)\leq c_d B_0^{d}M^de^{-c_1n\delta^2\xi^2}.
		\end{eqnarray*}
		In addition, one can take $c_2'$ large enough such that
		\begin{eqnarray*}
		&&\frac{2\epsilon}{n}\sum_{i=1}^n \left| \frac{1}{\sqrt{\|\theta_k\|^2+\xi^2/|x_i^{\top}\theta_k-y_i|^2}}-\frac{1}{\sqrt{\|\theta\|^2+\xi^2/|x_i^{\top}\theta-y_i|^2}} \right| |x_i^{\top}\theta_k-y_i|\\
		&=&\frac{2\epsilon}{n}\sum_{i=1}^n \left| \frac{|x_i^{\top}\theta_k-y_i|}{\sqrt{\|\theta_k\|^2|x_i^{\top}\theta_k-y_i|^2+\xi^2}}-\frac{|x_i^{\top}\theta-y_i|}{\sqrt{\|\theta\|^2|x_i^{\top}\theta-y_i|^2+\xi^2}} \right| |x_i^{\top}\theta_k-y_i|\\
		&=& \frac{2\epsilon}{n}\sum_{i=1}^n \left| \frac{|x_i^{\top}\theta_k-y_i|\sqrt{\|\theta\|^2|x_i^{\top}\theta-y_i|^2+\xi^2}- |x_i^{\top}\theta-y_i|\sqrt{\|\theta_k\|^2|x_i^{\top}\theta_k-y_i|^2+\xi^2} }{\sqrt{\|\theta_k\|^2|x_i^{\top}\theta_k-y_i|^2+\xi^2}\sqrt{\|\theta\|^2|x_i^{\top}\theta-y_i|^2+\xi^2}} \right| |x_i^{\top}\theta_k-y_i|.
		\end{eqnarray*}
		Furthermore,
		\begin{eqnarray*}
		&&\left||x_i^{\top}\theta_k-y_i|\sqrt{\|\theta\|^2|x_i^{\top}\theta-y_i|^2+\xi^2}- |x_i^{\top}\theta-y_i|\sqrt{\|\theta_k\|^2|x_i^{\top}\theta_k-y_i|^2+\xi^2}\right|\\
		&\leq&\left||x_i^{\top}\theta_k-y_i|\sqrt{\|\theta\|^2|x_i^{\top}\theta-y_i|^2+\xi^2}- |x_i^{\top}\theta_k-y_i|\sqrt{\|\theta_k\|^2|x_i^{\top}\theta_k-y_i|^2+\xi^2}\right|\\
		&&+\left||x_i^{\top}\theta_k-y_i|\sqrt{\|\theta_k\|^2|x_i^{\top}\theta_k-y_i|^2+\xi^2}- |x_i^{\top}\theta-y_i|\sqrt{\|\theta_k\|^2|x_i^{\top}\theta_k-y_i|^2+\xi^2}\right|\\
		&\leq&|x_i^{\top}\theta_k-y_i| \left|\sqrt{\|\theta\|^2|x_i^{\top}\theta-y_i|^2+\xi^2}-\sqrt{\|\theta_k\|^2|x_i^{\top}\theta_k-y_i|^2+\xi^2}\right|\\
		&&+ \|x_i\|\|\theta_k-\theta\|\sqrt{\|\theta_k\|^2|x_i^{\top}\theta_k-y_i|^2+\xi^2}.
		\end{eqnarray*}
		As a result, we have
		\begin{eqnarray*}
		&&\left\|\frac{2\epsilon}{n}\sum_{i=1}^n \left( \frac{\theta}{\sqrt{\|\theta_k\|^2+\xi^2/|x_i^{\top}\theta_k-y_i|^2}}-\frac{\theta}{\sqrt{\|\theta\|^2+\xi^2/|x_i^{\top}\theta-y_i|^2}} \right) |x_i^{\top}\theta_k-y_i| \right\|\\
		&\leq&\|\theta\|\frac{2\epsilon}{n}\sum_{i=1}^n \left| \frac{1}{\sqrt{\|\theta_k\|^2+\xi^2/|x_i^{\top}\theta_k-y_i|^2}}-\frac{1}{\sqrt{\|\theta\|^2+\xi^2/|x_i^{\top}\theta-y_i|^2}} \right| |x_i^{\top}\theta_k-y_i|\\
		&\leq&  \|\theta\|\frac{2\epsilon}{n}\sum_{i=1}^n \left| \frac{\|x_i\|\|\theta_k-\theta\|\sqrt{\|\theta_k\|^2|x_i^{\top}\theta_k-y_i|^2+\xi^2} }{\sqrt{\|\theta_k\|^2|x_i^{\top}\theta_k-y_i|^2+\xi^2}\sqrt{\|\theta\|^2|x_i^{\top}\theta-y_i|^2+\xi^2}} \right| |x_i^{\top}\theta_k-y_i|\\
		&&+ \|\theta\|\frac{2\epsilon}{n}\sum_{i=1}^n \left| \frac{ \left|\sqrt{\|\theta\|^2|x_i^{\top}\theta-y_i|^2+\xi^2}-\sqrt{\|\theta_k\|^2|x_i^{\top}\theta_k-y_i|^2+\xi^2}\right| }{\sqrt{\|\theta_k\|^2|x_i^{\top}\theta_k-y_i|^2+\xi^2}\sqrt{\|\theta\|^2|x_i^{\top}\theta-y_i|^2+\xi^2}} \right| |x_i^{\top}\theta_k-y_i|^2\\
		&\leq& \frac{2\epsilon B_0}{nM}\sum_{i=1}^n \frac{\|x_i\|}{\xi}|x_i^{\top}\theta_k-y_i|\\
		&&+\frac{2\epsilon}{n}\sum_{i=1}^n \left|\frac{\|\theta\|}{\sqrt{\|\theta_k\|^2|x_i^{\top}\theta_k-y_i|^2+\xi^2}}-\frac{\|\theta\|}{\sqrt{\|\theta\|^2|x_i^{\top}\theta-y_i|^2+\xi^2}} \right| |x_i^{\top}\theta_k-y_i|^2\\
		&=& \frac{2\epsilon B_0}{nM}\sum_{i=1}^n \frac{\|x_i\|}{\xi}|x_i^{\top}\theta_k-y_i|\\
		&&+\frac{2\epsilon}{n}\sum_{i=1}^n \left|\frac{\|\theta_k\||x_i^{\top}\theta_k-y_i|-\|\theta\||x_i^{\top}\theta_k-y_i|}{\sqrt{\|\theta_k\|^2|x_i^{\top}\theta_k-y_i|^2+\xi^2}} \right| |x_i^{\top}\theta_k-y_i|\\
		&\leq& \frac{2\epsilon B_0}{nM}\sum_{i=1}^n \frac{\|x_i\|}{\xi}|x_i^{\top}\theta_k-y_i|\\
		&&+\frac{2\epsilon}{n}\sum_{i=1}^n \left|\frac{\|\theta_k\||x_i^{\top}\theta_k-y_i|-\|\theta_k\||x_i^{\top}\theta-y_i|}{\sqrt{\|\theta_k\|^2|x_i^{\top}\theta_k-y_i|^2+\xi^2}} \right| |x_i^{\top}\theta_k-y_i|
		+\frac{2\epsilon}{n}\sum_{i=1}^n \left|\frac{\|\theta_k\||x_i^{\top}\theta-y_i|-\|\theta\||x_i^{\top}\theta-y_i|}{\sqrt{\|\theta_k\|^2|x_i^{\top}\theta_k-y_i|^2+\xi^2}} \right| |x_i^{\top}\theta_k-y_i|\\
		&\leq& \frac{2\epsilon B_0}{nM}\sum_{i=1}^n \frac{\|x_i\|}{\xi}|x_i^{\top}\theta_k-y_i|\\
		&&+B_0\frac{2\epsilon}{n}\sum_{i=1}^n\frac{\|x_i\|}{M\xi} |x_i^{\top}\theta_k-y_i|
		+\frac{2\epsilon}{n}\sum_{i=1}^n \left|\frac{\|\theta_k-\theta\||x_i^{\top}\theta_k-y_i|+\|\theta_k-\theta\|^2\|x_i\|}{\sqrt{\|\theta_k\|^2|x_i^{\top}\theta_k-y_i|^2+\xi^2}} \right| |x_i^{\top}\theta_k-y_i|\\
		&\leq& \frac{2\epsilon B_0}{nM}\sum_{i=1}^n \frac{\|x_i\|}{\xi}|x_i^{\top}\theta_k-y_i|+B_0\frac{2\epsilon}{n}\sum_{i=1}^n\frac{\|x_i\|}{M\xi} |x_i^{\top}\theta_k-y_i|
		+\frac{2\epsilon}{n}\frac{1}{M\xi}\sum_{i=1}^n |x_i^{\top}\theta_k-y_i|^2+\frac{2\epsilon}{n}\frac{1}{M^2\xi}\sum_{i=1}^n |x_i^{\top}\theta_k-y_i|\|x_i\|,
		\end{eqnarray*}
		which can be bounded using Bernstein inequality as well.
		
		On the other hand, 
		\begin{eqnarray*}
		A_3&=&\left\|\frac{2\epsilon}{n}\sum_{i=1}^n\frac{\theta}{\sqrt{\|\theta\|^2+\xi^2/|x_i^{\top}\theta-y_i|^2}}\left( |x_i^{\top}\theta_k-y_i| -|x_i^{\top}\theta-y_i|\right)  \right\|\\
		&\leq& \left|\frac{2\epsilon}{n}\sum_{i=1}^n\left( |x_i^{\top}\theta_k-y_i| -|x_i^{\top}\theta-y_i|\right)  \right|\leq 2\epsilon  \|\theta_k-\theta\|(\sup_i\|x_i\|).
		\end{eqnarray*}
    	Further, assume $\|\theta_k\|\leq\|\theta\|$,
		\begin{eqnarray*}
		A_4&=&\frac{2\epsilon}{n}\sum_{i=1}^n \left| \frac{\|\theta_k\|^2}{\sqrt{\|\theta_k\|^2+\xi^2/|x_i^{\top}\theta_k-y_i|^2}}-\frac{\|\theta\|^2}{\sqrt{\|\theta\|^2+\xi^2/|x_i^{\top}\theta-y_i|^2}} \right| \|x_i\|\\
		&\leq&\frac{2\epsilon}{n}\sum_{i=1}^n \left| \frac{\|\theta_k\|^2}{\sqrt{\|\theta_k\|^2+\xi^2/|x_i^{\top}\theta_k-y_i|^2}}-\|\theta\|+\|\theta_k\|-\|\theta_k\|\right| \|x_i\|\\
		&\leq&\frac{2\epsilon}{n}\sum_{i=1}^n \|\theta_k-\theta\| \|x_i\|\\
		&&+\frac{2\epsilon}{n}\sum_{i=1}^n \left| \frac{\|\theta_k\|^2-\|\theta_k\|\sqrt{\|\theta_k\|^2+\xi^2/|x_i^{\top}\theta_k-y_i|^2}}{\sqrt{\|\theta_k\|^2+\xi^2/|x_i^{\top}\theta_k-y_i|^2}}\right| \|x_i\|\\
		&\leq&\frac{2\epsilon}{n}\sum_{i=1}^n \|\theta_k-\theta\| \|x_i\|+\frac{2\epsilon}{n}\sum_{i=1}^n  \frac{\|\theta_k\|\xi}{\sqrt{\|\theta_k\|^2|x_i^{\top}\theta_k-y_i|^2+\xi^2}} \|x_i\|,
		\end{eqnarray*}
		where both the terms can be bounded using Bernstein inequality.
		
		For $A_5$,
	\begin{eqnarray*}
	A_5&\leq& B_0\left\| \frac{2\epsilon}{n} \sum_{i=1}^n x_i\sgn(x_i^{\top}\theta_k-y_i)-x_i\sgn(x_i^{\top}\theta-y_i) \right\|\\
	&\leq& B_0\left(\frac{2\epsilon}{n}\sum_{i=1}^n \|x_i\|1\left\{  |x_i^{\top}\theta_k-y_i|\leq \|x_i\|\|\theta_k-\theta\|   \right\}\right).
	\end{eqnarray*}
	For any $\theta_k$, $x_i^{\top}\theta_k-y_i$ follows a Gaussian distribution with some mean and the variance is finite. As a result, for some $c_3>0$, we have
	\begin{eqnarray*}
	P\left(\left|\frac{1}{n}\sum_{i=1}^n \|x_i\| 1\left\{  |x_i^{\top}\theta_k-y_i|\leq \kappa  \right\}-\mathbb{E}\left(\|x\|1\{ |x^{\top}\theta_k-y| \leq \kappa\}\right)\right|>\delta\right)\leq e^{-nc_3\delta^2/(d\kappa)}.
	\end{eqnarray*}
Note that $\mathbb{E}\left(\|x\|1\{ |x^{\top}\theta_k-y| \leq \kappa\}\right)=O(\kappa\sqrt{d})$. Finally,
		\begin{eqnarray*}
		A_6&\leq& c_2' \|\theta-\theta_k\|.
		\end{eqnarray*}

		Putting $\left\|\triangledown \widehat R_{L,\xi}(\theta,\epsilon)-\widehat\triangledown R_{L,\xi}(\theta_k,\epsilon) \right\|$, $\left\|\triangledown R_{L,\xi}(\theta,\epsilon)-\triangledown R_{L,\xi}(\theta_k,\epsilon) \right\|$, and $\left\|\triangledown \widehat R_{L,\xi}(\theta_k,\epsilon)-\triangledown R_{L,\xi}(\theta_k,\epsilon) \right\|$ together, there exists some function $g_1$ such that
		\begin{eqnarray*}
		P\left(\sup_{\|\theta\|\leq B_0}\left\|\triangledown \widehat R_{L,\xi}(\theta,\epsilon)-\triangledown R_{L,\xi}(\theta,\epsilon) \right\|> \delta \right)\leq
		e^{-g_1(\delta,\xi,d,n)}.
		\end{eqnarray*}
		When using this lemma, instead of figuring out the complicated form of $g_1$, one can directly use $(\xi,\delta,d,n)$ to determine a suitable $M$, and then use the upper bound of $\left\|\triangledown R_{L,\xi}(\theta,\epsilon)-\triangledown R_{L,\xi}(\theta_k,\epsilon) \right\|$, $\left\|\triangledown \widehat R_{L,\xi}(\theta_k,\epsilon)-\triangledown R_{L,\xi}(\theta_k,\epsilon) \right\|$ and $A_1$ to $A_5$ to obtain a probability bound.
	\end{proof}
		\begin{proof}[Proof of Theorem \ref{thm:opt}]
		Rewrite $\theta_{\xi}^{(t)}$ as $\theta^{(t)}$ for simplicity. After we defining $\triangledown R_{L,\xi}$, unlike $\triangledown R_{L,0}$, we may not find a function whose gradient is $\triangledown R_{L,\xi}$. As s result, we propose another function $\widetilde{R}_{{L,\xi}}(\theta,\epsilon)$ and figure out its gradient $\partial\widetilde{R}_{{L,\xi}}(\theta,\epsilon)/\partial \theta$, then bound the difference between $\partial\widetilde{R}_{{L,\xi}}(\theta,\epsilon)/\partial \theta$ and $\triangledown R_{L,\xi}$.
		
		Denote $\widetilde{R}_{{L,\xi}}(\theta,\epsilon)$ as 
		\begin{eqnarray*}
			\widetilde{R}_{{L,\xi}}(\theta,\epsilon)=\mathbb{E}\bigg[ (x^{\top}\theta-y)^2+2\epsilon\sqrt{\|\theta\|^2+\xi^2/|x^{\top}\theta-y|^2} |x^{\top}\theta -y|+\epsilon^2 \|\theta\|^2
			\bigg],
		\end{eqnarray*}		
		then taking gradient on $\widetilde{R}_{\xi}$ w.r.t $\theta$, it becomes
		\begin{eqnarray*}
		&&\triangledown\widetilde{R}_{{L,\xi}}(\theta,\epsilon)\\
		&=&\mathbb{E}\bigg[2x(x^{\top}\theta-y)+2\epsilon\frac{\theta}{\sqrt{\|\theta\|^2+\xi^2/|x^{\top}\theta-y|^2 }}|x^{\top}\theta-y|-\frac{2\epsilon \xi^2 x}{\sqrt{\|\theta\|^2+\xi^2/|x^{\top}\theta-y|^2 }}\frac{\sgn(x^{\top}\theta-y)}{(x^{\top}\theta-y)^2}\\
		&&+2\epsilon x\sqrt{\|\theta\|^2+\xi^2/|x^{\top}\theta-y|^2}\sgn(x^{\top}\theta-y)+2\epsilon^2\theta\bigg]
		\end{eqnarray*}
		For $\triangledown R_{L,\xi}(\theta,\epsilon)$, we have
		\begin{eqnarray*}
		&&\triangledown R_{L,\xi}(\theta,\epsilon)\\
		&=&\mathbb{E} \bigg[2x(x^{\top}\theta-y)+2\epsilon \frac{\theta}{\sqrt{\|\theta\|^2+\xi^2/|x^{\top}\theta-y|^2}}|x^{\top}\theta-y|\\
			&&+2\epsilon x\frac{\|\theta\|^2}{\sqrt{ \|\theta\|^2+\xi^2/|x^{\top}\theta-y|^2 }}\sgn(x^{\top}\theta-y)+2 \epsilon^2\theta\frac{\|\theta\|^2}{\|\theta\|^2+\xi^2/|x^{\top}\theta-y|^2}\bigg]\\
        \end{eqnarray*}
		We assume $\theta\in B(0,B_1)$, then there exists $B_2$ such that both $\|\triangledown R_{{L,\xi}}(\epsilon,\theta)\|$ and $\|\partial\widetilde{R}_{{L,\xi}}(\theta,\epsilon)/\partial \theta\|$ are bounded by some large constant $B_2>0$.
		
		To compare the difference between $\frac{\partial \widetilde{R}_{{L,\xi}}}{\partial \theta}$ and $\triangledown R_{{L,\xi}}(\theta,\epsilon)$, we have
		\begin{eqnarray*}
			\left\|\triangledown R_{{L,\xi}}(\theta,\epsilon)-\frac{\partial \widetilde{R}_{{L,\xi}}}{\partial \theta}\right\|=2\epsilon\left\|\theta \mathbb{E}\left(\frac{\xi^2/|x^{\top}\theta-y|^2}{\|\theta\|^2+\xi^2/|x^{\top}\theta-y|^2}\right)\right\|.
		\end{eqnarray*}
		
		As a result, there exists some $B_3>0$ such that
		\begin{equation*}
		\left\|\triangledown R_{\xi}(\theta,\epsilon)-\frac{\partial \widetilde{R}_{{L,\xi}}}{\partial \theta}\right\|\leq \frac{\xi}{B_3} \left\|\frac{\partial \widetilde{R}_{{L,\xi}}}{\partial \theta}\right\|.
		\end{equation*}
		Therefore the gradient of $\widetilde{R}_{{L,\xi}}$ is the dominant term when updating $\theta$ in each iteration.
		
		Next we show that $\widetilde{R}_{{L,\xi}}(\theta^{(t+1)},\epsilon)$ is smaller than  $\widetilde{R}_{{L,\xi}}(\theta^{(t)},\epsilon)$ in probability. From the definition of $\widetilde{R}$, similar with Proposition \ref{prop:geo}, there exists $L$ such that 
		\begin{eqnarray*}
			\widetilde{R}_{L,\xi}(\theta^{(t+1)},\epsilon)&\leq& \widetilde{R}_{L,\xi}(\theta^{(t)},\epsilon) +\left(\frac{\partial \widetilde{R}_{L,\xi}}{\partial \theta^{(t)}} \right)^{\top}(\theta^{(t+1)}-\theta^{(t)})+\frac{L}{\xi}\|\theta^{(t+1)}-\theta^{(t)}\|^2.
		\end{eqnarray*}
		From Lemma \ref{lem:gradient}, with probability at least $1-e^{-g_1(\delta,\xi,d,n)}$, \begin{eqnarray*}
			\left\|\triangledown \widehat R_{L,\xi}(\theta,\epsilon)-\triangledown R_{L,\xi}(\theta,\epsilon)\right\|<\delta,
		\end{eqnarray*}
		thus
		\begin{eqnarray*}
			\left(\frac{\partial \widetilde{R}_{L,\xi}}{\partial \theta^{(t)}} \right)^{\top}\left(\theta^{(t+1)}-\theta^{(t)}\right)= -\eta\left(\frac{\partial \widetilde{R}_{L,\xi}}{\partial \theta^{(t)}} \right)^{\top}\triangledown\widehat{R}_{L,\xi}(\theta^{(t)},\epsilon)
			&\leq&-\eta\left\|\frac{\partial \widetilde{R}_{L,\xi}}{\partial \theta^{(t)}} \right\|^2(1-\xi/B_3)+\eta\delta B_2,
		\end{eqnarray*}
		and
		\begin{eqnarray*}
			\|\theta^{(t+1)}-\theta^{(t)}\|^2&=&\eta^2\|\triangledown\widehat{R}_{L,\xi}(\theta^{(t)},\epsilon)\|^2\leq \eta^2\|\triangledown {R}_{L,\xi}(\theta^{(t)},\epsilon)\|^2+2\eta^2\delta B_2+\eta^2B_2^2\delta^2\\&\leq&\eta^2\left\|\frac{\partial \widetilde{R}_{L,\xi}}{\partial \theta^{(t)}} \right\|^2(1+\xi/B_3)^2+2\eta^2\delta B_2+o,
		\end{eqnarray*}
		where $A+o$ represents $A+a$ where $a/A\rightarrow 0$ for expressions $a$ and $A$.
		
		Therefore,
		\begin{eqnarray*}
			\widetilde{R}_{L,\xi}(\theta^{(t+1)},\epsilon)&\leq& \widetilde{R}_{L,\xi}(\theta^{(t)},\epsilon)-\eta\left\|\frac{\partial \widetilde{R}_{{L,\xi}}}{\partial \theta^{(t)}} \right\|^2(1-\xi/B_3)+\eta\delta B_2 + \frac{L\eta^2}{\xi}\left\|\frac{\partial \widetilde{R}_{\xi}}{\partial \theta^{(t)}} \right\|^2(1+2\xi/B_3)+o. 
		\end{eqnarray*}
		
		When taking $\eta=\frac{\xi}{2L}$, 
		\begin{equation}\label{eqn:bound_B1}
		\widetilde{R}_{L,\xi}(\theta^{(t+1)},\epsilon)\leq \widetilde{R}_{L,\xi}(\theta^{(t)},\epsilon) -\frac{\xi}{4L} \left\|\frac{\partial \widetilde{R}_{L,\xi}}{\partial \theta^{(t)}}  \right\|^2+\frac{\xi^2B_2^2}{LB_3}+\frac{\xi\delta B_2}{2L}+o .   
		\end{equation}
		From (\ref{eqn:bound_B1}), when taking $\xi\rightarrow 0$ and $\delta\rightarrow 0$, we know that $\widetilde{R}_{\xi}(\theta^{(t)},\epsilon)$ will decrease in $t$ at the beginning of training; thus based on the shape of $\widetilde{R}_{\xi}$ and $\|\theta^{(0)}\|\leq B_0$, one can use induction to show the existence of $B_1$, i.e. $\theta^{(t)}\in B(0,B_1)$ for any $t$.
		
		To bound the difference $\widetilde R_\xi(\theta^{(t+1)},\epsilon)$ and $\widetilde R_{\xi}(\theta^*,\epsilon)$, since $\widetilde{R}_{\xi}$ is a convex function, we have
		\begin{eqnarray*}
			\widetilde{R}_{L,\xi}(\theta^{(t)},\epsilon)\leq \widetilde{R}_{L,\xi}(\theta^*,\epsilon)+\left(\frac{\partial \widetilde{R}_{L,\xi}}{\partial \theta^{(t)}}\right)^{\top}(\theta^{(t)}-\theta^*).
		\end{eqnarray*}
		Thus
		\begin{eqnarray*}
			\widetilde{R}_{L,\xi}(\theta^{(t+1)},\epsilon)&\leq& \widetilde{R}_{L,\xi}(\theta^*,\epsilon)+\left(\frac{\partial \widetilde{R}_{L,\xi}}{\partial \theta^{(t)}}\right)^{\top}(\theta^{(t)}-\theta^*)-\frac{\xi}{4L }\left\|\frac{\partial \widetilde{R}_{L,\xi}}{\partial \theta^{(t)}}\right\|^2+\frac{\xi^2B_2^2}{LB_3}+\frac{\xi\delta B_2}{2L}+o \\
			&\leq&\widetilde{R}_{L,\xi}(\theta^*,\epsilon)+\triangledown\widehat{R}_{L,\xi}(\theta^{(t)},\epsilon)^{\top}(\theta^{(t)}-\theta^*)-\frac{\xi}{4L }\left\|\triangledown\widehat{R}_{L,\xi}(\theta^{(t)},\epsilon)\right\|^2+\frac{\xi B_1B_2}{B_3}+\delta B_2+o \\
			&=&\widetilde{R}_{L,\xi}(\theta^*,\epsilon)-\frac{L}{\xi}\left\| \theta^{(t+1)}-\theta^*  \right\|^2+\frac{L}{\xi} \|\theta^{(t)}-\theta^*\|^2+\frac{\xi B_1B_2}{B_3}+\delta B_2+o,
		\end{eqnarray*}
		where the last step is obtained from the updating rule $\theta^{(t+1)}-\theta^{(t)}=-\eta \triangledown\widehat{R}_{L,\xi}(\theta^{(t)},\epsilon)$.
		
		As a result, summing up from $t=1$ to $T$,
		\begin{equation*}
		\sum_{t=1}^T\widetilde{R}_{L,\xi}(\theta^{(t)},\epsilon)-\widetilde{R}_{L,\xi}(\theta^*,\epsilon)\leq \frac{L}{\xi }\|\theta^{(0)}-\theta^*\|^2 +\frac{T\xi B_1B_2}{B_3}+T\delta B_2+o .
		\end{equation*}

		If $\theta^{(t)}$ is away from $\theta^*$, then $\widetilde{R}_{L,\xi}(\theta^{(t)},\epsilon)$ reduces in every step.

		When $\theta^{(t)}$ is close to $\theta^*$, since $\xi\rightarrow0$, we have $\widetilde{R}_{L,\xi}(\theta,\epsilon)- R_{L,\xi}(\theta,\epsilon)\rightarrow 0$, thus through taking suitable choice of $(\xi,t,\eta)$, we can get $R_{L,\xi}(\theta^{(T)},\epsilon)\rightarrow R_{L,0}(\theta^*,\epsilon)$.
		
		The proof of $\widehat R_{L,\xi}(\theta^{(T)},\epsilon)-R_{L,\xi}(\theta^{(T)}\epsilon)\rightarrow 0$ in probability follows a concentration bound similar as in Lemma \ref{lem:gradient}. 
		
		Finally we relax the condition of $\|\theta_0\|\leq B_0$. For $y_i=\theta_0^{\top}x_i+\varepsilon_i$ with $\|\theta_0\|_{\Sigma}^2+\sigma^2=1$, $\sigma>\xi/v$, denote $y'_i:=vy_i=(v\theta_0)^{\top}x_i+v\varepsilon_i:=(\theta_0')^{\top}x_i+\varepsilon_i'$ for some $v\rightarrow \infty$. In this case, if the initialization satisfy $(\theta^{'(0)})=v\theta^{(0)}$ and $\xi'=v^2\xi$, then for any $t=1,...,T$, we always have $(\theta^{'(t)})=v\theta^{(t)}$. Therefore, since $\theta^{(T)}\rightarrow \theta^*$, we also have $\theta^{'(T)}/v \rightarrow\theta^{*}$, which is just the minimizer of population adversarial loss for $y_i'$.
		
	\end{proof}
	
	\section{Proofs for low-dimensional nonlinear network}
	
	\begin{proof}[Proof of Theorem \ref{thm:low_nonlinear}]
		In the proof, we first consider  taking gradient on loss w.r.t $x$ to get the attack direction, i.e. fast gradient method to obtain the attack. And after the main proof, we discuss how to adapt the true attack.
		
		We consider three optimization problems: (1) a linear network using zero initialzation, (2) a nonlinear network using zero intialization, and (3) a nonlinear network with vanishing initialization. Extending from Theorem \ref{thm:opt}, we know how (1) works, then we bound the difference between (1), (2), and (3).
		
		We first bound the difference among (1) and (2) assuming $\xi$ and $t$ are any arbitrary number. And finally choose some suitable $\xi$ and $t$ to ensure the consistency. 
		
		Denote $$A_{\epsilon,\xi,OP1}(f,x,y)=\epsilon\sgn(f(x)-y)\frac{\frac{1}{\sqrt{h}} \sum_{j=1}^ha_j\phi'(0)\theta_j}{\sqrt{\left\|\frac{1}{\sqrt{h}}\sum_{j=1}^ha_j\phi'(0)\theta_j\right\|^2+\xi^2/|f_{\theta}(x)-y|^2}},$$
		and for the weight $\theta_j$ for the $j$th node, 
		$$ g_{j,\xi}^{OP1}(\theta,\epsilon)=\frac{2}{n}\sum_{i=1}^n\frac{\phi(0)'a_j}{\sqrt{h}}\left[x_i+A_{\epsilon,\xi,OP1}(f,x_i,y_i)\right]\left[\frac{1}{\sqrt{h}}\sum_{k=1}^h\left(\phi(0)'x_i^{\top}\theta_ka_k+\phi'(0)A_{\epsilon,\xi,OP1}(f,x_i,y_i)^{\top}\theta_ka_k\right)-y_i\right].$$
		$$ g_{j,\xi}^{OP2}(\theta,\epsilon)=\frac{2}{n}\sum_{i=1}^n\frac{\phi'\left(\left(x+  A_{\epsilon,\xi}(f,x,y)\right)^{\top}\theta_j\right)a_j}{\sqrt{h}}\left[x_i+A_{\epsilon,\xi}(f,x_i,y_i)\right]\left[\frac{1}{\sqrt{h}}\sum_{k=1}^h \phi\left(\left(x+  A_{\epsilon,\xi}(f,x,y)\right)^{\top}\theta_k\right)a_k-y_i\right].$$
		
		Denote $\theta^{OP1}(t)$ as the weight obtained using linear network with zero initialization, i.e. for each hidden node $j=1,...,h$,
		\begin{eqnarray*}
			&&\theta^{OP1}_j(0)=0,\\
			&&\theta^{OP1}_j(t+1)=\theta^{OP1}_j(t)-\eta g_{j,\xi}^{OP1}(\theta^{OP1}_{j}(t),\epsilon).
		\end{eqnarray*}Also denote $\theta^{OP2}(t)$ as the weight obtained using nonlinear network with nonzero initialization, i.e.
		\begin{eqnarray*}
			&&\theta^{OP2}_j(0)=0,\\
			&&\theta^{OP2}_j(t+1)=\theta^{OP2}_j(t)-\eta g_{j,\xi}^{OP2}(\theta^{OP2}_{j}(t),\epsilon).
		\end{eqnarray*}
		For the original problem we consider (i.e. nonlinear network with vanishing initialization), define
		\begin{eqnarray*}
			&&\theta^{OP3}_j(0)\sim N\left( 0,\frac{1}{dh^{1+\delta'}}I \right),\\
			&&\theta^{OP3}_j(t+1)=\theta^{OP3}_j(t)-\eta g_{j,\xi}^{OP2}(\theta^{OP3}_{j}(t),\epsilon).
		\end{eqnarray*}
		\paragraph{Difference between (1) and (2)}At $(t+1)$th step, we have
		\begin{eqnarray*}
			\theta^{OP1}_j(t+1)-\theta^{OP2}_j(t+1)&=&\theta^{OP1}_j(t)-\theta^{OP2}_j(t)-\eta g_{j,\xi}^{OP1}(\theta^{OP1}_j(t),\epsilon)+\eta g_{j,\xi}^{OP2}(\theta^{OP2}_j(t),\epsilon)\\
			&=&\theta^{OP1}_j(t)-\theta^{OP2}_j(t)-\eta \left[g_{j,\xi}^{OP1}(\theta^{OP1}_j(t),\epsilon)- g_{j,\xi}^{OP1}(\theta^{OP2}_j(t),\epsilon)\right]\\&&-\eta \left[g_{j,\xi}^{OP1}(\theta^{OP2}_j(t),\epsilon)- g_{j,\xi}^{OP2}(\theta^{OP2}_j(t),\epsilon)\right].
		\end{eqnarray*}
		Denote 
		\begin{eqnarray*}
			\Delta_{1}(\theta,x,y)&=&A_{\epsilon,\xi}(f,x,y)-A_{\epsilon,\xi,OP1}(f,x,y),
		\end{eqnarray*}
		\begin{eqnarray*}
			\Delta_{2}(\theta,x,y)&=&\frac{1}{\sqrt{h}}\sum_{j=1}^h \phi\left(\left(x+  A_{\epsilon,\xi}(f,x,y)\right)^{\top}\theta_j\right)a_j\\&&-\frac{1}{\sqrt{h}}\sum_{j=1}^h\phi(0)'x^{\top}\theta_ja_j+\phi'(0)A_{\epsilon,\xi,OP1}(f,x,y)^{\top}\theta_ja_j,
		\end{eqnarray*}
		and
		\begin{eqnarray*}
			\Delta_{3,j}(\theta,x,y)=\phi'\left(\left(x+  A_{\epsilon,\xi}(f,x,y)\right)^{\top}\theta_j\right)-\phi'(0).
		\end{eqnarray*}
		Assume all $\Delta_{1}(\theta,x,y)$, $\Delta_{2}(\theta,x,y)$, $\Delta_{3,j}(\theta,x,y)$ converges to zero for any $j$ and $(x,y)\in\{ (x_i,y_i)\}_{i=1,...,n}$ (we will later go back to validate this assumption), then 
		\begin{eqnarray*}
			&&g_{j,\xi}^{OP1}(\theta,\epsilon)- g_{j,\xi}^{OP2}(\theta,\epsilon)\\
			&=&-\frac{2}{n}\sum_{i=1}^n\frac{\phi'\left(\left(x+  A_{\epsilon,\xi}(f,x,y)\right)^{\top}\theta_j\right)a_j}{\sqrt{h}}\left[x_i+A_{\epsilon,\xi}(f,x_i,y_i)\right]\left[\frac{1}{\sqrt{h}}\sum_{k=1}^h \phi\left(\left(x_i+  A_{\epsilon,\xi}(f,x_i,y_i)\right)^{\top}\theta_j\right)a_k-y_i\right]\\
			&&+\frac{2}{n}\sum_{i=1}^n\frac{\phi(0)'a_j}{\sqrt{h}}\left[x_i+A_{\epsilon,\xi,OP1}(f,x_i,y_i)\right]\left[\frac{1}{\sqrt{h}}\sum_{k=1}^h\left(\phi(0)'x_i^{\top}\theta_ka_k+\phi'(0)A_{\epsilon,\xi,OP1}(f,x_i,y_i)^{\top}\theta_ka_k\right)-y_i\right]\\
			&=&-\frac{2}{n}\sum_{i=1}^n\frac{a_j\Delta_{3,j}(\theta,x_i,y_i)}{\sqrt{h}}\left[x_i+A_{\epsilon,\xi,OP1}(f,x_i,y_i)\right]\left[\frac{1}{\sqrt{h}}\sum_{k=1}^h\left(\phi(0)'x_i^{\top}\theta_ka_k+\phi'(0)A_{\epsilon,\xi,OP1}(f,x_i,y_i)^{\top}\theta_ka_k\right)-y_i\right]\\
			&&-\frac{2}{n}\sum_{i=1}^n\frac{\phi(0)'a_j}{\sqrt{h}}\Delta_{1}(\theta,x_i,y_i)\left[\frac{1}{\sqrt{h}}\sum_{k=1}^h\left(\phi(0)'x_i^{\top}\theta_ka_k+\phi'(0)A_{\epsilon,\xi,OP1}(f,x_i,y_i)^{\top}\theta_ka_k\right)-y_i\right]\\
			&&-\frac{2}{n}\sum_{i=1}^n\frac{\phi(0)'a_j}{\sqrt{h}}\left[x_i+A_{\epsilon,\xi,OP1}(f,x_i,y_i)\right]\Delta_{2}(\theta,x_i,y_i)+o,
		\end{eqnarray*}
		for some remainder term $o$, thus
		\begin{eqnarray*}
			&&\frac{1}{\sqrt{h}}\sum_{i=1}^ha_jg_{j,\xi}^{OP1}(\theta,\epsilon)- \frac{1}{\sqrt{h}}\sum_{i=1}^ha_jg_{j,\xi}^{OP2}(\theta,\epsilon)\\
			&=&-\frac{2}{n}\sum_{i=1}^n\frac{\sum_{j=1}^h a_j^2\Delta_{3,j}(\theta,x_i,y_i)}{{h}}\left[x_i+A_{\epsilon,\xi,OP1}(f,x_i,y_i)\right]\left[\frac{1}{\sqrt{h}}\sum_{k=1}^h\left(\phi(0)'x_i^{\top}\theta_ka_k+\phi'(0)A_{\epsilon,\xi,OP1}(f,x_i,y_i)^{\top}\theta_ka_k\right)-y_i\right]\\
			&&-\frac{2}{n}\sum_{i=1}^n\frac{\phi(0)'\|a\|^2}{{h}}\Delta_{1}(\theta,x_i,y_i)\left[\frac{1}{\sqrt{h}}\sum_{k=1}^h\left(\phi(0)'x_i^{\top}\theta_ka_k+\phi'(0)A_{\epsilon,\xi,OP1}(f,x_i,y_i)^{\top}\theta_ka_k\right)-y_i\right]\\
			&&-\frac{2}{n}\sum_{i=1}^n\frac{\phi(0)'\|a\|^2}{h}\left[x_i+A_{\epsilon,\xi,OP1}(f,x_i,y_i)\right]\Delta_{2}(\theta,x_i,y_i)+o.
		\end{eqnarray*}
		When $\|\theta^{\top}a/\sqrt{h}\|\leq B_1 v$, we have
		\begin{eqnarray*}
			&&\left\|\frac{1}{\sqrt{h}}\sum_{i=1}^ha_jg_{j,\xi}^{OP1}(\theta,\epsilon)- \frac{1}{\sqrt{h}}\sum_{i=1}^ha_jg_{j,\xi}^{OP2}(\theta,\epsilon)\right\|\\
			&=&O\left( v\max_{i}\frac{|\sum_{j=1}^h a_j^2\Delta_{3,j}(\theta,x_i,y_i)|}{h} \right)\\
			&&+\left\|\frac{2}{n}\sum_{i=1}^n\frac{\phi(0)'\|a\|^2}{{h}}\Delta_{1}(\theta,x_i,y_i)\left[\frac{1}{\sqrt{h}}\sum_{k=1}^h\left(\phi(0)'x_i^{\top}\theta_ka_k+\phi'(0)A_{\epsilon,\xi,OP1}(f,x_i,y_i)^{\top}\theta_ka_k\right)-y_i\right]\right\|\\
			&&+\left\|\frac{2}{n}\sum_{i=1}^n\frac{\phi(0)'\|a\|^2}{h}\left[x_i+A_{\epsilon,\xi,OP1}(f,x_i,y_i)\right]\Delta_{2}(\theta,x_i,y_i)\right\|+o.
		\end{eqnarray*}
		We know that with probability at least $1-e^{-g_1(\delta,\xi,d,n)}$,
		\begin{eqnarray*}
			&&\left\|\frac{1}{\sqrt{h}}\sum_{i=1}^ha_jg_{j,\xi}^{OP1}(\theta,\epsilon)- \frac{1}{\sqrt{h}}\sum_{i=1}^ha_jg_{j,\xi}^{OP1}(\theta',\epsilon)\right\|\\
			&\leq&  \frac{\|a\|^2}{h}\|\triangledown \widehat R_{L,\xi}(\theta^{\top}a/\sqrt{h},\epsilon)-\triangledown \widehat R_{L,\xi}((\theta')^{\top}a/\sqrt{h},\epsilon)\|\mbox{ (from linear network to linear model)}\\
			&\leq& 2\delta + \frac{\|a\|^2}{h}\|\triangledown  R_{L,\xi}(\theta^{\top}a/\sqrt{h},\epsilon)-\triangledown  R_{L,\xi}((\theta')^{\top}a/\sqrt{h},\epsilon)\|\\
			&\leq&2\delta +\frac{\|a\|^2}{h} \frac{Lv}{\sqrt{h}\xi}\|\theta^{\top}a-(\theta')^{\top}a \|.
		\end{eqnarray*}
		As a result, with probability at least $1-e^{-g_1(\delta,\xi,d,n)}$,
		\begin{eqnarray}
			&&\left\| \frac{1}{\sqrt{h}}\theta^{OP1}(t+1)^{\top}a-\frac{1}{\sqrt{h}}\theta^{OP2}(t+1)^{\top}a \right\|\label{eqn:from}\\
			&\leq& \left\|\frac{1}{\sqrt{h}} \theta^{OP1}(t)^{\top}a-\frac{1}{\sqrt{h}}\theta^{OP2}(t)^{\top}a \right\| + \eta\left\|\frac{1}{\sqrt{h}}\sum_{i=1}^ha_jg_{j,\xi}^{OP1}(\theta^{OP1}(t),\epsilon)- \frac{1}{\sqrt{h}}\sum_{i=1}^ha_jg_{j,\xi}^{OP1}(\theta^{OP2}(t),\epsilon) \right\|\nonumber\\
			&&+\eta\left\| \frac{1}{\sqrt{h}}\sum_{i=1}^ha_jg_{j,\xi}^{OP1}(\theta^{OP2}(t),\epsilon)- \frac{1}{\sqrt{h}}\sum_{i=1}^ha_jg_{j,\xi}^{OP2}(\theta^{OP2}(t),\epsilon)  \right\|.\label{eqn:to}
		\end{eqnarray}

		As will be mentioned later, due to zero initialization, one can track $\theta_j^{OP1}(t)$ once given $\frac{1}{\sqrt{h}}\theta^{OP1}(t)^{\top}a/\sqrt{h}$. Similar property holds for $\theta_j^{OP2}(t)$. Therefore we bound $ \frac{1}{\sqrt{h}}\theta^{OP1}(t+1)^{\top}a-\frac{1}{\sqrt{h}}\theta^{OP2}(t+1)^{\top}a$, which further enable us to bound $\|\theta^{OP1}_j(t)-\theta^{OP2}_j(t)\|$.

		\paragraph{Linear network in (1)} To further figure out the difference, we need some knowledge on $\theta^{OP1}(t)$, i.e. how $\theta^{OP1}(t)$ affects $\Delta_1$, $\Delta_2$, and $\Delta_3$. 
		
		Observe that $ \frac{\phi'(0)}{\sqrt{h}}\theta^{OP1}(T)^{\top}a$ is just the coefficients of a linear model, so using zero initialization on $\theta^{OP1}(t)$, we have for any $j,k$ in $1,...,h$,
		\begin{eqnarray*}
			\frac{\|\theta^{OP1}_j(t)\|}{\|\theta^{OP1}_k(t)\|}=\frac{|a_j|}{|a_k|}.
		\end{eqnarray*}
		
		Now we study the sufficient conditions on $(a,h)$ which enable us to bound the difference between $\theta^{OP1}(t)$ and $\theta^{OP2}(t)$. In general, we want the Taylor approximation of $\phi(x_i^{\top}\theta_j)$ valid, and $\left\|\frac{1}{\sqrt{h}}\sum_{i=1}^ha_jg_{j,\xi}^{OP1}(\theta,\epsilon)- \frac{1}{\sqrt{h}}\sum_{i=1}^ha_jg_{j,\xi}^{OP2}(\theta,\epsilon)\right\|$ goes to zero.
		
		To validate the approximation of $\phi(x_i^{\top}\theta_k)=\phi(0)+\phi'(0)(x_i^{\top}\theta_k)+o$, we require $\max_{i,j}|\theta_j^{OP1}(t)^{\top}x_i|\rightarrow 0$. Since $\max_i\|x_i\|=O(\sqrt{d\log n})$ and $\|\theta_j^{OP1}(t)\|\leq\sqrt{h}|a_j|\|\theta_0\|/\|a\|^2$, a sufficient condition becomes
		\begin{eqnarray*}
		\sqrt{d\log n}\frac{\|a\|_{\infty}}{\|a\|^2}\|\theta_0\|\sqrt{h}\rightarrow 0.
		\end{eqnarray*}
		(We will use a stronger version $(d\log n)\|a\|_{\infty}v\sqrt{h}/\|a\|^2\rightarrow 0$.)
		
		Next we study the bound of $\left\|\frac{1}{\sqrt{h}}\sum_{i=1}^ha_jg_{j,\xi}^{OP1}(\theta,\epsilon)- \frac{1}{\sqrt{h}}\sum_{i=1}^ha_jg_{j,\xi}^{OP2}(\theta,\epsilon)\right\|$. Recall that $\theta_0$ is the true model (i.e. a vector, not a matrix). By the definition of $\Delta_{3,j}$, we have
		\begin{eqnarray*}
			|\Delta_{3,j}(\theta,x,y)|&=&\left|\phi'\left(\left(x+  A_{\epsilon,\xi}(f,x,y)\right)^{\top}\theta_j\right)-\phi'(0)\right|\\
			&=&O(\phi''(0)|(x+A_{\epsilon,\xi}(f,x,y))^{\top}\theta_j|)\\
			&=& O(\phi''(0)\|\theta_j\|(\|x\|+\epsilon)).
		\end{eqnarray*}
		As a result, since $\max_i\|x_i\|=O(\sqrt{d\log n})$ in probability,  we have
		\begin{eqnarray*}
			\max_i\frac{|\sum_{j=1}^ha_j^2\Delta_{3,j}(\theta^{OP1}(t),x_i,y_i)|}{ \|a\|^2}&=&O\left(\sqrt{d\log n}\max_{j}\|\theta^{OP1}(t)_j\| \right)\\
			&=&O\left( \sqrt{d\log n}\frac{\|a\|_{\infty}}{\|a\|^2} \sqrt{h}\|\theta_0\| \right).
		\end{eqnarray*}
		Similarly, for $\Delta_1$, if the sign of $\frac{\phi'(0)}{\sqrt{h}}x^{\top}\theta^{\top}a-y$ and $f_{\theta}(x)-y$ are the same, then
		\begin{eqnarray*}
			\|A_{\epsilon,\xi}(f,x,y)-	A_{\epsilon,\xi,OP1}(f,x,y)\|=O\left(\frac{\|x\|\sum_{j=1}^h |a_j|\|\theta_j\|^2}{\|\sum_{j=1}^h a_j\theta_j\|}\right)=O\left(\sqrt{d\log n} \frac{\|a\|_{\infty}}{\|a\|^2}\sqrt{h}\|\theta_0\| \right).
		\end{eqnarray*}
		If the signs are different, then since
		\begin{eqnarray*}
		\frac{\phi'(0)}{\sqrt{h}}x_i^{\top}\theta^{\top}a-f_{\theta}(x_i)=O\left(d\log n\sqrt{h}\|\theta_0\|^2\frac{\|a\|_{\infty}}{\|a\|^2}\right),
		\end{eqnarray*}
		there is a proportion of at most $O\left(d\log n\sqrt{h}\|\theta_0\|^2\frac{\|a\|_{\infty}}{\|a\|^2}/v\right)$ of samples such that $A_{\epsilon,\xi}$ and $A_{\epsilon,\xi,OP1}$ are far away from each other. Denote $b_{i}$ as the indicator that the sign of $\frac{\phi'(0)}{\sqrt{h}}x^{\top}\theta^{\top}a-y$ and $f_{\theta}(x)-y$ are the same, then
		\begin{eqnarray*}
		&&\left\|\frac{2}{n}\sum_{i=1}^n\frac{\phi(0)'\|a\|^2}{{h}}\Delta_{1}(\theta,x_i,y_i)\left[\frac{1}{\sqrt{h}}\sum_{k=1}^h\left(\phi(0)'x_i^{\top}\theta_ka_k+\phi'(0)A_{\epsilon,\xi,OP1}(f,x_i,y_i)^{\top}\theta_ka_k\right)-y_i\right]\right\|\\
		&\leq& \left\|\frac{2}{n}\sum_{i=1}^n\frac{\phi(0)'\|a\|^2}{{h}}\Delta_{1}(\theta,x_i,y_i)\left[\frac{1}{\sqrt{h}}\sum_{k=1}^h\left(\phi(0)'x_i^{\top}\theta_ka_k+\phi'(0)A_{\epsilon,\xi,OP1}(f,x_i,y_i)^{\top}\theta_ka_k\right)-y_i\right]1\{b_i=1\}\right\|\\
		&&+\left\|\frac{2}{n}\sum_{i=1}^n\frac{\phi(0)'\|a\|^2}{{h}}\Delta_{1}(\theta,x_i,y_i)\left[\frac{1}{\sqrt{h}}\sum_{k=1}^h\left(\phi(0)'x_i^{\top}\theta_ka_k+\phi'(0)A_{\epsilon,\xi,OP1}(f,x_i,y_i)^{\top}\theta_ka_k\right)-y_i\right]1\{b_i=0\}\right\|\\
		&=&O\left( (d\log n)\|a\|_{\infty}\frac{\|\theta_0\|^2}{\sqrt{h}} \right).
		\end{eqnarray*}
		
		Based on $\Delta_1$, for $\Delta_2$, when $ \sqrt{dh\log n}\|a\|_{\infty}\|\theta_0\|/\|a\|^2\rightarrow 0 $, i.e. $\max_{i,j}|\theta_j^{OP1}(t)^{\top}x_i|\rightarrow 0$, we have
		\begin{eqnarray*}
			|\Delta_{2}(\theta,x,y)|&=&\bigg|\frac{1}{\sqrt{h}}\sum_{j=1}^h \phi\left(\left(x+  A_{\epsilon,\xi}(f,x,y)\right)^{\top}\theta_j\right)a_j-\frac{1}{\sqrt{h}}\sum_{j=1}^h\phi(0)'x^{\top}\theta_ja_j+\phi'(0)A_{\epsilon,\xi,OP1}(f,x,y)^{\top}\theta_ja_j\bigg|\\
			&=&\bigg|\frac{1}{\sqrt{h}}\sum_{j=1}^h\phi(0)'x^{\top}\theta_ja_j+\phi'(0)A_{\epsilon,\xi}(f,x,y)^{\top}\theta_ja_j+o-\frac{1}{\sqrt{h}}\sum_{j=1}^h\phi(0)'x^{\top}\theta_ja_j+\phi'(0)A_{\epsilon,\xi,OP1}(f,x,y)^{\top}\theta_ja_j\bigg|\\
			&=&\bigg|\frac{1}{\sqrt{h}}\sum_{j=1}^h\phi'(0)[A_{\epsilon,\xi}(f,x,y)-A_{\epsilon,\xi,OP1}(f,x,y)]^{\top}\theta_ja_j\bigg|\\
			&\leq& \frac{1}{\sqrt{h}}\sum_{j=1}^h\phi'(0)\Delta_1(\theta,x,y)|a_j|\|\theta_j\|+o\\
			&\leq&\frac{1}{\sqrt{h}}\sum_{j=1}^h\phi'(0)\Delta_1(\theta,x,y) \frac{\sqrt{h}a_j^2}{\|a\|^2}\|\theta_0\|+o\\
			&=&\phi'(0)\Delta_1(\theta,x,y)\|\theta_0\|+o.
		\end{eqnarray*}
		Thus with probability tending to 1,
		\begin{eqnarray*}
			&&\left\|\frac{2}{n}\sum_{i=1}^n\frac{\phi(0)'\|a\|^2}{h}\left[x_i+A_{\epsilon,\xi,OP1}(f,x_i,y_i)\right]\Delta_{2}(\theta,x_i,y_i)\right\|\\
			&\leq& 2\frac{\phi'(0)\|a\|^2}{h}(\max_i\|x_i\|+\epsilon)\phi'(0)\|\theta_0\|\left(\frac{1}{n}\sum_{i=1}^n|\Delta_1(f,x_i,y_i)|\right)+o\\
			&=&O\left((d\log n)\|a\|_{\infty}\frac{\|\theta_0\|^2}{\sqrt{h}}\right).
		\end{eqnarray*}
		To summarize, with probability tending to 1,
		\begin{eqnarray*}
		\left\|\frac{1}{\sqrt{h}}\sum_{i=1}^ha_jg_{j,\xi}^{OP1}(\theta,\epsilon)- \frac{1}{\sqrt{h}}\sum_{i=1}^ha_jg_{j,\xi}^{OP2}(\theta,\epsilon)\right\|
		=O\left( (d\log n)\|a\|_{\infty}\frac{\|\theta_0\|^2}{\sqrt{h}} \right)
		\end{eqnarray*}
		when $(d\log n)\|a\|_{\infty}v^2\sqrt{h}/\|a\|^2\rightarrow 0$.
		
		\paragraph{Return to the difference between (1) and (2)}
		Now we use the property that $\theta^{OP2}(0)=0$. Similar as $\theta^{OP1}$, if $a_ja_k>0$, then
		\begin{eqnarray*}
			\frac{\|\theta_j^{OP2}(t)\|}{\|\theta_k^{OP2}(t)\|}\approx\frac{|a_j|}{|a_k|}.
		\end{eqnarray*}
		Therefore, we define $\theta^{OP2}_+(t)=\sum_{a_j>0}a_j\theta^{OP2}_j(t)$, and similarly define $\theta^{OP2}_-(t)$, $\theta^{OP1}_+(t)$, $\theta^{OP1}_-(t)$. One can obtain similar result for (\ref{eqn:from}) to (\ref{eqn:to}) when considering $\theta^{OP1}_+(t)-\theta^{OP2}_+(t)$ and $\theta^{OP1}_-(t)-\theta^{OP2}_-(t)$.

		When $\eta(\delta+(d\log n)\|a\|_{\infty}v^2/\sqrt{h})(1+\|a\|^2Lv\eta/(h\xi))^T\rightarrow 0$, one can use induction to bound $\|\theta^{OP1}(T)^{\top}a/\sqrt{h}-\theta^{OP2}(T)^{\top}a/\sqrt{h}\|/\|\theta_0\|$: recall that
			\begin{eqnarray*}
			\Delta_{t+1}&:=&\left\| \frac{1}{\sqrt{h}}\theta^{OP1}(t+1)^{\top}a-\frac{1}{\sqrt{h}}\theta^{OP2}(t+1)^{\top}a \right\|\label{eqn:1}\\
			&\leq& \left\|\frac{1}{\sqrt{h}} \theta^{OP1}(t)^{\top}a-\frac{1}{\sqrt{h}}\theta^{OP2}(t)^{\top}a \right\|+ \eta\left\|\frac{1}{\sqrt{h}}\sum_{i=1}^ha_jg_{j,\xi}^{OP1}(\theta^{OP1}(t),\epsilon)- \frac{1}{\sqrt{h}}\sum_{i=1}^ha_jg_{j,\xi}^{OP1}(\theta^{OP2}(t),\epsilon) \right\|\\
			&&+\eta\left\| \frac{1}{\sqrt{h}}\sum_{i=1}^ha_jg_{j,\xi}^{OP1}(\theta^{OP2}(t),\epsilon)- \frac{1}{\sqrt{h}}\sum_{i=1}^ha_jg_{j,\xi}^{OP2}(\theta^{OP2}(t),\epsilon)  \right\|\\
			&:=&\Delta_t+\eta\Delta_{1,t}+\eta\Delta_{2,t}(\theta^{OP1}(t)).
		\end{eqnarray*}
		
	    When studying the linear network $\theta^{OP1}(t)$, we argue that $\Delta_{2,t}(\theta^{OP1}(t))\rightarrow 0$ under proper choice of $(a,h)$. If $\|\theta^{OP2}_j(t)-\theta^{OP1}_j(t)\|/\|\theta^{OP1}_j(\infty)\|\rightarrow 0$, then using $\theta^{OP2}(t)$ we can also obtain $\Delta_{2,t}(\theta^{OP2}(t))\rightarrow 0$. In addition, $\Delta_{1,t}$ is a linear function of $\Delta_t$ plus some error $\delta$. Furthermore, after obtaining $\Delta_{t+1}$, one can further figure out $\|\theta^{OP2}_j(t+1)-\theta^{OP1}_j(t+1)\|$. To conclude, using induction, we have with probability $1-e^{-g_1(\delta,\xi,d,n)}$, $\|\theta^{OP1}(T)^{\top}a/\sqrt{h}-\theta^{OP2}(T)^{\top}a/\sqrt{h}\|/\|\theta_0\|$  converges to zero.
		
		\paragraph{Difference between (2) and (3)}
		
		Denote $A_{2,i}$ and $A_{3,i}$ as the attack for $i$th data w.r.t. $\theta^{OP2}(t)$ and $\theta^{OP3}(t)$, then we have
		\begin{eqnarray*}
		&& g_{j,\xi}^{OP2}(\theta^{OP2}(t),\epsilon)-g_{j,\xi}^{OP2}(\theta^{OP3}(t),\epsilon)\\
		&=&\frac{2}{n}\sum_{i=1}^n\frac{\phi'\left(\left(x+  A_{2,i}\right)^{\top}\theta^{OP2}_j(t)\right)a_j}{\sqrt{h}}\left[x_i+A_{2,i}(t)\right]\left[\frac{1}{\sqrt{h}}\sum_{k=1}^h \phi\left(\left(x_i+  A_{2,i}(t)\right)^{\top}\theta^{OP2}_k(t)\right)a_k-y_i\right]\\
		&&-\frac{2}{n}\sum_{i=1}^n\frac{\phi'\left(\left(x+  A_{3,i}\right)^{\top}\theta^{OP3}_j(t)\right)a_j}{\sqrt{h}}\left[x_i+A_{3,i}(t)\right]\left[\frac{1}{\sqrt{h}}\sum_{k=1}^h \phi\left(\left(x_i+  A_{3,i}(t)\right)^{\top}\theta^{OP3}_k(t)\right)a_k-y_i\right]\\
		&=&\frac{2}{n}\sum_{i=1}^n\frac{\phi'\left(\left(x+  A_{2,i}\right)^{\top}\theta^{OP2}_j(t)\right)a_j}{\sqrt{h}}\left[x_i+A_{2,i}(t)\right]\left[\frac{1}{\sqrt{h}}\sum_{k=1}^h \phi\left(\left(x_i+  A_{2,i}(t)\right)^{\top}\theta^{OP2}_k(t)\right)a_k-y_i\right]\\
		&&-\frac{2}{n}\sum_{i=1}^n\frac{\phi'\left(\left(x+  A_{2,i}\right)^{\top}\theta^{OP3}_j(t)\right)a_j}{\sqrt{h}}\left[x_i+A_{2,i}(t)\right]\left[\frac{1}{\sqrt{h}}\sum_{k=1}^h \phi\left(\left(x_i+  A_{2,i}(t)\right)^{\top}\theta^{OP3}_k(t)\right)a_k-y_i\right]\\
	    &&+\frac{2}{n}\sum_{i=1}^n\frac{\phi'\left(\left(x+  A_{2,i}\right)^{\top}\theta^{OP3}_j(t)\right)a_j}{\sqrt{h}}\left[x_i+A_{2,i}(t)\right]\left[\frac{1}{\sqrt{h}}\sum_{k=1}^h \phi\left(\left(x_i+  A_{2,i}(t)\right)^{\top}\theta^{OP3}_k(t)\right)a_k-y_i\right]\\
		&&-\frac{2}{n}\sum_{i=1}^n\frac{\phi'\left(\left(x+  A_{3,i}\right)^{\top}\theta^{OP3}_j(t)\right)a_j}{\sqrt{h}}\left[x_i+A_{3,i}(t)\right]\left[\frac{1}{\sqrt{h}}\sum_{k=1}^h \phi\left(\left(x_i+  A_{3,i}(t)\right)^{\top}\theta^{OP3}_k(t)\right)a_k-y_i\right]\\
		&:=& (B_{1,j}-B_{2,j})+(B_{2,j}-B_{3,j}).
		\end{eqnarray*}
		The term $B_{1,j}-B_{2,j}$ follows the same as C.8.3 in \cite{ba2020generalization}. For $B_{2,j}-B_{3,j}$, we know that
		\begin{eqnarray*}
		\|A_{2,i}-A_{3,i}\|=O\left( v\left\|\frac{ \sum_{j=1}^h a_j(\phi'(x^{\top}\theta_j^{OP2}(t))\theta_j^{OP2}(t)-\phi'(x^{\top}\theta_j^{OP3}(t))\theta_j^{OP3}(t) )}{\sqrt{h}\xi}\right\| \right).
		\end{eqnarray*}
		Since
				\begin{eqnarray*}
			\|\theta^{OP3}(T)a-\theta^{OP2}(T)a\|\leq \|a\|_{\infty}\|\theta^{OP3}(T)-\theta^{OP2}(T)\|_{F},
		\end{eqnarray*}
		 when $f_{\theta^{OP2}(t)}(x_i)-y_i$ and $f_{\theta^{OP3}(t)}(x_i)-y_i$ have the same sign, assuming $|x_i\theta_j^{OP3}(t)|\rightarrow 0$,
		\begin{eqnarray*}
		\|A_{2,i}-A_{3,i}\|&=&O\left( v\left\|\frac{ \sum_{j=1}^h a_j(\phi'(0) (\theta_j^{OP2}(t)-\theta_j^{OP3}(t))+o}{\sqrt{h}\xi}\right\| \right)\\
		&=&O\left(\frac{v\|a\|_{\infty}\|\theta^{OP3}(t)-\theta^{OP2}(t)\|_{F}}{\sqrt{h}\xi} \right),
		\end{eqnarray*}
		and there is a proportion at most $O(\sqrt{d\log n}\|(\theta^{OP2}(t)-\theta^{OP3}(t))^{\top}a/\sqrt{h}\|/v)$ samples whose $f_{\theta^{OP2}(t)}(x_i)-y_i$ and $f_{\theta^{OP3}(t)}(x_i)-y_i$ have different signs. As a result, denote $b_i'$ as the indicator that $f_{\theta^{OP2}(t)}(x_i)-y_i$ and $f_{\theta^{OP3}(t)}(x_i)-y_i$ have the same sign, then assuming $|x_i\theta_j^{OP3}(t)|\rightarrow 0$,
		\begin{eqnarray*}
		&&\left\|g_{j,\xi}^{OP2}(\theta^{OP2}(t),\epsilon)-g_{j,\xi}^{OP2}(\theta^{OP3}(t),\epsilon)\right\|\\
		&\leq& \|B_{1,j}-B_{2,j}\|\\
		&&+\bigg\|\frac{2}{n}\sum_{i=1}^n\frac{\phi'\left(\left(x+  A_{2,i}\right)^{\top}\theta^{OP3}_j(t)\right)a_j}{\sqrt{h}}\left[x_i+A_{2,i}(t)\right]\left[\frac{1}{\sqrt{h}}\sum_{k=1}^h \phi\left(\left(x+  A_{2,i}(t)\right)^{\top}\theta^{OP3}_k(t)\right)a_k-y_i\right]\\
		&&\quad-\frac{2}{n}\sum_{i=1}^n\frac{\phi'\left(\left(x+  A_{3,i}\right)^{\top}\theta^{OP3}_j(t)\right)a_j}{\sqrt{h}}\left[x_i+A_{3,i}(t)\right]\left[\frac{1}{\sqrt{h}}\sum_{k=1}^h \phi\left(\left(x+  A_{3,i}(t)\right)^{\top}\theta^{OP3}_k(t)\right)a_k-y_i\right]\bigg\|\\
		&=&\|B_{1,j}-B_{2,j}\|\\
		&&+\bigg\|\frac{2}{n}\sum_{i=1}^n\frac{\phi'(0)a_j}{\sqrt{h}}\left[x_i+A_{2,i}(t)\right]\left[\frac{1}{\sqrt{h}}\sum_{k=1}^h \phi'(0)(x_i+A_{2,i})^{\top}\theta_k^{OP3}(t)a_k-y_i\right]\\
		&&-\frac{2}{n}\sum_{i=1}^n\frac{\phi'(0)a_j}{\sqrt{h}}\left[x_i+A_{3,i}(t)\right]\left[\frac{1}{\sqrt{h}}\sum_{k=1}^h \phi'(0)(x_i+A_{3,i})^{\top}\theta_k^{OP3}(t)a_k-y_i\right]\bigg\|+o\\
		&\leq&\|B_{1,j}-B_{2,j}\|\\
		&&+\bigg\|\frac{2}{n}\sum_{i=1}^n\frac{\phi'(0)a_j}{\sqrt{h}}\left[A_{3,i}(t)-A_{2,i}(t)\right]\left[\frac{1}{\sqrt{h}}\sum_{k=1}^h \phi'(0)(x_i+A_{2,i})^{\top}\theta_k^{OP2}(t)a_k-y_i\right]\bigg\|\\
		&&+\bigg\|\frac{2}{n}\sum_{i=1}^n\frac{\phi'(0)a_j}{\sqrt{h}}\left[A_{3,i}(t)\right]\left[\frac{1}{\sqrt{h}}\sum_{k=1}^h \phi'(0)(A_{3,i}-A_{2,i})^{\top}\theta_k^{OP2}(t)a_k-y_i\right]\bigg\|+o\\
		&=& \|B_{1,j}-B_{2,j}\|+O\left(\frac{|a_j|}{\sqrt{h}} \frac{vL\|\theta_0\|\|a\|_{\infty}\|\theta^{OP3}(t)-\theta^{OP2}(t)\|_F}{\sqrt{h}\xi} \right).
		\end{eqnarray*}

		Taking $\xi \gg \sqrt{d\log n}/(\sqrt{h}v)$, we have
		\begin{eqnarray*}
			&&\|\theta^{OP3}(t+1)-\theta^{OP2}(t+1)\|_{F}\\
			&\leq&\|\theta^{OP3}(t)-\theta^{OP2}(t)\|_{F}+\eta\sqrt{\sum_{j=1}^h \left\|g_{j,\xi}^{OP2}(\theta^{OP3}(t),\epsilon)- g_{j,\xi}^{OP2}(\theta^{OP2}(t),\epsilon)\right\|^2}\\
			&=& \|\theta^{OP3}(t)-\theta^{OP2}(t)\|_{F}O\left( 1+\frac{vL\eta\|\theta_0\|\|a\|_{\infty}\|a\|^2}{h\xi}+L v\eta \right).
		\end{eqnarray*}
		Therefore,
		\begin{eqnarray*}
			\|\theta^{OP3}(T)-\theta^{OP2}(T)\|_{F}&=&O\left( \|\theta^{OP3}(0)\|_F \left(1+\frac{vL\eta\|\theta_0\|\|a\|_{\infty}\|a\|^2}{h\xi}+L v\eta\right)^{T} \right)\\
			&=&O\left( \frac{1}{h^{\delta'/2}}\left(1+\frac{vL\eta\|\theta_0\|\|a\|_{\infty}\|a\|^2}{h\xi}+L v\eta\right)^T \right).
		\end{eqnarray*}
		As a result, we require $\|\theta^{OP3}(T)-\theta^{OP2}(T)\|_{F}\sqrt{d\log n}\rightarrow 0$ so that with probability tending to 1, for all nodes, $|x_i^{\top}\theta^{OP3}_j(T)|\rightarrow 0$.

		\paragraph{Deciding proper choice of $\xi$}
		
		In Theorem 1, since $\frac{\phi'(0)}{\sqrt{h}}\theta^{OP1}(t)a$ is a linear model,  denote $\eta_{linear}$ as the learning rate for linear model, then the corresponding $\eta$ in linear network is $\eta=\eta_{linear}h/(\|a\|^2\phi'(0)^2)$. We require $T\eta_{linear}\rightarrow\infty$, $\xi/ v^2\rightarrow 0$, $\eta_{linear}=\xi/(2Lv^2)$, and $\delta v/\xi\rightarrow 0$ in Theorem \ref{thm:opt}.

		Now we list all the assumptions we made on $(\xi,\eta,T,\delta)$ during derivation. 
		\begin{itemize}
			\item Difference between $\theta^{OP1}(t)$ and $\theta^{OP2}(t)$: $e^{-g_1(\delta,\xi,d,n)}\rightarrow 0$, $(d\log n)\|a\|_{\infty}v/\sqrt{h}\rightarrow 0$, $(d\log n)\frac{\|a\|_{\infty}}{\|a\|^2}\|\theta_0\|^2/v\rightarrow 0$ (the gradient of linear network is in $\Theta(v)$ when it is far from $\theta^*$, thus we divide $v$ here), and $\eta(\delta+(d\log n)\|a\|_{\infty}v^2/\sqrt{h})(1+\|a\|^2Lv^2\eta/(h\xi)+L)^T/v\rightarrow 0$. 
			\item Difference between $\theta^{OP2}(t)$ and $\theta^{OP3}(t)$: $\frac{\sqrt{d\log n}}{h^{\delta'/2}}\left(1+\frac{vL\eta\|\theta_0\|\|a\|_{\infty}\|a\|^2}{h\xi}+L v\eta\right)^T/v\rightarrow 0$.
		\end{itemize}
		
		So we take $\delta=v\sqrt{(d^2\log n)/n}$, and $$\xi/v^2=-\log\log n/\log\left(\sqrt{\frac{d^2\log n}{n}}\vee (d\log n)\frac{\|a\|_{\infty}}{\sqrt{h}}\right).$$
		
		Finally, we adapt the above proof for the true attack. From Taylor expansion, we have when $x^{\top}\theta_j \rightarrow 0$,
		\begin{eqnarray*}
		    \phi(z^{\top}\theta_j)\approx \phi(x^{\top}\theta_j\phi'(0)(z-x)^{\top}\theta_j+O( \|\theta_j\|^2 ),
		\end{eqnarray*}
		which implies that the fast gradient attack used in the above proof converges to the true attack. Recall that we require $\Delta_3\rightarrow 0$, thus for $\theta^{OP2}(t)$, the arguments also holds when replacing fast gradient attack by true attack. The same argument for $\theta^{OP3}(t)$ as well.
		
	\end{proof}
	\begin{proof}[Proof of Theorem \ref{thm:high_network} ]
		The proof is similar as the one for Theorem \ref{thm:low_relu} through replacing $\theta_0$ with $\theta(\y)$. Based on Lemma \ref{lem:ridge}, we know that $\|\theta(\y)\|/v=O_p(n/d)$, thus in the step ``Linear Network in (1)", to ensure $\|f_{\theta_{\xi}^{(t)}}(x_i)-\y\|$ decreases in each iteration,  we require $\sqrt{d\log n}\sqrt{n\log n}\|a\|_{\infty}v/\sqrt{h}\rightarrow 0$. On the other hand, for the step ``Difference between (2) and (3)", we take $a$ and $\delta'$ such that $\sqrt{d\log n}\|\theta^{OP3}(t)a-\theta^{OP2}(t)a\|=o(1)$. 
	\end{proof}
	\begin{proof}[Proof of Theorem \ref{thm:low_relu} and \ref{thm:high_relu}]
		Assume $X=[X_0,-X_0]^{\top}$. Since $\|a^+\|=\|a^-\|$, one can check that all $\Delta_1$, $\Delta_2$, and $\Delta_3$ are zeros. Take $\phi'(0)=1$. Denote $S^+(\theta)$ as $\{ j|\;f\left(x+  A_{\epsilon,\xi}(f,x_i,y_i)\right)>0   \}$, and similarly denote $S^-(\theta)$. Then one can observe that, using zero initialization,
		\begin{eqnarray*}
			\frac{1}{\sqrt{h}}\sum_{j\in S^+(\theta)}a_jg_{j,\xi}^{OP2}(\theta,\epsilon)+\frac{1}{\sqrt{h}}\sum_{j\in S^-(\theta)}a_jg_{j,\xi}^{OP2}(\theta,\epsilon)=0,
		\end{eqnarray*}
		and hence ReLU-activated neural networks with zero initialization perform the same as linear networks when the learning rate for ReLU-activated neural is twice as the one for linear networks. 

	\end{proof}

	\section{Proofs for high-dimensional dense model}
	\begin{proof}[Proof of Lemma \ref{lem:ridge}]
		Denote $\theta^+=\theta(\y)$.  When $v$ is a constant, $\|\theta^+\|$ satisfies
		\begin{eqnarray*}
			\|\theta^+\|^2&=&\y^{\top}(\X\X^{\top})^{-1}\y\\
			&=&\theta_0^{\top}\X^{\top}(\X\X^{\top})^{-1}\X\theta_0+\epsilon^{\top}(\X\X^{\top})^{-1}\epsilon\\
			&\rightarrow&\theta_0^{\top}\X^{\top}(\X\X^{\top})^{-1}\X\theta_0+ \sigma^2 tr((\X\X^{\top})^{-1}).
		\end{eqnarray*}
		Furthermore, similar as \cite{belkin2019two}, we have
		\begin{eqnarray*}
			(\theta_0^{\top}\Sigma^{1/2})\Sigma^{-1/2}\X^{\top}(\X\Sigma^{-1}\X^{\top})^{-1}\X \Sigma^{-1/2}(\Sigma^{1/2}\theta_0)\rightarrow \frac{n}{d}\|\theta_0\Sigma^{1/2}\|^2.
		\end{eqnarray*}
		Therefore, since $\X\Sigma^{-1}\X^{\top} - \frac{2}{\lambda_{\min}} \X\X^{\top}$ is negative definite for the smallest eigenvalue $\lambda_{\min}$ of $\Sigma$ and , we have
		\begin{eqnarray*}
			\theta_0^{\top}\X^{\top}(\X\X^{\top})^{-1}\X\theta_0&\leq& \frac{2}{\lambda_{\min}} (\theta_0^{\top}\Sigma^{1/2})\Sigma^{-1/2}\X^{\top}(\X\Sigma^{-1}\X^{\top})^{-1}\X \Sigma^{-1/2}(\Sigma^{1/2}\theta_0)\\
			&\rightarrow& \frac{n}{d}\|\theta_0\Sigma^{1/2}\|^2.
		\end{eqnarray*}
		Finally, if $v\rightarrow\infty$, we obtain that $	\theta_0^{\top}\X^{\top}(\X\X^{\top})^{-1}\X\theta_0/v^2\rightarrow 0 $.
	\end{proof}

\begin{lemma}\label{lem:norm}
	When $(\log n)\sqrt{n/d}\rightarrow 0$, with probability tending to 1, the smallest eigenvalue of $\X\X^{\top}$ is in $\Theta (d)$.
\end{lemma}
\begin{proof}[Proof of Lemma \ref{lem:norm}]
	Assume $\Sigma=I$ for simplicity. Denote $b\in(0,1)$. Since $(\log n)\sqrt{n/d}\rightarrow 0$, we append $(bd-n)$ i.i.d samples of $x$ after $\X$ and denote the new data matrix as $\Z$. Based on \cite{bai2008limit}, the smallest eigenvalue of $\Z\Z^{\top}/d$ converges to $(1-\sqrt{b})^2$, and the largest eigenvalue converges to $(1+\sqrt{b})^2$. Since $\lambda_{\min}(\X\X^{\top}/d)\geq \lambda_{\min}(\Z\Z^{\top}/d)$,  and $\lambda_{\min}(\X\X^{\top}/d)\leq \lambda_{\max}(\X\X^{\top}/d)\leq \lambda_{\max}(\Z\Z^{\top}/d)$, we conclude that $\lambda_{\min}=\Theta(d)$ in probability.
\end{proof}

	\begin{proof}[Proof of Theorem \ref{thm:high_linear}]
		Assume $v$ is constant for simplicity and we take $\xi=0$. 
		Denote $\theta^{OP1}(t)$ and $\theta^{OP2}(t)$ satisfy
		\begin{eqnarray*}
			\theta^{OP1}(0)&=&0,\\
			\theta^{OP1}(t+1)&=&\theta^{OP1}(t)-\eta \triangledown\widehat{R}_{\xi}(\theta^{OP1}(t),0),
		\end{eqnarray*}
		and
		\begin{eqnarray*}
			\theta^{OP2}(0)&=&0,\\
			\theta^{OP2}(t+1)&=&\theta^{OP2}(t)-\eta \triangledown\widehat{R}_{\xi}(\theta^{OP2}(t),\epsilon).
		\end{eqnarray*}
	
	\paragraph{Proof sketch} The proof idea is that, since \cite{ba2020generalization} has studied $\theta^{OP1}(t)$, we want know how $\theta^{OP2}$ is closed to $\theta^{OP1}$. We show that when $\|\theta\|$ satisfies some certain condition \textbf{M}, $\widehat R_{\xi}(\theta,0)$ will get decreased in the next update in adversarial training. As a result, when $\theta^{OP2}$ satisfies these condition \textbf{M}, $\widehat R_{\xi}(\theta^{OP2}(t),0)$ is always decreasing. On the hand, when $\theta^{OP2}$ satisfies condition \textbf{M}, one can also show that $\|\triangledown\widehat{R}_{\xi}(\theta^{OP2}(t),\epsilon)-\triangledown\widehat{R}_{\xi}(\theta^{OP2}(t),0)\|=o(\|\triangledown\widehat{R}_{\xi}(\theta^{OP2}(t),0)\|)$. Since zero initialization satisfies condition \textbf{M} and clean training satisfies condition \textbf{M}, one can use induction to show that adversarial training is dominated by clean training. Finally, we obtain $\widehat R_{\xi}(\theta^{OP2}(T),0)\rightarrow 0$ and $\|\theta^{OP2}(T)\|\rightarrow 0$ if $\eta$ and $T$ are chosen properly.
	
\paragraph{Condition M} $\|\theta\|=O(\sqrt{n/d})$,  $\|\z\|=O(\sqrt{n})$, and $\|\z-\y\|>\sqrt{n}/\sqrt{\log n}$.

Now we begin our proof.
For $\theta^{OP2}(t)$, the change on $\widehat R_{\xi}(\theta,0)$ is
\begin{eqnarray*}
\widehat R_{\xi}(\theta^{OP2}(t+1),0)=\widehat R_{\xi}(\theta^{OP2}(t),0)-\frac{2\eta}{n} \triangledown\widehat R_{\xi}(\theta^{OP2}(t),\epsilon)^{\top}\X^{\top}(\X\theta^{OP2}(t)-\y) + \eta^2\|\X\triangledown\widehat R_{\xi}(\theta^{OP2}(t),\epsilon) \|^2.
\end{eqnarray*}
Rewrite $\X\theta=\z$ for simplicity, then
		\begin{eqnarray*}
		\X\triangledown\widehat{R}_{\xi}(\theta,\epsilon)= \frac{2\X\X^{\top}(\z-\y)}{n}+\frac{2\epsilon}{n}\|\theta\|\X\X^{\top}\sgn(\z-\y)+\frac{2\epsilon }{n}\frac{\z}{\|\theta\|}\|\z-\y\|_1+2\epsilon^2\z.
		\end{eqnarray*}
		Since with probability tending to 1, $\|\X\X^{\top}\sgn(\z-\y)\|=O(d\sqrt{n})$, thus when $\|\X\X^{\top}(\z-\y)\|=\Omega( d\|\z-\y\|)$, with probability tending to 1, we have
		\begin{eqnarray*}
		\left\|\frac{2\epsilon }{n}\frac{\z}{\|\theta\|}\|\z-\y\|_1\right\|\leq \frac{2\epsilon }{\sqrt{n}}\frac{\left\|\X\theta\right\|}{\|\theta\|}\|\z-\y\|_2 =o\left(\left\|\frac{\X\X^{\top}(\z-\y)}{n}\right\|\right),
		\end{eqnarray*}
		and
		\begin{eqnarray*}
		\left\|\frac{2\epsilon}{n}\|\theta\|\X\X^{\top}\sgn(\z-\y)+2\epsilon^2\z\right\|=\left\|\frac{\X\X^{\top}(\z-\y)}{n}\right\|O\left( \left(\|\theta\|+\frac{\sqrt{n}\|\z\|}{d}\right) \frac{\sqrt{n}}{\|\z-\y\|} \right).
		\end{eqnarray*}
		The statement $\|\X\X^{\top}(\z-\y)\|=\Omega( d\|\z-\y\|)$ for any $\z$ holds based on Lemma \ref{lem:norm}. 
		
		Consequently, under condition \textbf{M}, 
		\begin{eqnarray*}
\widehat R_{\xi}(\theta^{OP2}(t+1),0)&=&\widehat R_{\xi}(\theta^{OP2}(t),0)-\frac{4\eta}{n^2} (\z-\y)^{\top} \X\X^{\top}(\z-\y)+o,
		\end{eqnarray*}
		which implies that the decrease in $\widehat R_{\xi}(\theta^{OP2}(t),0)$ is almost the same as $\widehat R_{\xi}(\theta^{OP1}(t),0)$.
		
		Now our aim becomes to figure out $\|\theta^{OP2}(t)\|$ when $\|\z-\y\|>\sqrt{n}/\sqrt{\log n}$. 
		
		For $\theta^{OP1}(t)$, i.e. standard training, when $\eta$ is small enough such that the largest eigenvalue of $\eta \X^{\top}\X$ is smaller than 1, then
		\begin{eqnarray*}
			&&\theta^{OP1}(t)=\X^{\top}(I-(I-\eta \X\X^{\top}/n)^{t})(\X\X^{\top})^{-1}\y,\\
			&&\X\theta^{OP1}(t)= \y-(I-\eta \X\X^{\top}/n)^{t}\y,
		\end{eqnarray*}  
		which means that $\|\X\theta^{OP1}(t)-\y \|$ monotonly decreases in $t$. Solving $\|(I-\eta\X\X^{\top}/n)^T\y\|/\sqrt{n}=1/\sqrt{\log n}$, we obtain $\eta T=O(\frac{n}{d}\log\log n)$.

				Observe that
		\begin{eqnarray*}
		&&\left\|\theta^{OP1}(t+1)-\theta^{OP2}(t+1)\right\|\\
		&=&\left\|\left(\theta^{OP1}(t)-\theta^{OP2}(t) \right)-\eta \left( \triangledown\widehat{R}_{\xi}(\theta^{OP1}(t),0) - \triangledown\widehat{R}_{\xi}(\theta^{OP2}(t),0)+ \triangledown\widehat{R}_{\xi}(\theta^{OP2}(t),0)- \triangledown\widehat{R}_{\xi}(\theta^{OP2}(t),\epsilon)\right)\right\|\\
		&\leq&\left\|\theta^{OP1}(t)-\theta^{OP2}(t) \right\|\left(1+O\left(\frac{\eta d}{n}\right)\right)+\eta\left\| \triangledown\widehat{R}_{\xi}(\theta^{OP2}(t),0)- \triangledown\widehat{R}_{\xi}(\theta^{OP2}(t),\epsilon)\right\|.
		\end{eqnarray*}
		
		 When $\theta=\theta^{OP1}(t)$ for some $t>0$ and $\|\theta\|=O(\sqrt{n/d})$. Since $\z-\y=(I-\eta\X\X^{\top}/n)^{t}\y$, $\|\X^{\top}(\z-\y)\|\rightarrow c(d,n,t) \sqrt{d}\|\z-\y\|$ for some function $c(d,n,t)$ which is finite and bounded away from zero. As a result, with probability tending to 1, for the difference between $\triangledown\widehat{R}_{\xi}(\theta,\epsilon)$ and $\triangledown\widehat{R}_{\xi}(\theta,0)$, it becomes
			\begin{eqnarray}\label{eqn:gradient}
			\left\|\triangledown\widehat{R}_{\xi}(\theta,\epsilon)-\triangledown\widehat{R}_{\xi}(\theta,0)\right\|&\leq&\left\|\frac{2\epsilon}{n}\|\theta\|\X^{\top}\sgn(\z-\y)\right\|+\left\|\frac{2\epsilon}{n}\frac{\theta}{\|\theta\|}\|\z-\y\|_1\right\|+\left\|2\epsilon^2\theta\right\|\nonumber\\
			&=&\left\|\triangledown\widehat{R}_{\xi}(\theta,0)\right\|O\left(\frac{n/\sqrt{d}}{\|\z-\y\|} +\sqrt{\frac{n}{d}}+\frac{n^{3/2}/d}{\|\z-\y\|} \right)\nonumber\\
			&=&\left\|\triangledown\widehat{R}_{\xi}(\theta,0)\right\|O\left( \frac{n/\sqrt{d}}{\|\z-\y\|} \right).
			\end{eqnarray}
		When $\|\theta-\theta^{OP1}(t)\|=o(\|\theta^{OP1}(t)\|)$, a similar result can be obtained.

		As a result, one can use induction to show that $\|\theta^{OP2}(T)\|=O(\|\theta^{OP1}(T)\|+\eta T)=o(\sqrt{n/d})$, $\|\X\theta^{OP2}(t)\|=O(\sqrt{n})$ and $\|\X\theta^{OP2}(t)-\y\|$ decreases in $t$, i.e. condition \textbf{M} holds when taking $\eta T=O(\frac{n}{d}\log\log n)$. So the final conclusion holds.
		
	\end{proof}
	
	\section{Proofs for high-dimensional sparse model}
	\begin{proof}[Proof of Theorem \ref{thm:lasso}]
		Assume $v=O(1)$ first. For simplicity we assume $\xi=0$. Denote $$\frac{1}{n}\sum_{i=1}^n l(f_{\theta}(x_i+A_{\epsilon}(f_{\theta},x_i,y_i)),y_i)+\lambda\|\theta\|_1:=\frac{1}{n}\sum_{i=1}^n l_{\epsilon}(\theta,x_i,y_i)+\lambda\|\theta\|_1.$$
		Since $\widehat{\theta}$ minimizes the empirical penalized loss function, take $\Delta=\widehat{\theta}-\theta^*$, we have
		\begin{eqnarray*}
			\frac{1}{n}\sum_{i=1}^n l_{\epsilon}(\widehat{\theta},x_i,y_i)-\frac{1}{n}\sum_{i=1}^n l_{\epsilon}(\theta^*,x_i,y_i)\leq \lambda\|\theta^*\|_1-\lambda\|\widehat\theta\|_1\leq \lambda(\|\Delta_{S}\|_1-\|\Delta_{S^c}\|_1).
		\end{eqnarray*}
		Moreover, the structure of $l$ implies that it is a convex function, thus
		\begin{eqnarray*}
			\Delta^{\top}\frac{1}{n}\sum_{i=1}^n  l_{\epsilon}'(\theta^*,x_i,y_i)\leq \lambda(\|\Delta_{S}\|_1-\|\Delta_{S^c}\|_1).
		\end{eqnarray*}
		Further,
		\begin{eqnarray*}
			\lambda(\|\Delta_{S}\|_1-\|\Delta_{S^c}\|_1)\geq -\left|\Delta^{\top}\frac{1}{n}\sum_{i=1}^n  l_{\epsilon}'(\theta^*,x_i,y_i)\right|\geq -\|\Delta\|_1 \left\|\frac{1}{n}\sum_{i=1}^n  l_{\epsilon}'(\theta^*,x_i,y_i)\right\|_{\infty}.
		\end{eqnarray*}
		Since $\theta^*$ is fixed, we can figure out that with probability tending to 1,
		\begin{eqnarray*}
			\left\|\frac{1}{n}\sum_{i=1}^n  l_{\epsilon}'(\theta^*,x_i,y_i)\right\|_{\infty}\leq c_1\sqrt{\frac{ s\log d }{n}}
		\end{eqnarray*}
		for some constant $c_1>0$. Consequently, as our choice of $M$ is large enough, with probability tending to 1,
		\begin{eqnarray*}
			\|\Delta_S\|_1\geq \frac{\lambda-c_1\sqrt{\frac{ s\log d }{n}}}{\lambda+c_1\sqrt{\frac{ s\log d }{n}}}\|\Delta_{S^c}\|_1.
		\end{eqnarray*}
		As a result, from Lemma \ref{lem:lasso}, we know that with probability tending to 1, for some constant $c_2>0$,
		\begin{eqnarray*}  \lambda\|\Delta_S\|_1\geq \lambda\|\Delta_{S}\|_1-\lambda\|\Delta_{S^c}\|_1\geq \Delta^{\top}\frac{1}{n}\sum_{i=1}^n  l_{\epsilon}'(\theta^*,x_i,y_i) -c_2\|\Delta\|_1\sqrt{\frac{s\log d}{n}}  +\epsilon^2\left(1-\frac{2}{\pi}\right)\|\Delta\|^2.
		\end{eqnarray*}
		Therefore,
		\begin{eqnarray*}
			\lambda\|\Delta_S\|_1+\left(c_1\sqrt{\frac{s\log d}{n}}+c_2\sqrt{\frac{s\log d}{n}}\right)\|\Delta\|_1&\geq& \epsilon^2\left(1-\frac{2}{\pi}\right)\|\Delta\|^2\\&\geq& \epsilon^2\left(1-\frac{2}{\pi}\right)\frac{\|\Delta_{S}\|_1^2}{s},
		\end{eqnarray*}
		hence with probability tending to 1,
		\begin{eqnarray*}
			\|\Delta\|_1=O(\lambda),\; \|\Delta\|_2=O(\lambda).
		\end{eqnarray*}
		
		Finally, using Lemma \ref{lem:lasso_2}, we have 
		\begin{eqnarray*}
			&&l_{\epsilon}(\widehat{\theta},x_i,y_i)- l_{\epsilon}(\theta^*,x_i,y_i)\\&\leq& \Delta^{\top}\frac{\partial l_{\epsilon}(\theta^*,x_i,y_i)}{\partial\theta^*}+(\Delta^{\top}x_i)^2+\epsilon^2\|\Delta\|^2+4\epsilon\|\theta^*+\Delta\||\Delta^\top x_i|\\&&+2\epsilon(\|\theta^*+\Delta\|-\|\theta^*\|)\Delta^{\top}x_i+2\epsilon\frac{\|\theta^*+\Delta\|\|\theta^*\|-(\theta^*+\Delta)^{\top}\theta^*}{\|\theta^*\|}|y_i-x_i^{\top}\theta^*|.
		\end{eqnarray*}
		Taking expectation on $(x_i,y_i)$, it becomes
		\begin{eqnarray*}
			&&\mathbb{E}\left(l_{\epsilon}(\widehat{\theta},x_i,y_i)- l_{\epsilon}(\theta^*,x_i,y_i)\right)\\
			&\leq& (1+\epsilon^2)\|\Delta\|^2+O\left(\|\theta^*+\Delta_{S}\|\|\theta^*\|-(\theta^*+\Delta)^{\top}\theta^*\right)+O\left( \|\Delta\|
			\right)=O(\|\Delta\|).
		\end{eqnarray*}
		Therefore, if $\lambda\rightarrow 0$, with probability tending to 1,
		\begin{equation*}
		{R}_0(\widehat{\theta},\epsilon)-R_0(\theta^*,\epsilon)\rightarrow 0.
		\end{equation*}
		To extend for general $v$, similar with Theorem \ref{thm:opt}, the change on $v$ does not affect the convergence property of $\widehat\theta$ after adjustment w.r.t. $v$.
	\end{proof}
	\begin{lemma}\label{lem:lasso}
		Under the conditions in Theorem \ref{thm:lasso}, when $\|\Delta_{S^c}\|_1\leq c\|\Delta_{S}\|_1$, with probability tending to 1, for some $c_3>0$,
		\begin{eqnarray*}
			&&\frac{1}{n}\sum_{i=1}^nl_{\epsilon}(\widehat{\theta},x_i,y_i)-l(\theta^*,x_i,y_i)\\
			&\geq&  \frac{1}{n}\sum_{i=1}^n\Delta^{\top}\frac{\partial l_{\epsilon}}{\partial \theta} +(\Delta^{\top}x_i)^2-\|\Delta\|^2+\epsilon^2(1-2/\pi)\|\Delta\|^2-c_3\|\Delta\|_1\sqrt{\frac{s\log d}{n}},
		\end{eqnarray*}
		and
		\begin{eqnarray*}
			\frac{1}{\|\Delta_S\|_1^2}\left|\frac{1}{n}\sum_{i=1}^n (\Delta^{\top}x_i)^2-\|\Delta\|^2\right|\rightarrow 0.
		\end{eqnarray*}
	\end{lemma}
	\begin{proof}
		Assume $x^{\top}\theta_2-y>0$, then for any $\theta_1$,
		\begin{eqnarray*}
			&&(x^{\top}\theta_1-y)^2+2\epsilon\|\theta_1\||x^{\top}\theta_1-y|+\epsilon^2\| \theta_1 \|^2-(x^{\top}\theta_2-y)^2-2\epsilon\|\theta_2\|(x^{\top}\theta_2-y)-\epsilon^2\| \theta_2 \|^2\\
			&\geq&(x^{\top}\theta_1-y)^2+2\epsilon\|\theta_1\|(x^{\top}\theta_1-y)+\epsilon^2\| \theta_1 \|^2-(x^{\top}\theta_2-y)^2-2\epsilon\|\theta_2\|(x^{\top}\theta_2-y)-\epsilon^2\| \theta_2 \|^2\\
			&=&2(\theta_1-\theta_2)^{\top}x(x^{\top}\theta_2-y)+(\theta_1-\theta_2)^{\top}xx^{\top}(\theta_1-\theta_2)+2\epsilon^2(\theta_1-\theta_2)^{\top}\theta_2+\epsilon^2\|\theta_1-\theta_2\|^2\\
			&&+2\epsilon\|\theta_1\|(x^{\top}\theta_1-y)-2\epsilon\|\theta_2\|(x^{\top}\theta_2-y)-2\epsilon(\theta_1-\theta_2)\left[ \frac{\theta_2}{\|\theta_2\|}(x^{\top}\theta_2-y)+\|\theta_2\|x \right]\\
			&&+2\epsilon(\theta_1-\theta_2)\left[ \frac{\theta_2}{\|\theta_2\|}(x^{\top}\theta_2-y)+\|\theta_2\|x \right],
		\end{eqnarray*}
		where
		\begin{eqnarray*}
			&&\|\theta_1\|(x^{\top}\theta_1-y)-\|\theta_2\|(x^{\top}\theta_2-y)-(\theta_1-\theta_2)\left[ \frac{\theta_2}{\|\theta_2\|}(x^{\top}\theta_2-y)+\|\theta_2\|x \right]\\
			&=&\|\theta_1\|(x^{\top}\theta_2-y)+\|\theta_1\|(\theta_1-\theta_2)^{\top}x-\theta_1\left[ \frac{\theta_2}{\|\theta_2\|}(x^{\top}\theta_2-y)+\|\theta_2\|x \right]+\theta_2^{\top}x\|\theta_2\|\\
			&=&\frac{\|\theta_1\|\|\theta_2\|-\theta_1^{\top}\theta_2}{\|\theta_2\|}(x^{\top}\theta_2-y)+(\|\theta_1\|-\|\theta_2\|)(\theta_1-\theta_2)^{\top}x\\
			&\geq&(\|\theta_1\|-\|\theta_2\|)(\theta_1-\theta_2)^{\top}x\sgn(x^{\top}\theta_2-y)
		\end{eqnarray*}
		Thus if $x_i^{\top}\theta^*-y_i>0$,
		\begin{eqnarray}\label{eqn:lower}
		l_{\epsilon}(\widehat{\theta},x_i,y_i)-l_{\epsilon}(\theta^*,x_i,y_i)\geq \Delta^{\top}\frac{\partial l_{\epsilon}(\theta^*,x_i,y_i)}{\partial \theta^*}+(\Delta^{\top}x_i)^2+\epsilon^2\|\Delta\|^2+2\epsilon(\|\theta^*+\Delta\|-\|\theta^*\|)\Delta^{\top}x_i
		\end{eqnarray}
		When $x^{\top}\theta_2-y<0$,
		\begin{eqnarray*}
			&&\|\theta_1\|(y-x^{\top}\theta_1)-\|\theta_2\|(y-x^{\top}\theta_2)-(\theta_1-\theta_2)\left[ \frac{\theta_2}{\|\theta_2\|}(y-x^{\top}\theta_2)+\|\theta_2\|x \right]\\
			&=&\|\theta_1\|(y-x^{\top}\theta_2)-\|\theta_1\|(\theta_1-\theta_2)^{\top}x-\theta_1\left[ \frac{\theta_2}{\|\theta_2\|}(y-x^{\top}\theta_2)+\|\theta_2\|x \right]+\theta_2^{\top}x\|\theta_2\|\\
			&=&\frac{\|\theta_1\|\|\theta_2\|-\theta_1^{\top}\theta_2}{\|\theta_2\|}(y-x^{\top}\theta_2)-(\|\theta_1\|-\|\theta_2\|)(\theta_1-\theta_2)^{\top}x\\
			&\geq&-(\|\theta_1\|-\|\theta_2\|)(\theta_1-\theta_2)^{\top}x\\
			&=&(\|\theta_1\|-\|\theta_2\|)(\theta_1-\theta_2)^{\top}x\sgn(x^{\top}\theta_2-y).
		\end{eqnarray*}
		As a result,
		\begin{eqnarray}
		&&\frac{1}{n}\sum_{i=1}^nl_{\epsilon}(\widehat{\theta},x_i,y_i)-l(\theta^*,x_i,y_i)\nonumber\\
		&\geq&\frac{1}{n}\sum_{i=1}^n \Delta^{\top}\frac{\partial l_{\epsilon}}{\partial \theta}+(\Delta^{\top}x_i)^2+\epsilon^2\|\Delta\|^2+2\epsilon(\|\theta^*+\Delta\|-\|\theta^*\|)\Delta^{\top}x_i\sgn(x_i^{\top}\theta^*-y_i)\nonumber\\
		&=&\frac{1}{n}\sum_{i=1}^n \Delta^{\top}\frac{\partial l_{\epsilon}}{\partial \theta}+(\Delta^{\top}x_i)^2+\epsilon^2\|\Delta\|^2+2\epsilon(\|\theta^*+\Delta\|-\|\theta^*\|)(\Delta^{\top}_Sx_{i,S}+\Delta^{\top}_{S^c}x_{i,S^c})\sgn(x_i^{\top}\theta^*-y_i)\nonumber\\
		&\geq& \frac{1}{n}\sum_{i=1}^n\Delta^{\top}\frac{\partial l_{\epsilon}}{\partial \theta}+(\Delta^{\top}x_i)^2+\epsilon^2\|\Delta\|^2\nonumber\\&&-2\epsilon\left|\|\theta^*+\Delta\|-\|\theta^*\|\right|\left(\|\Delta_{S^c}\|_1\left\|\frac{1}{n}\sum_{i=1}^nx_{i,S^c}\sgn(x_i^{\top}\theta^*-y_i)\right\|_{\infty}+O_p\left(\|\Delta_S\|_1 \sqrt{\frac{s\log s}{n}} \right) \right)\nonumber\\
		&&+2\epsilon(\|\theta^*+\Delta\|-\|\theta^*\|)\mathbb{E}\Delta^{\top}_Sx_{i,S}\sgn(x_i^{\top}\theta^*-y_i) \nonumber\\
		&\geq& \frac{1}{n}\sum_{i=1}^n\Delta^{\top}\frac{\partial l_{\epsilon}}{\partial \theta}+(\Delta^{\top}x_i)^2+\epsilon^2\|\Delta\|^2\nonumber\\&&-2\epsilon\left|\|\theta^*+\Delta\|-\|\theta^*\|\right|\left(\|\Delta_{S^c}\|_1\left\|\frac{1}{n}\sum_{i=1}^nx_{i,S^c}\sgn(x_i^{\top}\theta^*-y_i)\right\|_{\infty}+O_p\left(\|\Delta_S\|_1 \sqrt{\frac{s\log s}{n}} \right) \right)\label{eqn:merge1}\\
		&&-2\epsilon\|\Delta\|\sqrt{\frac{2}{\pi} }\left|\Delta^{\top}\frac{\partial  }{\partial \theta}\sqrt{ \|\theta-\theta_0\|^2 +\sigma^2} \right|\nonumber\\
		&\geq&  \frac{1}{n}\sum_{i=1}^n\Delta^{\top}\frac{\partial l_{\epsilon}}{\partial \theta} +(\Delta^{\top}x_i)^2-\|\Delta\|^2+\epsilon^2(1-2/\pi)\|\Delta\|^2-O_p\left(\|\Delta\|_1\sqrt{\frac{s\log d}{n}}\right).\label{eqn:merge2}
		\end{eqnarray}
		
		From (\ref{eqn:merge1}) to (\ref{eqn:merge2}),
		\begin{eqnarray*}
			&&\|\Delta\|^2-2\epsilon\|\Delta\|\sqrt{\frac{2}{\pi}}\left|\Delta^{\top}\frac{\partial }{\partial \theta}\sqrt{\|\theta-\theta_0\|^2+\sigma^2}\right|\\&\geq&\left(\|\Delta\|-\epsilon\sqrt{\frac{2}{\pi}}\left|\Delta^{\top}\frac{\partial }{\partial \theta}\sqrt{\|\theta-\theta_0\|^2+\sigma^2}\right|  \right)^2-\epsilon^2\frac{2}{\pi}\left|\Delta^{\top}\frac{\partial }{\partial \theta}\sqrt{\|\theta-\theta_0\|^2+\sigma^2}\right|^2\\
			&\geq& -\epsilon^2\frac{2}{\pi}\|\Delta\|^2.
		\end{eqnarray*}
		
		When $\|\Delta_{S^c}\|_1\leq c\|\Delta_{S}\|_1$, since $\varepsilon_{d,n}:=\max_{i,j}|\widehat{\Sigma}_{i,j}-I_{i,j}|\rightarrow 0$ in probability, we have with probability tending to 1,
		\begin{eqnarray*}
			\frac{1}{\|\Delta_S\|_1^2}\left|\frac{1}{n}\sum_{i=1}^n (\Delta^{\top}x_i)^2-\|\Delta\|^2\right|=\frac{1}{\|\Delta_S\|_1^2}\left|\Delta^{\top}(I-\widehat{\Sigma})\Delta\right|\leq\frac{1}{\|\Delta_S\|_1^2} \|\Delta\|_1^2\varepsilon_{d,n}\rightarrow 0.
		\end{eqnarray*}
	\end{proof}
	\begin{lemma}\label{lem:lasso_2}
	    \begin{eqnarray*}
			&&l_{\epsilon}(\widehat{\theta},x_i,y_i)- l_{\epsilon}(\theta^*,x_i,y_i)\\&\leq& \Delta^{\top}\frac{\partial l_{\epsilon}(\theta^*,x_i,y_i)}{\partial\theta^*}+(\Delta^{\top}x_i)^2+\epsilon^2\|\Delta\|^2+4\epsilon\|\theta^*+\Delta\||\Delta^\top x_i|\\&&+2\epsilon(\|\theta^*+\Delta\|-\|\theta^*\|)\Delta^{\top}x_i+2\epsilon\frac{\|\theta^*+\Delta\|\|\theta^*\|-(\theta^*+\Delta)^{\top}\theta^*}{\|\theta^*\|}|y_i-x_i^{\top}\theta^*|.
		\end{eqnarray*}
	\end{lemma}
	\begin{proof}
	Assume $x^{\top}\theta_2-y > 0$,
	    \begin{eqnarray*}
			&&(x^{\top}\theta_1-y)^2+2\epsilon\|\theta_1\||x^{\top}\theta_1-y|+\epsilon^2\| \theta_1 \|^2-(x^{\top}\theta_2-y)^2-2\epsilon\|\theta_2\|(x^{\top}\theta_2-y)-\epsilon^2\| \theta_2 \|^2\\
			&=&(x^{\top}\theta_1-y)^2+2\epsilon\|\theta_1\|(x^{\top}\theta_1-y)+\epsilon^2\| \theta_1 \|^2-(x^{\top}\theta_2-y)^2-2\epsilon\|\theta_2\|(x^{\top}\theta_2-y)-\epsilon^2\| \theta_2 \|^2\\
			&&+2\epsilon\|\theta_1\||x^{\top}\theta_1-y|-2\epsilon\|\theta_1\|(x^{\top}\theta_1-y)\\
			&=&2(\theta_1-\theta_2)^{\top}x(x^{\top}\theta_2-y)+(\theta_1-\theta_2)^{\top}xx^{\top}(\theta_1-\theta_2)+2\epsilon^2(\theta_1-\theta_2)^{\top}\theta_2+\epsilon^2\|\theta_1-\theta_2\|^2\\
			&&+2\epsilon\|\theta_1\|(x^{\top}\theta_1-y)-2\epsilon\|\theta_2\|(x^{\top}\theta_2-y)-2\epsilon(\theta_1-\theta_2)\left[ \frac{\theta_2}{\|\theta_2\|}(x^{\top}\theta_2-y)+\|\theta_2\|x \right]\\
			&&+2\epsilon(\theta_1-\theta_2)\left[ \frac{\theta_2}{\|\theta_2\|}(x^{\top}\theta_2-y)+\|\theta_2\|x \right]+2\epsilon\|\theta_1\||x^{\top}\theta_1-y|-2\epsilon\|\theta_1\|(x^{\top}\theta_1-y)\\
			&=& \Delta^{\top}\frac{\partial l_{\epsilon}}{\partial \theta_2} + (\Delta^{\top}x)^2+\epsilon^2\|\Delta\|^2\\
			&&+2\epsilon\|\theta_1\|(x^{\top}\theta_1-y)-2\epsilon\|\theta_2\|(x^{\top}\theta_2-y)-2\epsilon(\theta_1-\theta_2)\left[ \frac{\theta_2}{\|\theta_2\|}(x^{\top}\theta_2-y)+\|\theta_2\|x \right]\\
			&&+2\epsilon\|\theta_1\||x^{\top}\theta_1-y|-2\epsilon\|\theta_1\|(x^{\top}\theta_1-y).
		\end{eqnarray*}
		From Lemma \ref{lem:lasso}, we know that
		\begin{eqnarray*}
			&&\|\theta_1\|(x^{\top}\theta_1-y)-\|\theta_2\|(x^{\top}\theta_2-y)-(\theta_1-\theta_2)\left[ \frac{\theta_2}{\|\theta_2\|}(x^{\top}\theta_2-y)+\|\theta_2\|x \right]\\
			&=&\frac{\|\theta_1\|\|\theta_2\|-\theta_1^{\top}\theta_2}{\|\theta_2\|}(x^{\top}\theta_2-y)+(\|\theta_1\|-\|\theta_2\|)(\theta_1-\theta_2)^{\top}x.
		\end{eqnarray*}
		For $2\epsilon\|\theta_1\||x^{\top}\theta_1-y|-2\epsilon\|\theta_1\|(x^{\top}\theta_1-y)$, since $x^{\top}\theta_2-y>0$,
		it becomes
		\begin{eqnarray*}
		&&2\epsilon\|\theta_1\||x^{\top}\theta_1-y|-2\epsilon\|\theta_1\|(x^{\top}\theta_1-y)\\
		&=&2\epsilon\|\theta_1\||x^{\top}(\theta_2+\Delta)-y|-2\epsilon\|\theta_1\|(x^{\top}\theta_1-y)\\
		&=& 2\epsilon\|\theta_1\|x^{\top}(\theta_2+\Delta)-y|-2\epsilon\|\theta_1\|(x^{\top}(\theta_2+\Delta)-y)\\
		&\leq& 2\epsilon\|\theta_1\|(x^{\top}\theta_2-y+|\Delta^{\top}x|)-2\epsilon\|\theta_1\|(x^{\top}(\theta_2+\Delta)-y)\\
		&=&2\epsilon\|\theta_1\||\Delta^{\top}x|-2\epsilon\|\theta_1\|x^{\top}\Delta\\
		&=& 4\epsilon\|\theta_1\||\Delta^{\top}x|.
		\end{eqnarray*}
		The case of $x^{\top}\theta_2-y<0$ has the same result.
	\end{proof}
\end{document}